\documentclass[onecolumn,11pt]{IEEEtran}

\PassOptionsToPackage{numbers, compress}{natbib}

\usepackage[utf8]{inputenc} 
\usepackage[T1]{fontenc}    
\usepackage{hyperref}       
\usepackage{url}            
\usepackage{booktabs}       
\usepackage{amsfonts}       
\usepackage{nicefrac}       
\usepackage{microtype}      

\usepackage[noadjust]{cite}
\usepackage{amsmath, amssymb, bm, epsfig, psfrag,amsthm}
\usepackage{algorithm,algorithmic}
\usepackage{bbm}
\usepackage[usenames]{color}
\usepackage{tikz}
\usepackage{etoolbox}
\usepackage{enumerate}
\usetikzlibrary{shapes,arrows}
\usetikzlibrary{patterns}
\usepackage{pgfplots}
\usepackage{url,hyperref}
\usepackage{xspace}
\usepackage{enumerate}
\usepackage{mathtools}
\usepackage{epstopdf}
\usepackage{eufrak}
\usepackage{sidecap}

\usepackage[normalem]{ulem} 

\newcommand{\citep}[1]{\cite{#1}}

\renewcommand{\tilde}{\widetilde}
\renewcommand{\hat}{\widehat}

\def\beq{\begin{equation}}
\def\eeq{\end{equation}}
\def\beqa{\begin{eqnarray}}
\def\eeqa{\end{eqnarray}}
\def\beqan{\begin{eqnarray*}}
\def\eeqan{\end{eqnarray*}}

\def\N{{\mathbb{N}}}

\def\R{{\mathbb{R}}}

\def\diag{\mathop{\mathrm{diag}}}
\def\x{\times}

\newtheorem{theorem}{Theorem}
\newtheorem{lemma}{Lemma}

\newtheorem{assumption}{Assumption}

\def\phat{\hat{p}}
\def\qhat{\hat{q}}

\def\arr{\rightarrow}
\def\Exp{\mathbb{E}}
\def\var{\mathrm{var}}
\def\Cov{\mathrm{Cov}}
\def\Tr{\mathrm{Tr}}

\def\alphabar{\overline{\alpha}}
\def\etabar{\overline{\eta}}
\def\gammabar{\overline{\gamma}}

\def\tm1{t\! - \! 1}
\def\tp1{t\! + \! 1}
\def\km1{k\! - \! 1}
\def\kp1{k\! + \! 1}
\def\lp1{\ell\! + \! 1}
\def\lm1{\ell\! - \! 1}
\def\Lm1{L\! - \! 1}
\def\ip1{i\! + \! 1}
\def\im1{i\! - \! 1}

\newcommand{\zero}{\mathbf{0}}

\newcommand{\bbf}{\mathbf{b}}
\newcommand{\cbf}{\mathbf{c}}
\newcommand{\dbf}{\mathbf{d}}

\newcommand{\fbf}{\mathbf{f}}
\newcommand{\gbf}{\mathbf{g}}
\newcommand{\hbf}{\mathbf{h}}
\newcommand{\pbf}{\mathbf{p}}
\newcommand{\pbfhat}{\hat{\mathbf{p}}}
\newcommand{\qbf}{\mathbf{q}}
\newcommand{\qbfhat}{\hat{\mathbf{q}}}
\newcommand{\rbf}{\mathbf{r}}

\newcommand{\sbf}{\mathbf{s}}

\newcommand{\ubf}{\mathbf{u}}

\newcommand{\vbf}{\mathbf{v}}

\newcommand{\wbf}{\mathbf{w}}

\newcommand{\xbf}{\mathbf{x}}

\newcommand{\ybf}{\mathbf{y}}

\newcommand{\zbf}{\mathbf{z}}

\newcommand{\zbfhat}{\hat{\mathbf{z}}}
\newcommand{\Abf}{\mathbf{A}}
\newcommand{\Bbf}{\mathbf{B}}

\newcommand{\Gbf}{\mathbf{G}}

\newcommand{\Ibf}{\mathbf{I}}

\newcommand{\Kbf}{\mathbf{K}}

\newcommand{\Pbf}{\mathbf{P}}

\newcommand{\Qbf}{\mathbf{Q}}
\newcommand{\Rbf}{\mathbf{R}}

\newcommand{\Ubf}{\mathbf{U}}

\newcommand{\Vbf}{\mathbf{V}}

\newcommand{\Wbf}{\mathbf{W}}

\def\betabf{{\boldsymbol \beta}}

\def\lambdabar{\overline{\lambda}}
\def\Lambdabar{\overline{\Lambda}}
\def\Sigmabf{{\boldsymbol \Sigma}}

\def\xibf{{\boldsymbol \xi}}

\def\varphibf{{\boldsymbol \varphi}}

\newcommand{\phibf}{{\bm{\phi}}}

\newcommand{\mubar}{\overline{\mu}}

\newcommand{\tran}{^{\text{\sf T}}}

\def\eqd{\stackrel{d}{=}}
\def\PLeq{\stackrel{PL(2)}{=}}
\def\Norm{{\mathcal N}}
\def\Range{\mathrm{Range}}

\def\alphabar{\overline{\alpha}}

\def\Ecal{{\mathcal E}}

\newcommand{\bkt}[1]{{\langle #1 \rangle}}
\newcommand{\bktAuto}[1]{{\left\langle #1 \right\rangle}} 
\def\Gset{\mathfrak{G}}
\def\Gsetbar{\overline{\mathfrak{G}}}

\title{Inference in Deep Networks in High Dimensions}

\author{Alyson~K.~Fletcher and Sundeep Rangan%
  \thanks{A.~K.~Fletcher (email: akfletcher@ucla.edu) is with
          the Department of Statistics and Electrical Engineering,
          the University of California, Los Angeles, CA, 90095\@.
          Her work was supported in part by the National Science Foundation under
          Grants 1254204 and 1738286, and the Office of Naval Research under Grant
          N00014-15-1-2677.

          S. Rangan (email: srangan@nyu.edu) is with
          the Department of Electrical and Computer Engineering,
          New York University, Brooklyn, NY, 11201\@.
          His work was supported in part by the National Science Foundation
          under Grants 1116589, 1302336, and 1547332, as well as
          the industrial affiliates of NYU WIRELESS.}
}

\begin{document}

\maketitle

\begin{abstract}
Deep generative networks provide a powerful tool for modeling complex data
in a wide range of applications.  In inverse problems that use these networks
as generative priors on data, one must often perform inference of the inputs of the networks
from the outputs.  Inference is also required for sampling
during stochastic training of these generative models.
This paper considers inference in a deep stochastic neural network where the parameters
(e.g., weights, biases and activation functions)
are known and the problem is to estimate the values of the input and
hidden units from the output.
While several approximate algorithms have been proposed for this task,
there are few analytic tools that can provide rigorous guarantees on the reconstruction error.
This work presents a novel and computationally tractable output-to-input
inference method
called Multi-Layer Vector Approximate Message Passing (ML-VAMP).
The proposed algorithm, derived from expectation propagation, extends
earlier AMP methods that are known to achieve the replica predictions
for optimality in simple linear inverse problems.
Our main contribution shows that the mean-squared error of ML-VAMP can be exactly
predicted in a certain large system limit
where the numbers of layers is fixed and weight matrices are random
and orthogonally-invariant with dimensions that grow to infinity.
ML-VAMP is thus a principled method for output-to-input inference in deep networks with
a rigorous and precise performance achievability result in high dimensions.
\end{abstract}

\section{Introduction}

Deep neural networks are increasingly used for describing
probabilistic generative models of complex data such as images,
audio and text \citep{rezende2014stochastic,kingma2013auto,salakhutdinov2015learning}.
In these models,  data $\ybf$ is typically represented as 
the output of a feedforward
neural network with randomness in the input.  Randomness may also appear in the hidden layers.
This work considers the \emph{inference} problem for such a network,
where the parameters are known (i.e., already trained) and
we are to estimate the values of the inputs and hidden units
from output data values $\ybf$.

Inference tasks of this form arise in inverse
problems where a deep network is used as
a generative prior for the data (such as an image)
and additional layers are added to model the measurements (such as blurring, occlusion or noise)
\citep{yeh2016semantic,bora2017compressed}.
Inference can then be used to reconstruct the original image from the measurements
and provides an alternative to direct training for reconstruction~\citep{MousaviPB:15-Allerton}.
Also, in unsupervised learning of the parameters of a generative network,
one must sample from the posterior density of the hidden
variables to perform stochastic gradient descent or EM
\citep{rezende2014stochastic,kingma2013auto,sohl2015deep}.

While inference
is most commonly a feedforward operation (given a new input, we compute the output
of the network to make a classification or other prediction decision),
here we are considering the inference problem in the reverse direction where we need
to infer the input from the output.
Although optimal output-to-input inference is generally intractable due to the nonlinear nature of neural networks,
there are several methods that have worked well in practice.
For example, MAP estimation can be performed by gradient descent on the negative
log likelihood where the gradients can be computed efficiently from backpropagation and has been successful
for problems such as inpainting
\citep{yeh2016semantic,bora2017compressed}.  Approximate inference can also be performed
via a separate learned deep network as is done in variational autoencoders
\citep{rezende2014stochastic,kingma2013auto} and adversarial
networks \citep{dumoulin2016adversarially}.  See also, \citep{bengio2014deep}.
However, similar to the situation in deep learning in general,
there are few analytic tools for understanding how these algorithms perform
or how far the estimates are from optimal.  

In this work, we address this shortcoming by considering inference
based on approximate message passing (AMP)~\citep{DonohoMM:09,DonohoMM:10-ITW1}.
AMP methods are a class of expectation propagation (EP) 
techniques \citep{Minka:01}
that perform inference by attempting to minimize an approximation to the Bethe Free Energy
\citep{Krzakala:14-ISITbethe,rangan2017inference}.
In addition to their computational simplicity, AMP methods
have the benefit that the reconstruction error can be precisely
characterized in certain high-dimensional random settings.  Moreover, under further assumptions, they can provably obtain
the Bayesian optimal performance as predicted by the replica method
\citep{barbier2016mutual,reeves2016replica,tulino2013support} -- see, also 
\citep{RanganFG:12-IT}.
Since their original work in sparse linear inverse problems \cite{DonohoMM:09,DonohoMM:10-ITW1}, 
AMP techniques have been successfully used to obtain rigorous theoretical guarantees in a wide range of settings
including generalized linear models \cite{Rangan:11-ISIT}, clustering \citep{krzakala2013spectral},
finding hidden cliques \cite{deshpande2015finding} and  matrix factorization \cite{RanganF:12-ISIT}.

A recent extension of these
methods, called multi-layer AMP, has been proposed for inference in deep networks
\citep{manoel2017multi}.  That work characterizes the replica prediction for optimality
in multi-layer networks
and argues that the proposed ML-AMP method can achieve this optimal inference in certain scenarios.
Unfortunately, the convergence of ML-AMP in \citep{manoel2017multi} is not rigorously proven.
In addition, ML-AMP assumes Gaussian i.i.d.\ weight matrices $\Wbf_\ell$, and it is well-known
that AMP methods
often fail to converge when this assumption does not hold
\cite{RanSchFle:14-ISIT,Caltagirone:14-ISIT,Vila:ICASSP:15,manoel2015swamp,rangan2015admm}.

In this work, we propose a novel AMP method
called multi-layer vector AMP (ML-VAMP) that
builds on the recent VAMP method of \citep{rangan2016vamp} and its extensions to GLMs
in \citep{he2017generalized,schniter2017vector}.
The VAMP algorithm of \citep{rangan2016vamp} was itself derived
from the expectation consistent approximate inference framework of \citep{opper2004expectation,OppWin:05,fletcher2016expectation}
and applies to the special case of a linear problem.
The ML-VAMP algorithm proposed here extends the VAMP method to networks with multiple layers and nonlinearities.
Prior works in EP techniques for neural networks such as \citep{jylanki2014expectation,bui2016deep}
apply to the learning problem, not the inference problem considered here.

We analyze ML-VAMP in a setting where the
number of layers is fixed and the weight matrices are random and orthogonally
invariant with dimensions that grow to infinity.  This class is much larger
than the Gaussian i.i.d.\ matrices.  Importantly, it includes weight matrices with arbitrary condition
numbers, which is known to be the main failure mechanism in conventional AMP convergence
\citep{RanSchFle:14-ISIT}.  Our main theoretical contribution (Theorem~\ref{thm:semlvamp})
shows that the mean squared error (MSE) of
ML-VAMP algorithm can be precisely predicted by a simple set of scalar state evolution (SE) equations.
The SE equations relate the achievable MSE to the key parameters of the network including
the statistics of the weight matrices and bias vectors, their dimensions, the
noise terms and activation functions.  In this way, we develop a principled algorithm
that enables computationally tractable inference with rigorous analysis of its achievable performance
and convergence.

Our methods for analyzing inference algorithms bear some similarities
with related recent theoretical analyses in  deep learning.  
For example, \citep{patel2016probabilistic} have shown that deep networks can be interpreted as
max-sum inference on a certain deep rendering network.
The work \citep{SaxeMG:14arXiv} studies the dynamics of gradient descent learning
in a deep \emph{linear} network.  Interestingly, this work uses a singular value
decomposition (SVD) of the weight matrices in the analysis, which is
critical in our analysis as well.
Also, \citep{choromanska2015loss} uses a large random weight matrix model in analyzing the
the geometry of the loss function.  This model is similar to the large system limit considered here.   
However, these methods all consider the more challenging problem
of learning deep networks.  In the inference problem considered here, the parameters of the network are already known
and we only need to estimate the input and hidden states.  In this regard, 
we can think of this inference problem as a simpler starting point for understanding deep learning methods.

\section{ML-VAMP Algorithm} \label{sec:mlvamp}

\subsection{Algorithm Overview}

We consider the following $M$-layer stochastic generative neural network model for data:
A random input $\zbf_0$ with some density $p(\zbf_0)$ 
generates a sequence of vectors, $\zbf_\ell$, $\ell=1,\ldots,L$, $L=2M$,
through operations of the form,
\begin{subequations}  \label{eq:nn}
\begin{align}
    \zbf_\ell &= \Wbf_{\ell}\zbf_{\lm1} + \bbf_{\ell} + \xibf_\ell, \quad
        & \xibf_\ell \sim \Norm(\zero,\nu_\ell^{-1}\Ibf), \quad
        & \ell=1,3,\ldots,2M-1
        \label{eq:nnlin} \\
    \zbf_{\ell} &=  \phibf_\ell(\zbf_{\lm1},\xibf_{\ell}), \quad
        & \xibf_\ell \sim p(\xibf_\ell), \quad
        & \ell = 2,4,\ldots,2M.
        \label{eq:nnnonlin}
\end{align}
\end{subequations}
The updates \eqref{eq:nnlin} are the \emph{linear stages} of the network and are defined
by weight matrices $\Wbf_{\ell}$, bias vectors $\bbf_{\ell}$ and Gaussian noise
terms $\xibf_\ell$.
The updates \eqref{eq:nnnonlin} are the \emph{nonlinear stages} and are
defined with activation functions $\phi_\ell(\cdot)$ such as a sigmoid or
rectified linear unit (ReLU).
The vectors $\xibf_\ell$ are noise terms to model randomness in each stage.
The final output $\ybf=\zbf_L$ is the final (observed) data.
We consider the output-to-input inference problem, where we are 
to estimate all the hidden variables in the network $\zbf_\ell$, $\ell=0,\ldots,\Lm1$ from the output
$\ybf=\zbf_L$.
Importantly,
the weight matrices,
bias terms,
and
activation functions $\phi_\ell(\cdot)$ are known (i.e.\ already trained).
Thus, we are not looking
at the \emph{learning problem}.

The proposed ML-VAMP algorithm
for this inference problem
is shown in Algorithm~\ref{algo:ml-vamp}.
It
can be derived as
an extension of the GEC-SR algorithm of \citep{he2017generalized} which is used
for inference in a GLM, which is the special case of the multi-layer problem with $L=2$.
We assume a Bayesian setting where the initial condition $\zbf_0$ and noise terms $\xibf_\ell$ are independent
random vectors, so the sequence $\zbf_\ell$ in \eqref{eq:nn} is Markov.
The joint density of the variables $\zbf_\ell$ can then be written as
\beq \label{eq:pjoint}
    \pbf(\zbf_0,\ldots,\zbf_L) = p(\zbf_0) \prod_{\ell=0}^{\Lm1} p(\zbf_{\lp1}|\zbf_\ell),
\eeq
where the transition probabilities
$p(\zbf_{\lp1}|\zbf_\ell)$
are defined implicitly by the updates in \eqref{eq:nn}.
This density can be represented as a factor graph with $L+1$ factors corresponding to the terms
$p(\zbf_0)$ and $p(\zbf_{\ell+1}|\zbf_\ell)$, $\ell=0,\ldots,L-1$.  Since the
final variable $\zbf_L=\ybf$ is observed, we have $L$ hidden variables $\zbf_0,\ldots,\zbf_{\Lm1}$.
This creates a linear graph.  

Similar to the GEC-SR algorithm, the ML-VAMP algorithm shown in Algorithm~\ref{algo:ml-vamp}
passes messages in the forward and reverse directions
along the graph.  The values $\rbf^+_{k\ell}$ and $\gamma^+_{k\ell}$ represent the mean
and precision (inverse variance) values in the forward direction, and
$\rbf^+_{k\ell}$ and $\gamma^-_{k\ell}$ are the values in the reverse direction.
The formulae for the updates can be derived almost exactly the same as those in the GEC-SR
algorithm  \citep{he2017generalized} and are also similar to those in \citep{rangan2016vamp,fletcher2016expectation}.
We thus simply repeat the formulae without derivation.
To update the messages, at each factor node $\ell=1,\ldots,\Lm1$ in the ``middle" of the
factor graph, we compute a \emph{belief estimate}, which is a probability density
\beq \label{eq:bdef}
    b_\ell(\zbf_\ell,\zbf_{\lm1}|\rbf^+_{\lm1},\rbf^-_{\ell},\gamma^+_{\lm1},\gamma^-_{\ell})
     \propto \exp\left[ -H_\ell(\zbf_\ell,\zbf_{\lm1}) \right],
\eeq
where $H_\ell(\cdot)$ is the energy function
\beq \label{eq:Hdef}
    H_\ell(\zbf_\ell,\zbf_{\lm1}) := -\ln p(\zbf_\ell|\zbf_{\lm1}) +
        \frac{\gamma^-_{\ell}}{2}\|\zbf_\ell-\rbf^-_\ell\|^2
        + \frac{\gamma^+_{\lm1}}{2}\|\zbf_{\lm1}-\rbf^-_{\lm1} \|^2.
\eeq
At each iteration $k$, the belief estimates \eqref{eq:bdef} with the values $\rbf^+_{k,\lm1}$,
$\rbf^-_{k\ell}$, $\gamma^+_{k,\lm1}$ and $\gamma^-_{k\ell}$
represent the estimates
of the posterior density $p(\zbf_{\lm1},\zbf_{\ell}|\xbf)$.
We define the \emph{estimation functions} $\gbf^{\pm}_\ell(\cdot)$ as the functions
that compute the expected values of $\zbf_{\lm1}$ and $\zbf_\ell$ with respect to
these densities,
\begin{subequations} \label{eq:gexp}
\begin{align}
    & \gbf_\ell^+(\rbf^+_{\lm1},\rbf^-_{\ell},\gamma^+_{\lm1},\gamma^-_{\ell})
    = \Exp\left[ \zbf_\ell | \rbf^+_{\lm1},\rbf^-_{\ell},\gamma^+_{\lm1},\gamma^-_{\ell} \right], \\
    & \gbf_\ell^-(\rbf^+_{\lm1},\rbf^-_{\ell},\gamma^+_{\lm1},\gamma^-_{\ell})
    = \Exp\left[ \zbf_{\lm1} | \rbf^+_{\lm1},\rbf^-_{\ell},\gamma^+_{\lm1},\gamma^-_{\ell} \right],
\end{align}
\end{subequations}
where the expectations are with respect to the belief estimates \eqref{eq:bdef}.
For the end points $\ell=0$ and $L$ in the factor graph, we define the belief estimates
\[
    b_0(\zbf_0|\rbf^-_0,\gamma^-_0), \quad b_L(\zbf_{\Lm1}| \rbf^+_{\Lm1},\gamma^+_{\Lm1}),
\]
similar to \eqref{eq:bdef} but with the terms from the left of the $\ell=0$ factor node
or right of the $\ell=L$ factor node removed.

\begin{algorithm}[t]
\caption{ML-VAMP}
\begin{algorithmic}[1]  \label{algo:ml-vamp}
\REQUIRE{Forward estimation functions $g_\ell^+(\cdot)$, $\ell=0,\ldots,\Lm1$ and
reverse estimation functions $\gbf_\ell^-(\cdot)$, $\ell=1,\ldots,L$.}
\STATE{Initialize $\rbf^-_{0\ell}$, $\gamma^-_{0\ell}$}
\FOR{$k=0,1,\dots,N_{\rm it}-1$}

    \STATE{// Forward Pass }
    \FOR{$\ell=0,\ldots,\Lm1$}
        \IF{$\ell=0$}
        \STATE{$\zbfhat^+_{k\ell} =
            \gbf_\ell^+(\rbf^-_{k\ell},\gamma^-_{k\ell})$}
            \label{line:zp0}
        \STATE{$\alpha^+_{k\ell} = \bkt{\partial
            \gbf_\ell^+(\rbf^-_{k\ell},\gamma^-_{k\ell})/
            \partial \rbf^-_{k\ell}}$}
            \label{line:alphap0}
        \ELSE
            \STATE{$\zbfhat^+_{k\ell} =
            \gbf_\ell^+(\rbf^+_{k,\lm1},\rbf^-_{k\ell},\gamma^+_{k,\lm1},\gamma^-_{k\ell})$}
            \label{line:zp}
            \STATE{$\alpha^+_{k\ell} = \bkt{\partial
                \gbf_\ell^+(\rbf^+_{k,\lm1},\rbf^-_{k\ell},\gamma^+_{k\lm1},\gamma^-_{k\ell})/
                \partial \rbf^-_{k\ell}}$}
                \label{line:alphap}

        \ENDIF
        \STATE{$\gamma^+_{k\ell} = \eta^+_{k\ell} - \gamma^-_{k\ell}$,
            $\eta^+_{k\ell} = \gamma^-_{k\ell}/\alpha^+_{k\ell}$}
            \label{line:gamp}
        \STATE{$\rbf^+_{k\ell} = (\eta^+_{k\ell}\zbfhat^+_{k\ell} - \gamma^-_{k\ell}\rbf^-_{k\ell})/
            \gamma^+_{k\ell}$}
            \label{line:rp}
    \ENDFOR

    \STATE{// Reverse Pass }
    \FOR{$\ell=\Lm1,\ldots,0$}
        \IF{$\ell=\Lm1$}
            \STATE{$\zbfhat^-_{k\ell} =
                \gbf_{\lp1}^-(\rbf^+_{k\ell},\gamma^+_{k\ell},)$}
                \label{line:znL}
            \STATE{$\alpha^-_{k\ell} = \bkt{\partial
            \gbf_{\lp1}^-(\rbf^+_{k\ell},\gamma^+_{k\ell})/ \partial \rbf^+_{k\ell}}$}
                \label{line:alphanL}
        \ELSE
            \STATE{$\zbfhat^-_{k\ell} =
                \gbf_{\lp1}^-(\rbf^+_{k\ell},\rbf^-_{\kp1,\lp1},\gamma^+_{k\ell},\gamma^-_{\kp1,\lp1})$}
                \label{line:zn}
            \STATE{$\alpha^-_{k\ell} = \bkt{\partial
            \gbf_{\lp1}^-(\rbf^+_{k\ell},\rbf^-_{\kp1,\lp1},\gamma^+_{k\ell},\gamma^-_{\kp1,\lp1}) /
                \partial \rbf^+_{k\ell}}$}
                \label{line:alphan}
        \ENDIF
        \STATE{$\gamma^-_{\kp1,\ell} = \eta^-_{k\ell} - \gamma^+_{k\ell}$,
            $\eta^-_{k\ell} = \gamma^+_{k\ell}/\alpha^-_{k\ell}$}
            \label{line:gamn}
        \STATE{$\rbf^-_{\kp1,\ell} = (\eta^-_{k\ell}\zbfhat^-_{k\ell}
            - \gamma^+_{k\ell}\rbf^+_{k\ell})/ \gamma^+_{\kp1,\ell}$}
            \label{line:rn}
    \ENDFOR

\ENDFOR
\end{algorithmic}
\end{algorithm}

We also use the notation that for any vector $\ubf \in \R^N$,
$\bkt{\ubf} := (1/N) \sum_{n=1}^N u_n$
which is the empirical average over the components.  For a matrix $\Qbf \in \R^{N \x N}$ we let
$\bkt{\Qbf} = (1/N) \Tr(\Qbf)$ which is the average of the diagonal components.
Similar to the derivations in \citep{rangan2016vamp} and \citep{he2017generalized},
the derivatives $\alpha^{\pm}_{k\ell}$ are given by
\begin{subequations}  \label{eq:gderiv}
\begin{align}
    \frac{1}{\eta^+_{k\ell}} &=  \bkt{ \diag \Cov\left(\zbf_\ell|
        \rbf^+_{k,\lm1},\rbf^-_{k\ell},\gamma^+_{k\lm1},\gamma^-_{k\ell} \right) },
    \\
    \frac{1}{\eta^-_{k\ell}} &=  \bkt{ \diag \Cov\left(\zbf_\ell|
        \rbf^+_{k,\lm1},\rbf^-_{\kp1,\ell},\gamma^+_{k\lm1},\gamma^-_{\kp1,\ell} \right) },
\end{align}
\end{subequations}
and
\beq \label{eq:alphaderiv}
    \alpha_{k\ell}^+ = \frac{\gamma_{k\ell}^+}{\eta_{k\ell}^+}, \quad
    \alpha_{k,\lm1}^+ = \frac{\gamma_{k,\lm1}^+}{\eta_{k,\lm1}^-}
\eeq
Hence, the derivatives $\alpha_{k\ell}^{\pm}$ can be computed from the trace of the covariance
matrices under the belief estimates.
For the initial conditions, we set $\rbf^-_{0\ell} = \zero$ and $\gamma_{0\ell}^- = 0$ for all $\ell$.

\subsection{Estimation Functions for the Neural Network} \label{sec:nnest}

For the stochastic neural network \eqref{eq:nn}, the estimation functions have a particularly
simple form for both the linear and nonlinear stages.

\paragraph*{Estimation functions for the nonlinear stages}
Let $\ell=2,4,\ldots,L$ corresponding to a nonlinear stage \eqref{eq:nnnonlin}.
We will assume that the activation function $\phibf_\ell(\cdot)$ acts componentwise and the noise
$\xibf_\ell$ is i.i.d.\ meaning
\[
    z_{\ell,n} = \left[ \phibf_{\ell}(\zbf_{\lm1}, \xibf_\ell) \right]_n
    = \phi_\ell(z_{\lm1,n},\xi_{\ell n}), \quad p(\xibf_\ell) = \prod_{n=1}^{N_\ell} p(\xi_{\ell,n}),
\]
where $\phi_\ell(\cdot)$ is a scalar-valued function.  This model applies to many activation functions
in neural networks including ReLUs and sigmoids.
Under this assumption, the transition probability
factorizes as $p(\zbf_\ell|\zbf_{\lm1}) = \prod_n p(z_{\ell,n}|z_{\lm1,n})$ and therefore the
estimation functions also act componentwise,
\beq \label{eq:gnl}
    \left[ \gbf_\ell^+(\rbf^+_{\lm1},\rbf^-_{\ell},\gamma^+_{\lm1},\gamma^-_{\ell}) \right]_n
    = g_\ell^+(r^+_{\lm1,n},r^-_{\ell,n},\gamma^+_{\lm1},\gamma^-_{\ell})
    := \Exp\left[ z_{\ell,n} | r^+_{\lm1,n},r^-_{\ell,n},\gamma^+_{\lm1},\gamma^-_{\ell} \right],
\eeq
where the expectation is with respect to the scalar density,
\[
    b_{\ell}(z_{\lm1,n},z_{\ell,n}|r_{\lm1,n}^+,r_{\ell,n}^-,\gamma^+_{\lm1},\gamma^-_\ell)
        \propto \exp\left[ -H_{\ell,n}(z_{\lm1,n},z_{\ell,n}) \right], \quad
\]
where $H_{\ell,n}(\cdot)$ is the scalar energy function,
\beq \label{eq:Hdefsca}
    H_{\ell,n}(z_{\lm1,n},z_{\ell,n}) := -\ln p(z_{\ell,n}|z_{\lm1,n}) +
        \frac{\gamma^-_{\ell}}{2}\|z_{\ell,n}-r^-_{\ell,n}\|^2
        + \frac{\gamma^+_{\lm1}}{2}\|z_{\lm1,n}-r^-_{\lm1,n} \|^2.
\eeq
Hence, the estimation on the nonlinear stages can be evaluated by integration
of $N_\ell$ two dimensional densities.  The estimation function for the reverse direction $\gbf_\ell^-(\cdot)$
has a similar componentwise structure.

\paragraph*{Estimation functions for the linear stages}
Let $\ell=1,3,\ldots,\Lm1$.  Since the linear relation~\ref{eq:nnlin} is Gaussian,
the energy function \eqref{eq:Hdef} is quadratic and the belief estimate
\eqref{eq:bdef} is Gaussian.  Therefore, the expectation in \eqref{eq:gexp} and covariance
\eqref{eq:gderiv} can be computed  via a standard least squares problem.

For our analysis below, it will be easiest to write the solution to the least squares estimation
in terms of an SVD.
For $\ell=1,3,\ldots,\Lm1$, we assume that the weight matrix $\Wbf_\ell$ is given by,
\beq \label{eq:WSVD}
    \Wbf_{\ell} = \Vbf_{\ell}\Sigmabf_\ell\Vbf_{\lm1}, \quad
    \Sigmabf_{\ell} =
    \left[ \begin{array}{cc}
    \diag(\sbf_\ell) & \zero \\
    \zero & \zero  \end{array} \right] \in \R^{N_\ell \x N_{\lm1}},
\eeq
where the matrix $\Wbf_\ell$ has at most rank $R_{\ell}$,
$\sbf_\ell = (s_{\ell 1},\ldots,s_{\ell R_\ell})$ is the vector
of singular values and $\Vbf_\ell$ and $\Vbf_{\lm1}$ are orthogonal.
Also, let $\bar{\bbf}_\ell := \Vbf_\ell\tran\bbf_\ell$ and $\bar{\xibf}_\ell := \Vbf_\ell\tran \xibf_\ell$
so that
\beq \label{eq:bxibar}
    \bbf_\ell = \Vbf_\ell\bar{\bbf}_\ell, \quad
    \xibf_\ell = \Vbf_\ell\bar{\xibf}_\ell,
\eeq
Then, it shown in Appendix~\ref{sec:linestim} that
the linear estimation functions \eqref{eq:gexp} are  given by
\begin{subequations} \label{eq:glin}
\begin{align}
    \gbf^+_\ell(\rbf^+_{\lm1},\rbf^-_{\ell},\gamma_{\lm1}^+,\gamma_{\ell}^-) &=
    \Vbf_\ell\Gbf_\ell^+(\Vbf_{\lm1}\rbf^+_{\lm1},\Vbf_\ell\tran \rbf^-_{\ell},
        \sbf_\ell,\bar{\bbf}_\ell,\gamma_{\lm1}^+,\gamma_{\ell}^-), \\
    \gbf^-_\ell(\rbf^+_{\lm1},\rbf^-_{\ell},
        \gamma_{\lm1}^+,\gamma_{\ell}^-) &=
    \Vbf_{\lm1}\tran\Gbf_\ell^-(\Vbf_{\lm1}\rbf^+_{\lm1},\Vbf_\ell\tran \rbf^-_{\ell},
        \sbf_\ell,\bar{\bbf}_\ell,\gamma_{\lm1}^+,\gamma_{\ell}^-),
\end{align}
\end{subequations}
for some functions $\Gbf_\ell^{\pm}(\cdot)$ that act componentwise.

\section{State Evolution Analysis of ML-VAMP} \label{sec:seevo}

\subsection{Large System Limit Model} \label{sec:lsl}

Our main contribution is to rigorously analyze ML-VAMP in a certain
large system limit (LSL). The LSL analysis is widely-used in studying AMP algorithms
and their variants \citep{BayatiM:11,rangan2016vamp}.  The LSL model for ML-VAMP
is as follows.   We consider
a sequence of problems indexed by $N$.  The number of stages $L$ is fixed and,
the dimensions $N_\ell = N_\ell(N)$ and matrix ranks $R_\ell = R_\ell(N)$ in each stage
are deterministic functions of $N$.
We assume that
$\lim_{N \arr \infty} N_\ell/N$ and $\lim_{N \arr \infty} R_\ell/N$ converge to non-zero constants
so the dimensions grow linearly.
We follow the framework of Bayati-Montanari \citep{BayatiM:11}, and model
various sequences as deterministic but whose distributions converge empirically
-- See Appendix~\ref{sec:empconv} for a review of this framework.
Specifically, we assume that the initial condition $\zbf^0_0 \in \R^{N_0}$ and noise vectors
in the nonlinear stages
$\xibf_\ell$,  $\ell=2,4,\ldots,L$, converge empirically as
\beq \label{eq:varinitnl}
    \lim_{N \arr \infty} \left\{ z^0_{0,n} \right\} \PLeq Z^0_0, \quad
    \lim_{N \arr \infty} \left\{ \xi_{\ell,n} \right\} \PLeq \Xi_\ell, \quad \ell =2,4,\ldots,L
\eeq
to random variables $Z^0$ and $\Xi_\ell$.  For the linear stages $\ell=1,3,\ldots,\Lm1$,
let $\bar{\sbf}_\ell$ be the zero-padded singular value vector,
\beq \label{eq:sbar}
    \bar{s}_{\ell,n}= \begin{cases}
        s_{\ell,n} & \mbox{if } n=1,\ldots,R_\ell, \\
        0          & \mbox{if } n=R_\ell+1,\ldots,N_\ell,
    \end{cases}
\eeq
so that $\bar{\sbf}_\ell \in \N^\ell$.
We assume that zero-padded singular value vector $\bar{\sbf}_\ell$,
the transformed bias $\bar{\bbf}_\ell=\Vbf_\ell\tran\bbf_\ell$ and transformed noise
$\bar{\xibf}_\ell=\Vbf_\ell\tran\xibf_\ell$ converge empirically as
\beq \label{eq:varinitlin}
    \lim_{N \arr \infty} \left\{ \bar{s}_{\ell,n},\bar{b}_{\ell,n},\bar{\xi}_{\ell,n} \right\}
        \PLeq (\bar{S}_\ell, \bar{B}_\ell,\bar{\Xi}_\ell),     \quad \ell =1,3,\ldots,\Lm1,
\eeq
to independent random variables $\bar{S}_\ell$, $\bar{B}_\ell$ and $\bar{\Xi}_\ell$ with
$\bar{\Xi}_\ell \sim \Norm(0,\nu_\ell)$, where
$\nu_\ell$ is the  noise variance.  We assume that $\bar{S}_\ell \geq 0$ and bounded $\bar{S}_\ell \leq S_{max}$ for some
$S_{max}$.

The matrices $\Vbf_\ell$ are Haar distributed (i.e.\ uniform
on the set of $N_\ell \times N_\ell$ orthogonal matrices) where the matrices
$\Vbf_\ell$ are independent
of one another and the signals above.  For any linear stage $\ell$,
the weight matrix $\Wbf_\ell$, bias $\bbf_\ell$ and noise $\xibf_\ell$ are then generated
from \eqref{eq:WSVD} and \eqref{eq:bxibar}.
Finally, the vectors $\zbf^0_\ell$ in the neural network are generated from the recursions,
\begin{subequations}  \label{eq:nntrue}
\begin{align}
    \zbf^0_\ell &= \Wbf_{\ell}\zbf^0_{\lm1} + \bbf_{\ell} + \xibf_\ell,
    \quad \ell=1,3,\ldots,\Lm1
        \label{eq:nnlintrue} \\
    \zbf^0_{\ell} &=  \phibf_\ell(\zbf^0_{\lm1},\xibf_{\ell}), \quad \ell = 2,4,\ldots,L.
        \label{eq:nnnonlintrue}
\end{align}
\end{subequations}
Note that we have used the superscripted values, $\zbf^0_\ell$, to indicate the ``true"
values of $\zbf_\ell$. For each nonlinear stage $\ell=2,4,\ldots,L$,
we assume that the activation function $\phibf_\ell(\cdot)$ acts componentwise meaning
\beq \label{eq:phicomp}
    \left[ \phibf_\ell(\zbf_{\lm1},\xibf_\ell)\right]_n = \phi_\ell(z_{\lm1,n},\xi_{\ell,n}),
\eeq
for some scalar-valued function $\phi_\ell(\cdot)$.

Given the signals generated from the random model describe above,
we run the ML-VAMP algorithm (Algorithm~\ref{algo:ml-vamp}) using the estimation functions \eqref{eq:gexp}
matched to the true conditional densities $p(\zbf_\ell|\zbf_{\lm1})$.  Similar to
the analysis of the VAMP algorithm in \citep{rangan2016vamp}, one can study the ML-VAMP algorithm
under arbitrary Lipschitz-continuous estimation functions $\gbf^{\pm}_\ell(\cdot)$.  However, the state evolution
equations become more complicated.  For space considerations, we present only the SE equations in the
MMSE matched case.

\begin{algorithm}
\caption{ML-VAMP State Evolution}
\begin{algorithmic}[1]  \label{algo:mlvamp_se}

\REQUIRE{Random variables $Z^0_0$, $\Xi_\ell$, $\bar{B}_\ell$, $\bar{S}_\ell$, $\bar{\Xi}_\ell$.}

\STATE{}
\STATE{Initialize $\gammabar^-_{0\ell}=0$} \label{line:gaminit_mlse}
\STATE{$Q^0_0 = Z^0_0, \quad P_0 \sim \Norm(0,\tau^0_0),   \quad \tau^0_0 = \Exp(Q^0_0)^2$} \label{line:q0init_mlse}
\FOR{$\ell=1,2,\ldots,\Lm1$}
    \IF{$\ell$ is odd}
        \STATE{$Q^0_\ell=\bar{S}_\ell P^0_{\lm1} + \bar{B}_\ell + \bar{\Xi}_\ell$}
    \ELSE
        \STATE{$Q^0_\ell=\phi_\ell(P^0_{\lm1},\Xi_\ell)$}
    \ENDIF
    \STATE{$P^0_\ell = \Norm(0,\tau^0_\ell), \quad \tau^0_\ell = \Exp(Q^0_\ell)^2$} \label{line:p0init_mlse}
\ENDFOR
\STATE{}

\FOR{$k=0,1,\dots$}
    \STATE{// Forward Pass }
    \STATE{$\etabar_{k0}^+ = 1/\Ecal_0^+(\gammabar^-_{k0})$}  \label{line:etap0_mlse}
    \STATE{$\gammabar^+_{k0} = \etabar_{k0}^+ - \gammabar^-_{k0}, \quad \alphabar_{k0}^+ = \gammabar^+_{k0}/\etabar_{k0}^+$}
             \label{line:gamp0_mlse}

    \FOR{$\ell=1,\ldots,L-1$}
        \STATE{$\etabar_{k\ell}^+ = 1/\Ecal_0^+(\gammabar^+_{k,\lm1},\gammabar^-_{k\ell},\tau^0_{\lm1})$} \label{line:etap_mlse}
        \STATE{$\gammabar^+_{k\ell} = \etabar_{k\ell}^+ - \gammabar^-_{k\ell}, \quad \alphabar_{k\ell}^+ = \gammabar^+_{k\ell}/\etabar_{k\ell}^+$}
            \label{line:gamp_mlse}
    \ENDFOR
    \STATE{}
    \STATE{// Reverse Pass }
    \STATE{$\etabar_{k,\Lm1}^- = 1/\Ecal_L^-(\gammabar^+_{k,\Lm1})$}     \label{line:etan0_mlse}
    \STATE{$\gammabar^-_{k,\Lm1} = \etabar_{k,\Lm1}^+ - \gammabar^+_{k,\Lm1}, \quad
              \alphabar_{k,\Lm1}^- = \gammabar^-_{k,\Lm1}/\etabar_{k,\Lm1}^+$}
             \label{line:gamn0_mlse}

    \FOR{$\ell=\Lm1,\ldots,0$}
        \STATE{$\etabar_{k,\lm1}^- = 1/\Ecal_{\ell}^-(\gammabar^+_{k,\lm1},\gammabar^-_{k\ell},\tau^0_{\lm1})$}  \label{line:etan_mlse}
        \STATE{$\gammabar^-_{k,\lm1} = \etabar_{k,\lm1}^- - \gammabar^+_{k,\lm1},
        \quad \alphabar_{k,\lm1}^- = \gammabar^+_{k,\lm1}/\etabar_{k,\lm1}^-$}
        \label{line:gamn_mlse}
    \ENDFOR

\ENDFOR

\end{algorithmic}
\end{algorithm}

\subsection{State Evolution Equations} \label{sec:seeqn}
Define the quantities
\begin{align}  \label{eq:pq0}
\begin{split}
    \qbf^0_\ell &:= \zbf^0_\ell, \quad \pbf^0_\ell := \Vbf_\ell \qbf^0_\ell = \Vbf_\ell \zbf^0_\ell \quad \ell=0,2,\ldots,L \\
    \qbf^0_\ell &:= \Vbf_\ell\tran\zbf^0_\ell, \quad \pbf^0_\ell := \zbf^0_\ell = \Vbf_\ell \qbf^0_\ell, \quad \ell=1,3,\ldots,\Lm1
\end{split}
\end{align}
which represent the true vectors $\zbf^0_\ell$ and their transforms.  Similarly, define the ML-VAMP estimates
\begin{subequations} \label{eq:pqhat}
\begin{align}
    \qbfhat^{\pm}_{k\ell} &:= \zbfhat^{\pm}_{k\ell}, \quad \pbfhat^{\pm}_{k\ell} := \Vbf_\ell\zbfhat^{\pm}_{k\ell} \quad \ell=0,2,\ldots,L \\
    \qbfhat^{\pm}_{k\ell} &:= \Vbf_\ell\tran\zbfhat^{\pm}_{k\ell}, \quad \pbfhat^{\pm}_{k\ell} := \zbfhat^{\pm}_{k\ell} \quad \ell=1,3,\ldots,\Lm1.
\end{align}
\end{subequations}
Our goal is to describe the mean squared error of these estimates in the LSL.
To this end, similar to those in VAMP \citep{rangan2016vamp},
we introduce the concept of \emph{error functions}.
Let $\ell=2,4,\ldots,L-2$ be the index of a nonlinear stage
and suppose that we are given parameters $\gamma_{\lm1}^+$, $\gamma_{\ell}^-$ and $\tau^0_{\lm1}$.
Define a set of random variables $(R^+_{\lm1},Z^0_{\lm1},Z^0_\ell,R^-_\ell)$ by the Markov chain,
\begin{align} \label{eq:rznl}
\begin{split}
    R^+_{\lm1} &\sim \Norm(0,\tau^0_{\lm1}-1/\gamma_{\lm1}^+), \quad Z^0_{\lm1} \sim \Norm(R^+_{\lm1},1/\gamma_{\lm1}^-), \\
    Z^0_{\ell} &= \phi_\ell(Z^0_{\lm1},\Xi_\ell), \quad
    R^-_\ell \sim Z^0_\ell + \Norm(0,1/\gamma^-_\ell).
\end{split}
\end{align}
Thus, $Z^0_{\lm1}$ and $Z^0_\ell$ represent inputs and outputs of the activation function for the $\ell$-th stage
and $R^+_{\lm1}$ and $R^-_{\ell}$ are noisy observations of these inputs and outputs.  Define the
error functions
\beq \label{eq:Ecalnl}
    \Ecal_\ell^+(\gamma_{\lm1}^+,\gamma_\ell^-,\tau^0_{\lm1}) := \var(Z^0_{\ell}|R^+_{\lm1},R^-_{\ell}), \quad
    \Ecal_\ell^-(\gamma_{\lm1}^+,\gamma_\ell^-,\tau^0_{\lm1}) := \var(Z^0_{\lm1}|R^+_{\lm1},R^-_{\ell}),
\eeq
which represent the error variances in estimating the inputs and outputs.  For $\ell=0$, we can define
$\Ecal_0^+(\gamma_0^-)$ by dropping the terms associated with $R^+_{\lm1}$ and $Z^0_{\lm1}$.
For $\ell = L$, we define $\Ecal_{\Lm1}^-(\gamma_{\Lm1}^+,\tau^0_{\Lm1})$ by dropping the terms associated with
$R^-_\ell$.  Next, let $\ell=1,3,\ldots,\Lm1$ be the index of a linear stage, and consider a Markov chain,
\begin{align} \label{eq:rzlin}
    \bar{R}^+_{\lm1} &\sim \Norm(0,\tau^0_{\lm1}-1/\gamma_{\lm1}^+), \quad P^0_{\lm1} \sim \Norm(\bar{R}^+_{\lm1},1/\gamma_{\lm1}^-), \\
    Q^0_{\ell} &= \bar{S} P^0_{\lm1} + \bar{B} + \bar{\Xi}_\ell, \quad \bar{R}^-_\ell \sim Q^0_\ell + \Norm(0,1/\gamma^-_\ell),
\end{align}
which represents the inputs and outputs of a scalar linear channel with parameters $\bar{S}$, $\bar{B}$ and $\bar{\Xi}_\ell$
given from variables \eqref{eq:varinitlin}.  Define
\begin{align} \label{eq:Ecallin}
\begin{split}
    \Ecal_\ell^+(\gamma_{\lm1}^+,\gamma_\ell^-,\tau^0_{\lm1}) &:= \var(Q^0_{\ell}|\bar{R}^+_{\lm1},\bar{R}^-_{\ell},\bar{S}_\ell,\bar{B}_\ell), \\
    \Ecal_\ell^-(\gamma_{\lm1}^+,\gamma_\ell^-,\tau^0_{\lm1}) &:= \var(P^0_{\lm1}|\bar{R}^+_{\lm1},\bar{R}^-_{\ell},\bar{S}_\ell,\bar{B}_\ell),
\end{split}
\end{align}
Under these definitions, the SE equations for ML-VAMP are given in Algorithm~\ref{algo:mlvamp_se} which defines
a sequence of random variables and constants.

\begin{theorem} \label{thm:semlvamp}
Consider the outputs of the ML-VAMP algorithm, Algorithm~\ref{algo:ml-vamp}
and the corresponding outputs of the SE equations in Algorithm~\ref{algo:mlvamp_se}.
In addition to the assumptions in Section~\ref{sec:lsl}, assume:
\begin{enumerate}[(i)]
\item The constants  $\alphabar^{\pm}_{k\ell} \in (0,1)$ for all $k$ and $\ell$.
\item The activation functions $\phi_\ell(z_{\lm1},\xi_\ell)$ in \eqref{eq:phicomp} are pseudo-Lipschitz continuous of order two.
\item The component estimation functions $g^{\pm}_\ell(r^+_{\lm1},r^-_\ell,\gamma^+_{\lm1},\gamma^-_\ell)$ in \eqref{eq:gnl}
are uniformly Lipschitz continuous in $(r^+_{\lm1},r^-_\ell)$ at
$(\gamma^+_{\lm1},\gamma^-_\ell) = (\gammabar^+_{\lm1},\gammabar^-_\ell)$.
\end{enumerate}
Then, for any fixed iteration $k$ and index $\ell$,
\beq
    \lim_{N \arr \infty} (\gamma^{\pm}_{k\ell},\alpha^{\pm}_{k\ell},\eta^{\pm}_{k\ell}) =
    (\gammabar^{\pm}_{k\ell},\alphabar^{\pm}_{k\ell},\etabar^{\pm}_{k\ell}),
\eeq
almost surely, where the quantities on the right hand side are from the SE equations, Algorithm~\ref{algo:mlvamp_se}.
In addition the components of the transformed true vectors $\pbf^0_\ell$ and $\qbf^0_\ell$ and their estimates
$\pbfhat^{\pm}_{k\ell}$ and $\qbfhat^{\pm}_{k\ell}$ converge empirically as,
\beq
    \lim_{N \arr \infty} \left\{ (p^0_{\ell,n}, q^0_{\ell,n}, \phat^{\pm}_{k\ell,n}, \qhat^{\pm}_{k\ell,n})\right\}
        \PLeq (P^0_\ell,Q^0_\ell,\hat{P}^{\pm}_{k\ell},\hat{Q}^{\pm}_{k\ell}),
\eeq
where the random variable limits have moments,
\beq
    \Exp(P^0_\ell) = \Exp(Q^0_\ell) =\tau^0_\ell, \quad
    \Exp(\hat{P}^{\pm}_{k\ell}-P^0_\ell)^2 = \Exp(\hat{Q}^{\pm}_{k\ell}-Q^0_\ell)^2
    =\frac{1}{\etabar^{\pm}_{k\ell}}.
\eeq
\end{theorem}
\begin{proof} See Appendix~\ref{sec:semlpf}.
\end{proof}

Theorem~\ref{thm:semlvamp} shows that the components of the true signals $\pbf^0_\ell$ and $\qbf^0_\ell$
and the corresponding ML-VAMP estimates $\pbfhat^{\pm}_{k\ell}$ and $\qbfhat^{\pm}_{k\ell}$ converge empirically
to random variables $(P^0_\ell,Q^0_\ell,\hat{P}^{\pm}_{k\ell},\hat{Q}^{\pm}_{k\ell})$.  Appendix~\ref{sec:semlpf}
provides a complete description of the joint distribution of these variables and thus provides
an exact characterization of the asymptotic behavior of the true signal and their estimates.
In particular, the second moments of the true signals $P^0_\ell$ and $Q^0_\ell$ are given by the constants
$\tau^0_\ell$ computed in the first part of the SE equations in Algorithm~\ref{algo:mlvamp_se}.  This set of
operations is essentially an alternating sequence of scalar nonlinear and linear systems.
The error variances $\etabar^{\pm}_{k\ell}$
are then computed in the forward and backward passes of Algorithm~\ref{algo:mlvamp_se} by considering
a sequence of scalar estimation problems.  In summary, we have shown that
ML-VAMP is a computationally tractable algorithm for inference in MLPs that admits a simple and exact
characterization in the LSL.

\section{Numerical Experiments}
\label{sec:sim}
\paragraph*{Synthetic random network}  To illustrate the SE analysis, we first consider a
randomly generated neural network that follows the theoretical model of the paper.
Details are in Appendix~\ref{sec:simdetails}.
Briefly, the network input is an $N_0=20$ dimensional Gaussian unit noise vector $\zbf_0$.
The network then has three hidden layers with 100, 500 and 784 units (the same dimensions
will be used for the MNIST data set below).  The observed output is a random linear measurement
$\ybf = \Abf\zbf_5+\wbf$, where $\zbf_5$ is the 784-dimensional vector from the final hidden layer,
the matrix $\Abf$ is $M \x 784$, and $\wbf$ is Gaussian noise, set at 30 dB\@.  The number of measurements
$M$ is varied from 100 to 600\@.  To follow the theory, the weight matrices
are random Gaussian i.i.d.\ and the observation matrix $\Abf$ is a random orthogonally invariant matrix
with a fixed condition number $\kappa = 10$.  This model cannot be treated by the ML-AMP algorithm in \citep{manoel2017multi}.

\begin{figure}
\centering
\includegraphics[width=0.45\columnwidth]{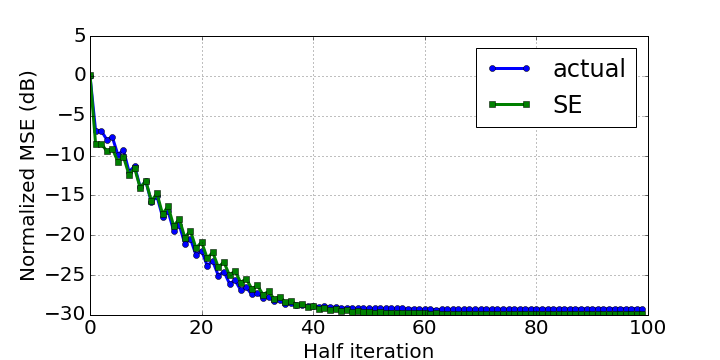}
\hfill
\includegraphics[width=0.45\columnwidth]{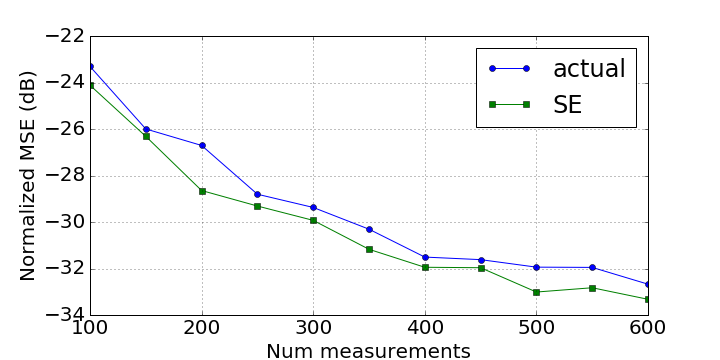}
\caption{Simulation with a randomly generated neural network.  Left panel:  Normalized mean squared error (NMSE)
for ML-VAMP and the predicted value from the state evolution as a function of the iteration with $M=300$ measurements.
Right panel:  Final NMSE for ML-VAMP and the SE prediction as a function of the measurements}
\label{fig:randmlp_sim}
\end{figure}

The left panel of Fig.~\ref{fig:randmlp_sim} shows the normalized mean squared error (NMSE) for the estimation of the inputs
to the networks $\zbf_0$ as a function of the iteration number for a fixed number of measurements $M=300$.
Also plotted is the state evolution (SE) prediction.
We see that the SE predicts the ML-VAMP behavior remarkably well, within approximately 1~dB\@.  The right panel shows
the NMSE after 50 iterations (100 half-iterations) for various values of $M$.  We again see an excellent agreement between
the actual values and the SE predictions.

\begin{figure}
\centering
\includegraphics[width=0.6\columnwidth]{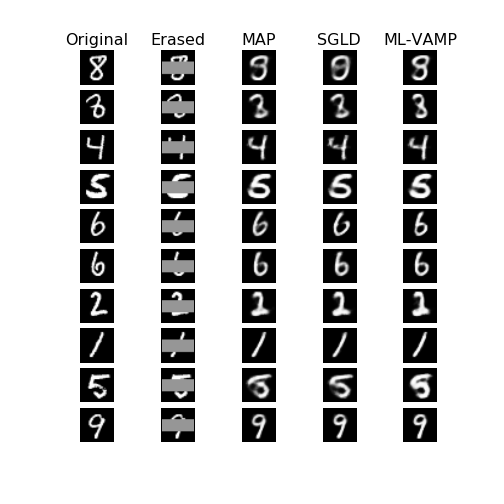}
\caption{Inpainting of handwritten digits using MAP estimation,
stochastic gradient Langevin dynamics (SGLD) and ML-VAMP.}
\label{fig:mnist}
\end{figure}

\paragraph*{MNIST inpainting}  To demonstrate the feasibility of ML-VAMP on a real dataset,
we used the algorithm for inpainting on the MNIST dataset, as
considered in
\citep{yeh2016semantic,bora2017compressed,tramel2016inferring}.
The MNIST dataset consists of 28 $\times$ 28 = 784 pixel images of hand-written digits
as shown in the first column of Fig.~\ref{fig:mnist}.
Following \citep{kingma2013auto},
a generative model for these digits
was trained using a variational autoencoder (VAE), so that each image $\xbf$
is  modeled as the output  of an $L$-stage neural network.
In this experiment, we used a single layer network with 20 input units,
400 hidden units and 784 output units corresponding to the dimension of the images
-- details of the network, training procedure and other simulation details are
given in Appendix~\ref{sec:simdetails}.  For each image $\xbf$,
we then created an occluded image, $\ybf$, by removing the rows 10--20 of
the image as  shown in the second column of Fig.~\ref{fig:mnist}.
The inpainting problem is to recover the original image $\xbf$
from the occluded image $\ybf$.

To perform ML-VAMP for reconstruction,
we first observe that the occluded image $\ybf$
is the output of the same neural network that generates $\xbf$, but with the
occuluded pixels removed in the final layer.
Using the ML-VAMP algorithm, we can then estimate the values of $\zbf_0$, the input to the
neural network.  Once $\zbf_0$ is estimated, we can estimate the original image
$\xbf$ by running the input $\zbf_0$ through the original network.
The resulting reconstructed images are shown in the final column of Fig.~\ref{fig:mnist},
which displays the reconstructed images after only 20 iterations of ML-VAMP.

For comparison, we have also shown the MAP estimates, which are the images
$\xbf$ that maximizes the posterior density $p(\xbf|\ybf)$.  As described in
Appendix~\ref{sec:simdetails}, the MAP estimates can be computed
using numerical optimization of the posterior density as performed in \citep{yeh2016semantic,bora2017compressed}.  Fig.~\ref{fig:mnist} also shows
the posterior mean $\Exp(\xbf|\ybf)$ as estimated via Stochastic Gradient Langevin Dynamics
(SGLD) \citep{welling2011bayesian} -- also see Appendix~\ref{sec:simdetails}.
We see that visually, the ML-VAMP, MAP and SGLD estimates are extremely similar.

In addition, the ML-VAMP algorithm is significantly faster:
ML-VAMP was performed for only 20 iterations,
while MAP used 500 iterations and SGLD used 10000.  Thus, the experiment
suggests that, in addition to its theoretical guarantees, ML-VAMP
may also be a computationally simpler approach for reconstruction.
Of course, much further experimentation on more complex data sets would be needed
to evaluate its practical applicability.

\section{Conclusions}  

While deep networks have been remarkably successful in a range of challenging machine
learning problems, their success is still not fully understood at a theoretical level.
We have presented a principled and
computationally tractable method for inference
in deep networks whose performance can be rigorously characterized in certain high-dimensional
random settings.  The proposed method is based on AMP techniques
which have proven to be information theoretically optimal in closely related problems.
It is possible that similar optimality results may be true for ML-VAMP as well.
For practical applications, the proposed method needs much further study, but its fast convergence
suggests that it may be useful in problems outside the theoretical model as well.
Going forward, the natural extension of these results would be to consider learning as well.
AMP and VAMP techniques have been combined with EM learning in 
\citep{krzakala2012statistical,vila2013expectation,KamRanFU:12-IT,fletcher2016emvamp} and it is possible that these
methods may be able to be used in the context of learning deep generative models as well.

\appendix

\section{Empirical Convergence of Vectors} \label{sec:empconv}

We review some definitions from the Bayati-Montanari paper~\citep{BayatiM:11}
and the original VAMP paper~\citep{rangan2016vamp} since we will use the same
analysis framework in this paper.
Let $\xbf = (\xbf_1,\ldots,\xbf_N)$ be a block vector with components $\xbf_n \in \R^r$
for some $r$.  Thus, the vector $\xbf$
is a vector with dimension $rN$.
Given any function $g:\R^r \arr \R^s$, we define the
\emph{componentwise extension} of $g(\cdot)$ as the function,
\beq \label{eq:gcomp}
    \gbf(\xbf) := (g(\xbf_1),\ldots,g(\xbf_N)) \in \R^{Ns}.
\eeq
That is,
$\gbf(\cdot)$
applies the function
$g(\cdot)$
on each $r$-dimensional component.
Similarly,
we say $\gbf(\xbf)$ \emph{acts componentwise} on $\xbf$ whenever it is of the form \eqref{eq:gcomp}
for some function $g(\cdot)$.

Next consider a sequence of block vectors of growing dimension,
\[
    \xbf(N) = (\xbf_1(N),\ldots,\xbf_N(N)),
\qquad
N=1,\,2,\,\ldots,
\]
where each component $\xbf_n(N) \in \R^r$.
In this case, we will say that
$\xbf(N)$ is a \emph{block vector sequence that scales with $N$
under blocks $\xbf_n(N) \in \R^r$.}
When $r=1$, so that the blocks are scalar, we will simply say that
$\xbf(N)$ is a \emph{vector sequence that scales with $N$}.
Such vector sequences can be deterministic or random.
In most cases, we will omit the notational dependence on $N$ and simply write $\xbf$.

Now, given $p \geq 1$,
a function $f:\R^r \arr \R^s$ is called \emph{pseudo-Lipschitz continuous of order $p$},
if there exists a constant $C > 0$ such that for all $\xbf_1,\xbf_2 \in\R^r$,
\[
    \| f(\xbf_1)- f(\xbf_2) \| \leq C\|\xbf_1-\xbf_2\|\left[ 1 + \|\xbf_1\|^{p-1}
    + \|\xbf_2\|^{p-1} \right].
\]
Observe that in the case $p=1$, pseudo-Lipschitz continuity reduces to
the standard Lipschitz continuity.
Given $p \geq 1$, we will say that the block vector sequence $\xbf=\xbf(N)$
\emph{converges empirically with $p$-th order moments} if there exists a random variable
$X \in \R^r$ such that
\begin{enumerate}[(i)]
\item $\Exp\|X\|_p^p < \infty$; and
\item for any $f : \R^r \arr \R$ that is pseudo-Lipschitz continuous of order $p$,
\beq \label{eq:PLp-empirical}
    \lim_{N \arr \infty} \frac{1}{N} \sum_{n=1}^N f(\xbf_n(N)) = \Exp\left[ f(X) \right].
\eeq
\end{enumerate}
In \eqref{eq:PLp-empirical}, we have
the empirical mean of the components $f(\xbf_n(N))$
of the componentwise extension $\fbf(\xbf(N))$
converging to the expectation $\Exp[ f(X) ]$.
In this case, with some abuse of notation, we will write
\beq \label{eq:plLim}
    \lim_{N \arr \infty} \left\{ \xbf_n \right\} \stackrel{PL(p)}{=} X,
\eeq
where, as usual, we have omitted the dependence on $N$ in $\xbf_n(N)$.
Importantly, empirical convergence can de defined on deterministic vector sequences,
with no need for a probability space.  If $\xbf=\xbf(N)$ is a random vector sequence,
we will often require that the limit \eqref{eq:plLim} holds almost surely.

We conclude with one final definition.
Let $\phibf(\rbf,\gamma)$ be a function on $\rbf \in \R^s$ and $\gamma \in \R$.
We say that $\phibf(\rbf,\gamma)$ is \emph{uniformly Lipschitz continuous} in $\rbf$
at $\gamma=\gammabar$ if there exists constants
$L_1$ and $L_2 \geq 0$ and an open neighborhood $U$ of $\gammabar$, such that
\beq \label{eq:unifLip1}
    \|\phibf(\rbf_1,\gamma)-\phibf(\rbf_2,\gamma)\| \leq L_1\|\rbf_1-\rbf_2\|,
\eeq
for all $\rbf_1,\rbf_2 \in \R^s$ and $\gamma \in U$; and
\beq \label{eq:unifLip2}
    \|\phibf(\rbf,\gamma_1)-\phibf(\rbf,\gamma_2)\| \leq L_2\left(1+\|\rbf\|\right)|\gamma_1-\gamma_2|,
\eeq
for all $\rbf \in \R^s$ and $\gamma_1,\gamma_2 \in U$.

\section{Linear Estimation Functions}
\label{sec:linestim}

Consider a linear factor node between two variables $\zbf_\ell$ and $\zbf_{\lm1}$
related by the linear relation \eqref{eq:nnlin}, which
corresponds to a Gaussian log conditional density,
\beq \label{eq:logplin}
    -\ln p(\zbf_\ell|\zbf_{\lm1}) =
        \frac{\nu_\ell}{2}\|\zbf_\ell-\Wbf_\ell \zbf_{\lm1} + \bbf_\ell \|^2 + \mbox{const}.
\eeq
Applying \eqref{eq:logplin},
the energy function \eqref{eq:Hdef} is the quadratic
\beq \label{eq:Hlin}
    H(\zbf_{\lm1},\zbf_{\ell}) := \frac{\nu_\ell}{2}
        \|\zbf_\ell-\Wbf_\ell \zbf_{\lm1} + \bbf_\ell\|^2 +
        \frac{\gamma^-_{\ell}}{2}\|\zbf_\ell-\rbf^-_\ell\|^2
        + \frac{\gamma^+_{\lm1}}{2}\|\zbf_{\lm1}-\rbf^-_{\lm1} \|^2,
\eeq
and the belief estimate $b(\cdot)$ in \eqref{eq:bdef} is the Gaussian density
\beq \label{eq:blinz}
    b(\zbf_{\lm1},\zbf_{\ell}|\rbf_\ell,\rbf_{\lm1},\gamma_{\ell}^-,\gamma_{\lm1}^+) \propto
        \exp\left[ -H(\zbf_{\lm1},\zbf_\ell) \right].
\eeq
As shown in \eqref{eq:gexp},
the estimation functions $\gbf^{\pm}_\ell(\cdot)$
are given by the expectation of $\zbf_{\lm1}$ and $\zbf_{\ell}$ with respect to the density
$b(\cdot)$ in \eqref{eq:blinz}.
Since $b(\cdot)$ is Gaussian, these expectations can be computed via a least-squares problem.

The solution to this least-squares problem is simplest to consider using the SVD in \eqref{eq:WSVD}.
Specifically, given the SVD, define the transformed variables
\beq \label{eq:udeflin}
    \ubf_{\ell} := \Vbf_{\ell}\tran \zbf_{\ell}, \quad
    \ubf_{\lm1} := \Vbf_{\lm1}\zbf_{\lm1}, \quad
    \bar{\ubf}_{\ell} := \Vbf_{\ell}\tran \rbf_{\ell}^-, \quad
    \bar{\ubf}_{\lm1} := \Vbf_{\lm1}\rbf_{\lm1}^+, \quad
    \bar{\bbf}_\ell   := \Vbf_{\ell}\tran \bbf_\ell.
\eeq
To compute the expectations on $(\zbf_{\lm1},\zbf_\ell)$, we will first
compute the expectations on $(\ubf_{\lm1},\ubf_\ell)$ and then use the transformations \eqref{eq:udeflin}
to find the expectations on $(\zbf_{\lm1},\zbf_\ell)$.
Using the transformations in \eqref{eq:udeflin}, it can be verified that
$\ubf=(\ubf_{\lm1},\ubf_{\ell})$ has a Gaussian probability density,
\beq \label{eq:bulin}
    b_u(\ubf) \propto \exp\left[ -H_u(\ubf) \right],
\eeq
where $H_u(\cdot)$ is the energy function
\beq \label{eq:Hlinu}
    H_u(\ubf_{\lm1},\ubf_\ell) = \frac{\nu_\ell}{2}\|\ubf_{\ell}-\Sigmabf_\ell\ubf_{\lm1} + \bar{\bbf}_\ell\|^2
   +   \frac{\gamma^-_{\ell}}{2}\|\ubf_\ell-\bar{\ubf}_\ell\|^2
        + \frac{\gamma^+_{\lm1}}{2}\|\ubf_{\lm1}-\bar{\ubf}_{\lm1} \|^2 .
\eeq
This density has mean and covariance
\beq \label{eq:uestlin}
    \Exp(\ubf|b_u) = \Qbf^{-1}\cbf, \quad \Cov(\ubf|b_u) = \Qbf^{-1},
\eeq
where
\beq
    \Qbf :=
        \left[ \begin{array}{cc}
        \gamma_{\lm1}^+\Ibf + \nu_\ell \Sigmabf_\ell\tran\Sigmabf_\ell & -\nu_\ell\Sigmabf_\ell\tran \\
        -\nu_\ell \Sigmabf_\ell & (\gamma_\ell^- +\nu_\ell)\Ibf \end{array} \right], \quad
    \cbf := \left[ \begin{array}{c}
        \gamma_{\lm1}^+ \bar{\ubf}_{\lm1} - \nu_\ell \Sigmabf_\ell\tran \bar{\bbf}_\ell \\
        \gamma_{\ell}^- \bar{\ubf}_{\ell} + \nu_\ell \bar{\bbf}_\ell
        \end{array}
        \right].
\eeq
To evaluate the expectation and covariance in \eqref{eq:uestlin},
define the function $G_\ell(\cdot)$ with two outputs
\beq \label{eq:gdeflin}
    G_\ell(\bar{u}_{\lm1},\bar{u}_\ell,s_\ell,\bar{b}_\ell,\gamma_{\lm1}^+,\gamma_{\ell}^-) :=
    \left[ \begin{array}{c}
        G_\ell^-(\bar{u}_{\lm1},\bar{u}_\ell,s_\ell,\bar{b}_\ell,\gamma_{\lm1}^+,\gamma_{\ell}^-) \\
        G_\ell^+(\bar{u}_{\lm1},\bar{u}_\ell,s_\ell,\bar{b}_\ell,\gamma_{\lm1}^+,\gamma_{\ell}^-)
        \end{array}
        \right] =  \Pbf^{-1}\dbf,
\eeq
where $\Pbf \in \R^{2 \x 2}$ and $\dbf \in \R^{2 \x 1}$ given by
\beq
   \Pbf :=
        \left[ \begin{array}{cc}
        \gamma_{\lm1}^+ + \nu_\ell s_\ell^2 & -\nu_\ell s_\ell \\
        -\nu_\ell s_\ell & \gamma_\ell^- +\nu_\ell \end{array} \right], \quad
    \dbf := \left[ \begin{array}{c}
        \gamma_{\lm1}^+ \bar{u}_{\lm1} - \nu_\ell s_\ell \bar{b}_\ell \\
        \gamma_{\ell}^- \bar{u}_{\ell} + \nu_\ell \bar{b}_\ell
        \end{array}
        \right].
\eeq
Since $\Sigmabf_\ell$ has the block diagonal structure in \eqref{eq:WSVD},
the expectations in \eqref{eq:uestlin} are given by
\begin{subequations} \label{eq:Glinvec}
\begin{align}
    \Exp(\ubf_{\lm1}|b_u) = \Gbf_\ell^-(\bar{\ubf}_{\lm1},\bar{\ubf}_\ell,\sbf_\ell,\bar{\bbf}_\ell,
        \gamma_{\lm1}^+,\gamma_{\ell}^-), \\
    \Exp(\ubf_{\ell}|b_u) = \Gbf_\ell^+(\bar{\ubf}_{\lm1},\bar{\ubf}_\ell,\sbf_\ell,\bar{\bbf}_\ell,
        \gamma_{\lm1}^+,\gamma_{\ell}^-),
\end{align}
\end{subequations}
where the functions $\Gbf_\ell^{\pm}(\cdot)$ are the componentwise extensions
(see Appendix~\ref{sec:empconv}) of $G_\ell^{\pm}(\cdot)$, meaning
\begin{align} \label{eq:Glincomp}
    \left[ \Gbf_\ell^{\pm}(\bar{\ubf}_{\lm1},\bar{\ubf}_{\ell},
        \sbf_\ell,\bar{\bbf}_\ell,\gamma_{\lm1}^+,\gamma_{\ell}^-) \right]_n
        = G_\ell^{\pm}(\bar{u}_{\lm1,n},\bar{u}_{\ell n},s_{\ell n}
        \bar{b}_{\ell n},\gamma_{\lm1}^+,\gamma_{\ell}^-),
\end{align}
for each component $n$.

One slight technicality in the componentwise definition \eqref{eq:Glincomp} is that
$\bar{\ubf}_{\lm1}$, $\bar{\ubf}_{\ell}$ and $\sbf_\ell$ may have different dimensions:
\[
    \bar{\ubf}_{\lm1} \in \R^{N_{\lm1}}, \quad
    \bar{\ubf}_{\ell} \in \R^{N_{\ell}}, \quad
    \sbf_\ell \in \R^{R_\ell}.
\]
We define the outputs of $\Gbf_\ell^+(\cdot)$ and $\Gbf_\ell^-(\cdot)$
as having output dimensions $N_{\lm1}$ and $N_\ell$, respectively.
For $n \leq R_\ell$,
we can use the formula \eqref{eq:Glincomp} since $N_{\lm1}, N_\ell \geq R_\ell$ so all
three terms $\bar{u}_{\lm1,n},\bar{u}_{\ell n},s_{\ell n}$ are defined for $n \leq R_\ell$.
For $n > R_\ell$, we use the convention that $s_{\ell n}=0$.
In this case, it can be verified from \eqref{eq:gdeflin} that when $s_{\ell n}=0$,
\begin{align*}
     G_\ell^-(\bar{u}_{\lm1,n},\bar{u}_{\ell n},s_{\ell n},
        \bar{b}_{\ell n},\gamma_{\lm1}^+,\gamma_{\ell}^-)
     &= \bar{u}_{\lm1}, \\
     G_\ell^+(\bar{u}_{\lm1,n},\bar{u}_{\ell n},s_{\ell n},
        \bar{b}_{\ell n},\gamma_{\lm1}^+,\gamma_{\ell}^-)
     &= \frac{\gamma_{\ell}^- \bar{u}_{\ell} + \nu_\ell \bar{b}_{\ell n}}{\gamma_{\ell}^-  + \nu_\ell}.
\end{align*}
Thus, $G_\ell^-(\cdot)$ does not depend on $\bar{u}_{\ell n}$ and $G_\ell^+(\cdot)$
does not depend on $\bar{u}_{\lm1,n}$.

Using the definitions \eqref{eq:udeflin}, we can compute the desired estimation functions
\begin{align*}
    \MoveEqLeft \Exp(\zbf_\ell|b) = \Vbf_\ell \Exp(\ubf_\ell|b_u)
    = \Vbf_\ell\Gbf_\ell^+(\bar{\ubf}_{\lm1},\bar{\ubf}_\ell,\sbf,\vbf,
        \gamma_{\lm1}^+,\gamma_{\ell}^-) \nonumber \\
    &= \Vbf_\ell\Gbf_\ell^+(\Vbf_{\lm1}\rbf^+_{\lm1},\Vbf_\ell\tran \rbf^-_{\ell},
        \sbf_\ell,\bar{\bbf}_\ell,\gamma_{\lm1}^+,\gamma_{\ell}^-).
\end{align*}
Therefore, from \eqref{eq:gexp}, the linear estimation function is given by
\beq \label{eq:glinGp}
    \gbf^+_\ell(\rbf^+_{\lm1},\rbf^-_{\ell},\gamma_{\lm1}^+,\gamma_{\ell}^-) =
    \Vbf_\ell\Gbf_\ell^+(\Vbf_{\lm1}\rbf^+_{\lm1},\Vbf_\ell\tran \rbf^-_{\ell},
        \sbf_\ell,\bar{\bbf}_\ell,\gamma_{\lm1}^+,\gamma_{\ell}^-),
\eeq
where we have suppressed the dependence on $\sbf_\ell$ and
$\bar{\bbf}_\ell$ on the left hand side.
Similarly, one can show
\beq \label{eq:glinGn}
    \gbf^-_\ell(\rbf^+_{\lm1},\rbf^-_{\ell},
        \gamma_{\lm1}^+,\gamma_{\ell}^-) =
    \Vbf_{\lm1}\tran\Gbf_\ell^-(\Vbf_{\lm1}\rbf^+_{\lm1},\Vbf_\ell\tran \rbf^-_{\ell},
        \sbf_\ell,\bar{\bbf}_\ell,\gamma_{\lm1}^+,\gamma_{\ell}^-).
\eeq

For the derivative $\alpha_{k \ell}$ in line \ref{line:alphap},
observe that
\begin{align}
    \alpha_{k \ell}^+ &\stackrel{(a)}{=}
        \bkt{ \partial \gbf^+_\ell(\rbf^+_{k,\lm1},\rbf^-_{k \ell},
        \sbf_\ell,\gamma_{k\lm1}^+,\gamma_{k\ell}^-) / \partial \rbf^-_{k \ell} } \nonumber \\
        &\stackrel{(b)}{=} \frac{1}{N_\ell} \Tr\left[ \Vbf_\ell
        \frac{\partial \Gbf_\ell^+(\bar{\ubf}_{k,\lm1},\bar{\ubf}_{\ell},
        \sbf_\ell,\bar{\bbf}_\ell,\gamma_{k,\lm1}^+,\gamma_{k\ell}^-)}
        {\partial \bar{\ubf}_{k\ell}} \Vbf_{\ell}\tran \right] \nonumber \\
        &\stackrel{(c)}{=} \frac{1}{N_\ell} \Tr\left[
        \frac{\partial \Gbf_\ell^+(\bar{\ubf}_{k,\lm1},\bar{\ubf}_{k\ell},
        \sbf_\ell,\bar{\bbf}_\ell,\gamma_{k,\lm1}^+,\gamma_{k\ell}^-)}
        {\partial \bar{\ubf}_{k\ell}}  \right] \nonumber \\
        &\stackrel{(d)}{=} \bktAuto{
        \frac{\partial \Gbf_\ell^+(\bar{\ubf}_{k,\lm1},\bar{\ubf}_{k\ell},
        \sbf_\ell,\bar{\bbf}_\ell,\gamma_{k,\lm1}^+,\gamma_{k\ell}^-)}
        {\partial \bar{\ubf}_{k\ell}}  }, \label{eq:alphapG}
\end{align}
where (a) follows from line~\ref{line:alphap} of Algorithm~\ref{algo:ml-vamp};
in (b), we have used \eqref{eq:glinGp} and set $\bar{\ubf}_{k,\lm1} = \Vbf_{\lm1}\rbf^+_{k,\lm1}$
and $\bar{\ubf}_{k\ell} = \Vbf_{\ell}\tran\rbf^+_{k\ell}$;
(c) follows from invariance of the trace of a product to cyclic permutation,
since $\Vbf_\ell\tran\Vbf_\ell = \Ibf$; and (d) follows from the definition of
the $\bkt{\cdot}$ operator.
Similarly, we can show
\beq \label{eq:alphanG}
        \alpha_{k \ell}^- = \bktAuto{
        \frac{\partial \Gbf_\ell^-(\bar{\ubf}_{k,\lm1},\bar{\ubf}_{k\ell},
        \sbf_\ell,\bar{\bbf}_\ell,\gamma_{k,\lm1}^+,\gamma_{k\ell}^-)}
        {\partial \bar{\ubf}_{k,\lm1}  } }.
\eeq

\section{General Multi-Layer Recursions}

To analyze Algorithm~\ref{algo:ml-vamp}, we consider a more general class
of recursions as shown in Algorithm~\ref{algo:gen}.
The Gen-ML Algorithm generates vectors $\qbf^{\pm}_{k\ell}$ and $\pbf^{\pm}_{k\ell}$
via a sequence of forward and backward passes through a multi-layer system.
As we will show below, we will associate $\qbf^{\pm}_{k\ell}$ and $\pbf^{\pm}_{k\ell}$
with certain error terms in the ML-VAMP algorithm.
The functions that update $\fbf_{k\ell}^{\pm}(\cdot)$ that produce the vectors
$\qbf^{\pm}_{k\ell}$ and $\pbf^{\pm}_{k\ell}$ will be called the \emph{vector update functions}.

To account for the effect of the parameters $\gamma^{\pm}_{k\ell}$ and $\alpha^{\pm}_{k\ell}$
in ML-VAMP, the Gen-ML algorithm describes the parameter update through a sequence of
\emph{parameter lists} $\Lambda^{\pm}_{k\ell}$.
The parameter lists are ordered lists of parameters that accumulate as the
algorithm progresses.  They are initialized with $\Lambda^-_{01}$ in
line~\ref{line:laminit_gen}.  Then, as the algorithm progresses, new parameters $\lambda^{\pm}_{k\ell}$
are computed and then added to the lists in lines~\ref{line:lamp0_gen}, \ref{line:lamp_gen}, \ref{line:lamL_gen}
and \ref{line:lamn_gen}.  The vector update functions $\fbf_{k\ell}^{\pm}(\cdot)$ may depend on any
sets of parameters accumulated in the parameter list.

In lines~\ref{line:mup0_gen}, \ref{line:mup_gen}, \ref{line:muL_gen} and \ref{line:mun_gen},
the new parameters $\lambda_{k\ell}^{\pm}$ are computed by:
(1) computing average values $\mu_{k\ell}^{\pm}$ of componentwise functions $\varphibf^{\pm}_{k\ell}(\cdot)$;
and (2) taking functions $T^{\pm}_{k\ell}(\cdot)$ of the average values $\mu_{k\ell}^{\pm}$.
Since the average values $\mu_{k\ell}^{\pm}$ represent statistics on the components of
$\varphibf^{\pm}_{k\ell}(\cdot)$, we will call $\varphibf^{\pm}_{k\ell}(\cdot)$ the \emph{parameter statistic
functions}.  We will call the $T^{\pm}_{k\ell}(\cdot)$ the \emph{parameter update functions}.
We will show below that the updates for the parameters $\gamma^{\pm}_{k\ell}$ and $\alpha^{\pm}_{k\ell}$
can be written in this form.

\begin{algorithm}[t]
\caption{General Multi-Layer Recursion (Gen-ML)}
\begin{algorithmic}[1]  \label{algo:gen}
\REQUIRE{Vector update functions $\fbf^\pm_{k\ell}(\cdot)$,
parameter statistic functions $\varphibf^\pm_{k\ell}(\cdot)$,
parameter update functions $T^{\pm}_{k\ell}(\cdot)$,
orthogonal matrices $\Vbf_\ell$,
disturbance vectors $\wbf^\pm_\ell$. }

\STATE{// Initialization }
\STATE{Initialize parameter list $\Lambda_{01}^-$, and vectors $\pbf_0^0$ and $\qbf_{0\ell}^-$, $\ell=0,\ldots,\Lm1$}
    \label{line:laminit_gen}
\STATE{$\qbf^0_0 = \fbf^0_0(\wbf_0), \quad \pbf^0_0 = \Vbf_0\qbf^0_0$} \label{line:q00init_gen}
\FOR{$\ell=1,\ldots,\Lm1$}
    \STATE{$\qbf^0_\ell = \fbf^0_\ell(\pbf^0_{\lm1},\wbf_\ell, \Lambda_{01}^-)$ }
    \label{line:q0init_gen}
    \STATE{$\pbf^0_\ell = \Vbf_\ell\qbf^0_\ell$ }  \label{line:p0init_gen}
\ENDFOR
\STATE{}
\FOR{$k=0,1,\dots$}
    \STATE{// Forward Pass }
    \STATE{$\lambda^+_{k0} = T_{k0}^+(\mu^+_{k0},\Lambda_{0k}^-), \quad
        \mu^+_{k0} = \bkt{\varphibf_{k0}^+(\qbf_{k0}^-,\wbf_0,\Lambda_{0k}^-)}$}    \label{line:mup0_gen}
    \STATE{$\Lambda_{k0}^+ = (\Lambda_{k1}^-,\lambda^+_{k0})$} \label{line:lamp0_gen}
    \STATE{$\qbf_{k0}^+ = \fbf^+_{k0}(\qbf_{k0}^-,\wbf_0,\Lambda^+_{k0})$}  \label{line:q0_gen}
    \STATE{$\pbf_{k0}^+ = \Vbf_0\qbf_{k0}^+$} \label{line:p0_gen}
    \FOR{$\ell=1,\ldots,L-1$}
        \STATE{$\lambda^+_{k\ell} = T_{k\ell}^+(\mu^+_{k\ell},\Lambda_{k,\lm1}^+), \quad
            \mu^+_{k\ell} = \bkt{\varphibf_{k\ell}^+(\pbf^0_{\lm1},\pbf^+_{k,\lm1},\qbf_{k\ell}^-,\wbf_\ell,\Lambda_{k,\lm1}^+)}$}    \label{line:mup_gen}
        \STATE{$\Lambda_{k\ell}^+ = (\Lambda_{k,\lm1}^+,\lambda^+_{k\ell})$}
            \label{line:lamp_gen}
        \STATE{$\qbf_{k\ell}^+ = \fbf^+_{k\ell}(\pbf^0_{\lm1},\pbf^+_{k,\lm1},\qbf_{k\ell}^-,\wbf_\ell,\Lambda^+_{k\ell})$}
            \label{line:qp_gen}
        \STATE{$\pbf_{k\ell}^+ = \Vbf_{\ell}\qbf_{k\ell}^+$}   \label{line:pp_gen}
    \ENDFOR
    \STATE{}

    \STATE{// Reverse Pass }
    \STATE{$\lambda^-_{\kp1,L} = T_{kL}^-(\mu^-_{kL},\Lambda_{k,\Lm1}^+), \quad
        \mu^-_{kL} = \bkt{\varphibf_{kL}^-(\pbf_{k,\Lm1}^+,\wbf_L,\Lambda_{k,\Lm1}^+)}$}    \label{line:muL_gen}
    \STATE{$\Lambda_{\kp1,L}^- = (\Lambda_{k,\Lm1}^+,\lambda^+_{\kp1,L})$} \label{line:lamL_gen}
    \STATE{$\pbf_{\kp1,\Lm1}^- = \fbf^-_{kL}(\pbf^0_{\Lm1},\pbf_{k,\Lm1}^+,\wbf_L,\Lambda^-_{\kp1,L})$}  \label{line:pL_gen}
    \STATE{$\qbf_{\kp1,\Lm1}^- = \Vbf_{\Lm1}\tran\pbf_{\kp1,\Lm1}$} \label{line:qL_gen}
    \FOR{$\ell=\Lm1,\ldots,1$}
        \STATE{$\lambda^-_{\kp1,\ell} = T_{k\ell}^-(\mu^-_{k\ell},\Lambda_{\kp1,\lp1}^-), \quad
            \mu^-_{k\ell} =
            \bkt{\varphibf_{k\ell}^-(\pbf_{\lm1}^0,\pbf_{k,\lm1}^+,\qbf_{\kp1,\ell}^-,\wbf_\ell,\Lambda_{\kp1,\lp1}^+)}$}    \label{line:mun_gen}
        \STATE{$\Lambda_{\kp1,\ell}^- = (\Lambda_{\kp1,\lp1}^-,\lambda^-_{\kp1,\ell})$} \label{line:lamn_gen}
        \STATE{$\pbf_{\kp1,\lm1}^- =
        \fbf^-_{k\ell}(\pbf_{\lm1}^0,\pbf^+_{k,\lm1},\qbf_{\kp1,\ell}^-,\wbf_\ell,\Lambda^-_{k\ell})$}
            \label{line:pn_gen}
        \STATE{$\qbf_{\kp1,\lm1}^- = \Vbf_{\lm1}\tran\pbf_{\kp1,\lm1}^-$}   \label{line:qn_gen}
    \ENDFOR

\ENDFOR
\end{algorithmic}
\end{algorithm}

Similar to our analysis of the ML-VAMP Algorithm,
we consider the following large-system limit (LSL) analysis of Gen-ML.
Specifically, we consider a sequence of runs of the recursions indexed by $N$.
For each $N$, let $N_\ell = N_\ell(N)$ be the dimension of the signals $\pbf_\ell^{\pm}$ and $\qbf_\ell^\pm$
as we assume that $\lim_{N \arr \infty} N_\ell/N$ is a constant so that $N_\ell$ scales linearly with $N$.
We then make the following assumptions.

\begin{assumption} \label{as:gen} For the vectors in the Gen-ML Algorithm (Algorithm~\ref{algo:gen}),
we assume:
\begin{enumerate}[(a)]
\item The components of the initial conditions
$\qbf_{0\ell}^-$, and disturbance vectors $\wbf_\ell$ converge jointly empirically with limits,
\beq \label{eq:qwinitlim}
    \lim_{N \arr \infty} \{q_{0\ell,n}^-\} \PLeq Q_{0\ell}^-, \quad
    \lim_{N \arr \infty} \{ w_{\ell,n} \} \PLeq W_\ell,
\eeq
where $Q_{0\ell}^-$ and $W_\ell$ are random variables such that $(Q_{00}^-,\cdots,Q^-_{0,\Lm1})$
is a jointly Gaussian random vector.  Also, for $\ell=0,\ldots,\Lm1$, $W_\ell$ and $Q_{0\ell}^-$ are independent.
We also assume that the initial parameter list converges almost surely as
\beq \label{eq:Lambar01lim}
    \lim_{N \arr \infty} \Lambda_{01}^- = \Lambdabar_{01}^-,
\eeq
to some list $\Lambdabar_{01}^-$.  The limit \eqref{eq:Lambar01lim} means
means that every element in the list $\lambda \in \Lambda_{01}^-$ converges to a limit
$\lambda \arr \lambdabar$ as $N \arr \infty$ almost surely.

\item The matrices $\Vbf_\ell$ are Haar distributed on the set of $N_\ell \x N_\ell$ orthogonal matrices and are
independent from one another and from the vectors $\pbf^0_0$,
$\qbf_{0\ell}^-$, disturbance vectors $\wbf_\ell$.

\item The vector update functions $\fbf_{k\ell}^\pm(\cdot)$
and parameter update functions $\varphibf_{k\ell}^\pm(\cdot)$ act componentwise.  For example,
in the forward pass, at each stage $\ell$, we assume that for each output component $n$,
\begin{align*}
    \left[ \fbf^+_{k\ell}(\pbf^0_{\lm1},\pbf^+_{k,\lm1},\qbf_{k\ell}^-,\wbf_\ell,\Lambda^+_{k\ell}) \right]_n
    = f^+_{k\ell}(p^0_{\lm1,n},p^+_{k,\lm1,n},q_{k\ell,n}^-,w_{\ell,n},\Lambda^+_{k\ell}) \\
    \left[ \varphibf^{\pm}_{k\ell}(\pbf^0_{\lm1},\pbf^+_{k,\lm1},\qbf_{k\ell}^-,\wbf_\ell,\Lambda^+_{k\ell}) \right]_n
    = \varphi^+_{k\ell}(p^0_{\lm1,n},p^+_{k,\lm1,n},q_{k\ell,n}^-,w_{\ell,n},\Lambda^+_{k\ell}),
\end{align*}
for some scalar-valued functions $f^+_{k\ell}(\cdot)$ and $\varphi^+_{k\ell}(\cdot)$.
Similar definitions apply in the reverse directions and for the initial update functions $f^0_\ell(\cdot)$.
We will call $f^{\pm}_{k\ell}(\cdot)$ the vector update component
functions and $\varphi^{\pm}_{k\ell}(\cdot)$ the parameter update component functions.
\end{enumerate}
\end{assumption}

\begin{algorithm}[t]
\caption{Gen-ML State Evolution}
\begin{algorithmic}[1]  \label{algo:gen_se}

\REQUIRE{Vector update component functions $f^0_\ell(\cdot)$ and $f^\pm_{k\ell}(\cdot)$,
parameter statistic component functions $\varphi^\pm_{k\ell}(\cdot)$,
parameter update functions $T^{\pm}_{k\ell}(\cdot)$}

\STATE{}
\STATE{// Initial pass}
\STATE{Initial random variables:  $W_\ell$, $Q_{0\ell}^-$, $\ell=0,\ldots,\Lm1$}
    \label{line:qinit_se_gen}
\STATE{Initial parameter list limit:  $\Lambdabar_{01}^-$} \label{line:laminit_se_gen}
\STATE{$Q^0_0 = f^0_0(W_0,\Lambdabar_{01}^-), \quad P^0_0 \sim \Norm(0,\tau^0_0),
    \quad \tau^0_0 = \Exp(Q^0_0)^2$} \label{line:q0init_se_gen}
\FOR{$\ell=1,\ldots,\Lm1$}
    \STATE{$Q^0_\ell=f^0_\ell(P^0_{\lm1},W_\ell,\Lambdabar_{01}^-)$, \quad
            $P^0_\ell = \Norm(0,\tau^0_\ell)$, \quad
            $\tau^0_\ell = \Exp(Q^0_\ell)^2$} \label{line:p0init_se_gen}
\ENDFOR
\STATE{}

\FOR{$k=0,1,\dots$}
    \STATE{// Forward Pass }
    \STATE{$\lambdabar^+_{k0} = T_{k0}^+(\mubar^+_{k0},\Lambdabar_{0k}^-), \quad
        \mubar^+_{k0} = \Exp(\varphi_{k0}^+(Q_{k0}^-,W_0,\Lambdabar_{0k}^-))$}    \label{line:mup0_se_gen}
    \STATE{$\Lambdabar_{k0}^+ = (\Lambdabar_{k1}^-,\lambdabar^+_{k0})$} \label{line:lamp0_se_gen}
    \STATE{$Q_{k0}^+ = f^+_{k0}(Q_{k0}^-,W_0,\Lambdabar^+_{k0})$}  \label{line:q0_se_gen}
    \STATE{$(P^0_0,P_{k0}^+) = \Norm(\zero,\Kbf_{k0}^+),
        \quad \Kbf_{k0}^+ = \Cov(Q^0_0,Q_{k0}^+)$} \label{line:p0_se_gen}
    \FOR{$\ell=1,\ldots,L-1$}
        \STATE{$\lambdabar^+_{k\ell} = T_{k\ell}^+(\mubar^+_{k\ell},\Lambdabar_{k,\lm1}^+), \quad
            \mubar^+_{k\ell} = \Exp(\varphi_{k\ell}^+(P^0_{\lm1},P^+_{k,\lm1},Q_{k\ell}^-,W_\ell,\Lambdabar_{k,\lm1}^+))$}    \label{line:mup_se_gen}
        \STATE{$\Lambdabar_{k\ell}^+ = (\Lambdabar_{k,\lm1}^+,\lambdabar^+_{k\ell})$}
            \label{line:lamp_se_gen}
        \STATE{$Q_{k\ell}^+ = f^+_{k\ell}(P^0_{\lm1},P^+_{k,\lm1},Q_{k\ell}^-,W_\ell,\Lambdabar^+_{k\ell})$}
            \label{line:qp_se_gen}
        \STATE{$(P^0_\ell,P_{k\ell}^+) = \Norm(\zero,\Kbf_{k\ell}^+), \quad
            \Kbf_{k\ell}^+ = \Cov(Q^0_\ell,Q_{k\ell}^+) $}   \label{line:pp_se_gen}
    \ENDFOR
    \STATE{}

    \STATE{// Reverse Pass }
    \STATE{$\lambdabar^-_{\kp1,L} = T_{kL}^-(\mubar^-_{kL},\Lambdabar_{k,\Lm1}^+), \quad
        \mubar^-_{kL} = \Exp(\varphi_{kL}^-(P^0_{\Lm1},P_{k,\Lm1}^+,W_L,\Lambdabar_{k,\Lm1}^+))$}    \label{line:muL_se_gen}
    \STATE{$\Lambdabar_{\kp1,L}^- = (\Lambdabar_{k,\Lm1}^+,\lambdabar^+_{\kp1,L})$} \label{line:lamL_se_gen}
    \STATE{$P_{\kp1,\Lm1}^- = f^-_{kL}(P^0_{\Lm1},P_{k,\Lm1}^+,W_L,\Lambdabar^-_{\kp1,L})$}  \label{line:pL_se_gen}
    \STATE{$Q_{\kp1,\Lm1}^- = \Norm(0,\tau_{\kp1,\Lm1}^-), \quad
        \tau_{\kp1,\Lm1}^- = \Exp(P^-_{\kp1,\Lm1})^2$} \label{line:qL_se_gen}
    \FOR{$\ell=\Lm1,\ldots,1$}
        \STATE{$\lambdabar^-_{\kp1,\ell} = T_{k\ell}^-(\mubar^-_{k\ell},\Lambdabar_{\kp1,\lp1}^-), \quad
            \mubar^-_{k\ell} =
                \Exp(\varphi_{k\ell}^-(P^0_{\lm1},P_{k,\lm1}^+,Q_{\kp1,\ell}^-,W_\ell,\Lambdabar_{\kp1,\lp1}^+))$}    \label{line:mun_se_gen}
        \STATE{$\Lambdabar_{\kp1,\ell}^- = (\Lambdabar_{\kp1,\lp1}^-,\lambdabar^-_{\kp1,\ell})$} \label{line:lamn_se_gen}
        \STATE{$P_{\kp1,\lm1}^- =
        f^-_{k\ell}(P^0_{\lm1},P^+_{k,\lm1},Q_{\kp1,\ell}^-,W_\ell,\Lambdabar^-_{k\ell})$}
            \label{line:pn_se_gen}
        \STATE{$Q_{\kp1,\lm1}^- = \Norm(0,\tau_{\kp1,\lm1}^-), \quad
        \tau_{\kp1,\lm1}^- = \Exp(P_{\kp1,\lm1}^-)^2$}   \label{line:qn_se_gen}
    \ENDFOR

\ENDFOR
\end{algorithmic}
\end{algorithm}

Under these assumptions, we iteratively define a sequences of constants and random variables through
the recursions in Algorithm~\ref{algo:gen_se}, which we call the Gen-ML state evolution.
The SE recursions in Algorithm~\ref{algo:gen_se} closely
mirrors those in the Gen-ML algorithm (Algorithm~\ref{algo:gen}).  The random vectors
$\qbf^\pm_{k\ell}$ and $\pbf^\pm_{k\ell}$ are replaced by scalar random variables
$Q^\pm_{k\ell}$ and $P^\pm_{k\ell}$; the vector and parameter update functions
$\fbf^+_{k\ell}(\cdot)$ and $\varphibf^+_{k\ell}(\cdot)$ are replaced by their
component functions $f^+_{k\ell}(\cdot)$ and $\varphi^+_{k\ell}(\cdot)$;
and the parameters $\lambda_{k\ell}^\pm$ are replaced
by their limits $\lambdabar_{k\ell}^\pm$.

The various random variables, expectations and covariances in Algorithm~\ref{algo:gen_se}
are computed as follows:  In the initial pass,
in line~\ref{line:p0init_se_gen}, we treat $P^0_{\lm1} \sim \Norm(0,\tau^0_{\lm1})$
and $W_\ell$ as independent for defining the random variable $Q^0_\ell$.
In the forward pass,
in lines~ \ref{line:mup0_se_gen} and \ref{line:q0_se_gen},  we treat
$Q_{k0}^-$ and $W_0$ as independent.  Then,
in lines~ \ref{line:mup_se_gen} and \ref{line:qp_se_gen},  we treat
\[
    (P^0_{\lm1},P_{k,\lm1}^+) \sim \Norm(\zero,\Kbf_{k,\lm1}^+), \quad Q^-_{k\ell}, \quad W_\ell
\]
as independent.
Similar independence assumptions are made in the reverse pass.  We next make the following assumptions.

\begin{assumption} \label{as:gen2} In addition to Assumption~\ref{as:gen} assume:
\begin{enumerate}[(a)]
\item The functions $T^\pm_{k\ell}(\mu_{k\ell}^\pm,\cdot)$ are continuous at
$\mu_{k\ell}^\pm = \mubar_{k\ell}^\pm$ where $\mubar_{k\ell}^\pm$ is the output of
Algorithm~\ref{algo:gen_se}.

\item The vector update component functions in the forward direction and their derivatives,
\[
    f^+_{k\ell}(p^+_{k,\lm1},q_{k\ell}^-,w_\ell,\Lambda^+_{k\ell}),\quad
    \partial f^+_{k\ell}(p^+_{k,\lm1},q_{k\ell}^-,w_\ell,\Lambda^+_{k\ell})/\partial q_{k\ell}^-,
\]
are uniformly Lipschitz continuous in $(p^+_{k,\lm1},q_{k\ell}^-,w_\ell)$ at
$\Lambda^+_{k\ell} = \Lambdabar^+_{k\ell}$.
Similarly,  in the reverse direction,
\[
    f^-_{k\ell}(p^+_{k,\lm1},q_{\kp1,\ell}^-,w_\ell,\Lambda^-_{\kp1,\ell}),\quad
    \partial f^+_{k\ell}(p^+_{k,\lm1},q_{\kp1,\ell}^-,w_\ell,\Lambda^-_{\kp1,\ell})
    /\partial p_{k,\lm1}^+,
\]
are uniformly Lipschitz continuous in $(p^+_{k,\lm1},q_{k\ell}^-,w_\ell)$ at
$\Lambda^-_{\kp1,\ell} = \Lambdabar^-_{\kp1,\ell}$.  Also, the initial vector
update component functions
$f^0_\ell(p^0_{k,\lm1},w_\ell,\Lambda^-_{01})$ are uniformly Lipschitz
continuous in $(p^0_{k,\lm1},w_\ell)$ at $\Lambda^-_{\kp1,\ell} = \Lambdabar^-_{\kp1,\ell}$.

\item The vector update functions are \emph{asymptotically divergence free} meaning
\beq \label{eq:fdivfree}
    \lim_{N \arr \infty} \bkt{\frac{\partial \fbf^+_{k\ell}(\pbf^+_{k,\lm1},\qbf_{k\ell}^-,\wbf_\ell,\Lambdabar^+_{k\ell})}{
        \partial \qbf_{k\ell}^-} } = 0,
    \quad
    \lim_{N \arr \infty} \bkt{\frac{\partial \fbf^-_{k\ell}(\pbf^+_{k,\lm1},\qbf_{k\ell}^-,\wbf_\ell,\Lambdabar^-_{k\ell})}{
        \partial \pbf_{k,\lm1}^+} } = 0 \\
\eeq

\item The parameter update component functions in the forward direction
$\varphi^+_{k\ell}(p^+_{k,\lm1},q_{k\ell}^-,w_\ell,\Lambda^+_{k,\lm1})$
are uniformly Lipschitz continuous in $(p^+_{k,\lm1},q_{k\ell}^-,w_\ell)$ at
$\Lambda^+_{k,\lm1} = \Lambdabar^+_{k,\lm1}$.
Analogous conditions apply to the reverse functions $\varphi^-_{k\ell}(\cdot)$
\end{enumerate}
\end{assumption}

\medskip
Under the above assumptions, the following theorem proves the SE equations for the Gen-ML recursion.

\begin{theorem} \label{thm:genConv}  Consider the outputs of the Gen-ML recursion (Algorithm~\ref{algo:gen})
and the corresponding random variables and parameter limits
defined by the SE updates in Algorithm~\ref{algo:gen_se} under Assumptions~\ref{algo:gen} and \ref{as:gen2}.
Then,
\begin{enumerate}[(a)]
\item For any fixed $k$ and $\ell=1,\ldots,\Lm1$,
the parameter list $\Lambda_{k\ell}^+$ converges as
\beq \label{eq:Lamplim}
    \lim_{N \arr \infty} \Lambda_{k\ell}^+ = \Lambdabar_{k\ell}^+
\eeq
almost surely.
Also, the components of
$\wbf_\ell$, $\pbf^0_{\lm1}$, $\qbf^0_{\ell}$, $\pbf_{0,\lm1}^+,\ldots,\pbf_{k,\lm1}^+$ and $\qbf_{0\ell}^\pm,\ldots,\qbf_{k\ell}^\pm$
almost surely empirically converge jointly with limits,
\beq \label{eq:PQplim}
    \lim_{N \arr \infty} \left\{
        (p^0_{\lm1,n},p^+_{i,\lm1,n},q^0_{\ell,n},q^-_{j\ell,n},q^+_{j\ell,n}) \right\} =
        (P^0_{\lm1},P^+_{i,\lm1},Q^0_{\ell},Q^-_{j\ell}, Q^+_{j\ell}),
\eeq
for all $i,j=0,\ldots,k$, where the variables
$P^0_{\lm1}$, $P_{i,\lm1}^+$ and $Q_{j\ell}^-$
are zero-mean jointly Gaussian random variables independent of $W_\ell$ with
\beq \label{eq:PQpcorr}
    \Cov(P^0_{\lm1},P_{i,\lm1}^+) = \Kbf_{i,\lm1}^+, \quad \Exp(Q_{j\ell}^-)^2 = \tau_{j\ell}^-, \quad \Exp(P_{i,\lm1}^+Q_{j\ell}^-)  = 0,
    \quad \Exp(P^0_{\lm1}Q_{j\ell}^-)  = 0,
\eeq
and $Q^0_\ell$ and $Q^+_{j\ell}$ are the random variable in line~\ref{line:qp_se_gen}:
\beq \label{eq:Qpf}
    Q^0_\ell = f^0_\ell(P^0_{\lm1},W_{\ell}), \quad
    Q^+_{j\ell} =
    f^+_{j\ell}(P^0_{\lm1},P^+_{i,\lm1},Q^-_{j\ell},W_\ell,\Lambdabar_{k\ell}^+).
\eeq
The identical result holds for $\ell=0$ with all the variables $\pbf_{i,\lm1}^+$ and $P_{i,\lm1}^+$ removed.

\item For any fixed $k > 0$ and $\ell=1,\ldots,\Lm1$,
the parameter lists $\Lambda_{k\ell}^-$ converge as
\beq \label{eq:Lamnlim}
    \lim_{N \arr \infty} \Lambda_{k\ell}^- = \Lambdabar_{k\ell}^-
\eeq
almost surely.
Also, the components of
$\wbf_\ell$, $\pbf^0_{\lm1}$, $\pbf_{0,\lm1}^+,\ldots,\pbf_{\km1,\lm1}^+$, $\pbf_{0,\lm1}^+,\ldots,\pbf_{\km1,\lm1}^+$, and $\qbf_{0\ell}^-,\ldots,\qbf_{k\ell}^-$
almost surely empirically converge jointly with limits,
\beq \label{eq:PQnlim}
    \lim_{N \arr \infty} \left\{
        (p^0_{\lm1,n},p^+_{i,\lm1,n},q^-_{j\ell,n},q^+_{j\ell,n}) \right\} =
        (P^0_{\lm1},P^+_{i,\lm1},Q^-_{j\ell}, Q^+_{j\ell}),
\eeq
for all $i=0,\ldots,\km1$ and $j=0,\ldots,k$, where the variables
$P^0_{\lm1}$, $P_{i,\lm1}^+$ and $Q_{j\ell}^-$
are zero-mean jointly Gaussian random variables independent of $W_\ell$ with
\beq \label{eq:PQncorr}
    \Cov(P^0_{\lm1},P_{i,\lm1}^+) = \Kbf_{i,\lm1}^+,\quad
    \quad \Exp(Q_{j\ell}^-)^2 = \tau_{j\ell}^-, \quad \Exp(P_{i,\lm1}^+Q_{j\ell}^-)  = 0,
    \quad \Exp(P^0_{\lm1}Q_{j\ell}^-)  = 0,
\eeq
and $P^-_{j\ell}$ is the random variable in line~\ref{line:pn_se_gen}:
\beq \label{eq:Pnf}
    P^-_{j\ell} = f^-_{j\ell}(P^0_{\lm1},P^+_{i,\lm1},
                Q^-_{j\ell},W_\ell,\Lambdabar_{k\ell}^-).
\eeq
The identical result holds for $\ell=L$ with all the variables $\qbf_{j\ell}^-$ and $Q_{j\ell}^-$ removed.
Also, for $k=0$, we remove the variables with $\pbf_{\km1,\ell}^+$ and $P_{\km1,\ell}^+$.
\end{enumerate}
\end{theorem}
\begin{proof}  We will prove this in Appendix~\ref{sec:genConvPf}.
\end{proof}

\section{Proof of Theorem~\ref{thm:genConv}} \label{sec:genConvPf}

\subsection{Overview of the Induction Sequence}

The proof is similar to that of \cite[Theorem 4]{rangan2016vamp},
which provides a SE analysis for VAMP on a single-layer network.
The critical challenge here is to extend that proof
to multi-layer recursions.
Many of the ideas in the two proofs are similar, so we highlight only the
key differences between the two.

Similar to the SE analysis of VAMP in \citep{rangan2016vamp},
we use an induction argument.  However, for the multi-layer proof,
we must index over both the iteration index $k$ and layer index $\ell$. To this end,
let $\mathcal{H}_{k\ell}^+$ and $\mathcal{H}_{k\ell}^-$ be the hypotheses:
\begin{itemize}
\item $\mathcal{H}_{k\ell}^+$:  The hypothesis that Theorem~\ref{thm:genConv}(a)
is true for some $k$ and $\ell$.
\item $\mathcal{H}_{k\ell}^-$:  The hypothesis that Theorem~\ref{thm:genConv}(b)
is true for some $k$ and $\ell$.
\end{itemize}
We prove these hypotheses by induction via a sequence of implications,
\beq \label{eq:induc}
    \cdots \Rightarrow \mathcal{H}_{k1}^- \Rightarrow \mathcal{H}_{k0}^+ \Rightarrow \cdots \Rightarrow  \mathcal{H}_{k,\Lm1}^+
    \Rightarrow \mathcal{H}_{\kp1,L}^- \Rightarrow \cdots \Rightarrow \mathcal{H}_{k1}^- \Rightarrow \cdots,
\eeq
beginning with the hypothesis $\mathcal{H}^-_{0\ell}$ for all $\ell=1,\ldots,\Lm1$.

\subsection{Proof of the Induction Update}

Now fix a stage index $\ell=1,\ldots,\Lm1$ and an iteration index $k=0,1,\ldots$.
Assume, as an induction hypothesis,
that all the hypotheses \emph{prior} to $\mathcal{H}^+_{k,\lp1}$ in the sequence \eqref{eq:induc}
but not including $\mathcal{H}^+_{k,\lp1}$ are true.  We show that,
under this assumption, $\mathcal{H}^+_{k,\lp1}$ is true.
The other implications in the hypothesis sequence \eqref{eq:induc} can be proven similarly.

We introduce with some notation.  Let
\[
    \Pbf_{k\ell}^+ := \left[ \pbf_{0\ell}^+ \cdots \pbf_{k\ell}^+ \right] \in \R^{N_\ell \x (\kp1)},
\]
which is the matrix whose columns are the first $\kp1$ values of the vector $\pbf^+_{i\ell}$.
We define the matrices $\Pbf_{k\ell}^-$, $\Qbf_{k\ell}^+$ and $\Qbf_{k\ell}^-$ similarly.
Let $\Gset_{k\ell}^\pm$ denote the collection of random variables associated with the hypotheses,
$\mathcal{H}^{\pm}_{k\ell}$.  That is, for $\ell=1,\ldots,\Lm1$,
\beq \label{eq:Gsetdef}
    \Gset_{k\ell}^+ := \left\{ \wbf_{\ell},\pbf^0_{\lm1},\Pbf^+_{k,\lm1},\qbf^0_\ell,\Qbf^-_{k\ell},\Qbf_{k\ell}^+ \right\}, \quad
    \Gset_{k\ell}^- := \left\{ \wbf_{\ell},\pbf^0_{\lm1},\Pbf^+_{\km1,\lm1},\qbf^0_\ell,
        \Qbf^-_{k\ell},\Pbf^-_{k,\lm1} \right\}.
\eeq
For $\ell=0$ and $\ell=L$ we set,
\beq \label{eq:Gsetdefend}
    \Gset_{k0}^+ := \left\{ \wbf_{0},\Qbf^-_{k0},\Qbf_{k0}^+ \right\}, \quad
    \Gset_{kL}^- := \left\{ \wbf_L,\pbf^0_{\Lm1},\Pbf^+_{\km1,\Lm1},\Pbf^-_{k,\Lm1} \right\}.
\eeq
With some abuse of notation, we let $\Gset_{km\ell}$ also denote the sigma-algebra
generated by these vectors.
Also, let $\Gsetbar_{k\ell}^+$ be the union of all the sets $\Gset_{i\ell'}^\pm$
as they appear in the sequence \eqref{eq:induc} up to and including the final
set $\Gsetbar_{k\ell}^+$.
Thus, the set $\Gsetbar_{k\ell}^+$ contains all the vectors produced by Algorithm~\ref{algo:gen}
immediately \emph{before} line~\ref{line:pp_gen} in stage $\ell$ of iteration $k$.

Now, the actions of the matrix $\Vbf_\ell$ in Algorithm~\ref{algo:gen}
are through the matrix-vector multiplications in lines~\ref{line:pp_gen} and \ref{line:pn_gen}.
Hence, if we define the matrices,
\beq \label{eq:ABdef}
    \Abf_{k\ell} := \left[ \pbf^0_\ell, \Pbf_{\km1,\ell}^+ ~ \Pbf_{k\ell}^- \right], \quad
    \Bbf_{k\ell} := \left[ \qbf^0_\ell, \Qbf_{\km1,\ell}^+ ~ \Qbf_{k\ell}^- \right],
\eeq
all the vectors in the set $\Gsetbar_{k\ell}^+$ will be unchanged for all
matrices $\Vbf_\ell$ satisfying the linear constraints
\beq \label{eq:ABVconk}
    \Abf_{k\ell} = \Vbf_\ell\Bbf_{k\ell}.
\eeq
Hence, the conditional distribution of $\Vbf_\ell$ given $\Gsetbar_{k\ell}^+$ is precisely
the uniform distribution on the set of orthogonal matrices satisfying
\eqref{eq:ABVconk}.  The matrices $\Abf_{k\ell}$ and $\Bbf_{k\ell}$ are of dimensions
$N_\ell \x s$ where $s=2k+2$.
From \citep{rangan2016vamp}, this conditional distribution is given by
\beq \label{eq:Vconk}
    \left. \Vbf_\ell \right|_{\Gsetbar_{k\ell}^+} \eqd
    \Abf_{k\ell}(\Abf\tran_{k\ell}\Abf_{k\ell})^{-1}\Bbf_{k\ell}\tran + \Ubf_{\Abf_{k\ell}^\perp}\tilde{\Vbf}_\ell\Ubf_{\Bbf_{k\ell}^\perp}\tran,
\eeq
where $\Ubf_{\Abf_{k\ell}^\perp}$ and $\Ubf_{\Bbf_{k\ell}^\perp}$ are $N \x (N-s)$ matrices
whose columns are an orthonormal basis for $\Range(\Abf_{k\ell})^\perp$ and $\Range(\Bbf_{k\ell})^\perp$.
The matrix $\tilde{\Vbf}_\ell$ is  Haar distributed on the set of $(N-s)\x(N-s)$
orthogonal matrices and independent of $\Gsetbar_{k\ell}^+$.

Next, similar to the proof of \cite[Theorem 4]{rangan2016vamp},
we use \eqref{eq:Vconk} and write $\pbf_{k\ell}^+$ from line~\ref{line:pp_gen}
as a sum of two terms
\beq \label{eq:ppart}
    \pbf_{k\ell}^+ = \Vbf_\ell\qbf_{k\ell}^+ = \pbf_{k\ell}^{\rm det} + \pbf_{k\ell}^{\rm ran},
\eeq
where $\pbf_{k\ell}^{\rm +det}$ is what we will call the \emph{deterministic} part:
\beq \label{eq:pdet}
    \pbf_{k\ell}^{\rm det} = \Abf_{k\ell}(\Bbf\tran_{k\ell}\Bbf_{k\ell})^{-1}\Bbf_{k\ell}\tran\qbf_{k\ell}^+
\eeq
and $\pbf_{k\ell}^{\rm ran}$ is what we will call the \emph{random} part:
\beq \label{eq:pran}
    \pbf_k^{\rm ran} = \Ubf_{\Bbf_k^\perp}\tilde{\Vbf}_\ell\tran \Ubf_{\Abf_k^\perp}\tran \qbf_{k\ell}^+.
\eeq
The next two lemmas characterize the limiting distributions
of the deterministic and random components.

\begin{lemma} \label{lem:pconvdet}
Under the induction hypothesis, the components of the ``deterministic" component
$\pbf_{k\ell}^{\rm det}$ along with the components
of the vectors in $\Gsetbar_{k\ell}^+$  converge empirically.
In addition, there exists constants $\beta_{0\ell}^+,\ldots,\beta^+_{\km1,\ell}$ such that
\beq \label{eq:pconvdet}
    \lim_{N \arr \infty} \{ p_{k\ell,n}^{\rm det} \} \PLeq P_{k\ell}^{\rm det}
    = \beta^0_\ell P^0_\ell +  \sum_{i=0}^{\km1}    \beta_{i\ell} P_{i\ell}^+,
\eeq
where $P_{k\ell}^{\rm det}$ is the limiting random variable for the components of $\pbf_{k\ell}^{\rm det}$.
\end{lemma}
\begin{proof}
The proof is similar to the proof of \cite[Lemma 6]{rangan2016vamp}, but we will go over the details
as there are some important differences in the multi-layer case.
Define
\beq \label{eq:PQaug}
    \tilde{\Pbf}_{\km1,\ell}^+ = \left[ \pbf^0_\ell, ~ \Pbf_{\km1,\ell}^+ \right], \quad
    \tilde{\Qbf}_{\km1,\ell}^+ = \left[ \qbf^0_\ell, ~ \Qbf_{\km1,\ell}^+ \right],
\eeq
which are the matrices $\Pbf_{\km1,\ell}^+$ and $\Qbf_{\km1,\ell}^+$
with the additions of the columns $\pbf^0_\ell$ and $\qbf^0_\ell$.
We can then write $\Abf_{k\ell}$ and $\Bbf_{k\ell}$ in \eqref{eq:ABdef} as
\beq \label{eq:ABdef2}
    \Abf_{k\ell} := \left[ \tilde{\Pbf}_{\km1,\ell}^+ ~ \Pbf_{k\ell}^- \right], \quad
    \Bbf_{k\ell} := \left[ \tilde{\Qbf}_{\km1,\ell}^+ ~ \Qbf_{k\ell}^- \right],
\eeq
We first evaluate the asymptotic values of various terms in \eqref{eq:pdet}.
Using the definition of $\Abf_{k\ell}$ in \eqref{eq:ABdef},
\[
    \Bbf\tran_{k\ell}\Bbf_{k\ell} = \left[ \begin{array}{cc}
        (\tilde{\Qbf}_{\km1,\ell}^+)\tran\tilde{\Qbf}_{\km1,\ell}^+ & (\tilde{\Qbf}_{\km1,\ell}^+)\tran\Qbf_{k\ell}^- \\
        (\Qbf_{k\ell}^-)\tran\tilde{\Qbf}_{\km1,\ell}^+ & (\Qbf_{k\ell}^-)\tran\Qbf_{k\ell}^-
        \end{array} \right]
\]
We can then easily evaluate the asymptotic value of these
terms as follows:  The asymptotic value of the
$(i+1,j+1)$-th component of the matrix $(\tilde{\Qbf}_{\km1,\ell}^+)\tran\tilde{\Qbf}_{\km1,\ell}^+$ is given by
\begin{align*}
    \MoveEqLeft \lim_{N \arr \infty} \frac{1}{N_\ell} \left[ (\tilde{\Qbf}_{\km1,\ell}^+)\tran\tilde{\Qbf}_{\km1,\ell}^+ \right]_{i+1,j+1}
        \stackrel{(a)}{=} \lim_{N \arr \infty}
        \frac{1}{N_\ell} (\qbf_{i\ell}^+)\tran\qbf_{j\ell}^+ \\
        &= \lim_{N \arr \infty} \frac{1}{N_\ell} \sum_{n=1}^{N_\ell} q_{i\ell,n}^+q_{j\ell,n}^+
        \stackrel{(b)}{=} E\left[ Q_{i\ell}^+Q_{j\ell}^+ \right]
\end{align*}
where (a) follows since the $i+1$-st column of $\tilde{\Qbf}_{\km1,\ell}^+$
is precisely the vector
$\qbf_{i\ell}^+$; and (b) follows due to convergence assumption in \eqref{eq:PQplim}.
Also, since the first column of $\tilde{\Qbf}_{\km1,\ell}^+$ is $\qbf^0_\ell$,
we obtain that the
\[
    \lim_{N_\ell \arr \infty}  \frac{1}{N_\ell}
        (\tilde{\Qbf}_{k\ell}^-)\tran\tilde{\Qbf}_{k\ell}^- = \Rbf^+_{k\ell},
\]
where $\Rbf^+_{k\ell}$ is the correlation matrix of the vector
$(Q^0_\ell,Q_{0\ell}^+,\ldots,Q_{k\ell}^+)$.
Similarly,
\[
    \lim_{N_\ell \arr \infty}  \frac{1}{N_\ell}  (\Qbf_{k\ell}^-)\tran\Qbf_{k\ell}^- = \Rbf^-_{k\ell},
\]
where $\Rbf^-_{k\ell}$ is the correlation matrix of the vector
$(Q_{0\ell}^-,\ldots,Q_{k\ell}^-)$.
For the matrix $(\Qbf_{\km1,\ell}^+)\tran\Qbf_{k\ell}^-$,
first observe that the limit of the divergence free condition \eqref{eq:fdivfree} implies
\beq \label{eq:fpdivfree}
    \Exp\left[ \frac{\partial f_{i\ell}^+(P_{i,\lm1}^+,Q_{i\ell}^-,W_\ell,\Lambdabar_{i\ell})}{\partial q_{i\ell}^-} \right]
    = \lim_{N_{\lm1} \arr \infty}  \bkt{\frac{\partial \fbf^+_{i\ell}(\pbf^+_{i,\lm1},\qbf_{i\ell}^-,\wbf_\ell,\Lambdabar^+_{i\ell})}{
        \partial \qbf_{i\ell}^-} }  = 0,
\eeq
for any $i$.  Also, by the induction hypothesis $\mathcal{H}_{k\ell}^+$,
\beq \label{eq:pqxcorrpf}
    \Exp(P_{i,\lm1}^+Q_{j\ell}^-)  = 0, \quad
    \Exp(P_{\lm1}^0 Q_{j\ell}^-) = 0,
\eeq
for all $i,j \leq k$.
Therefore, the expectations for the cross-terms are given by
\begin{align}
    \MoveEqLeft \Exp(Q_{i\ell}^+Q_{j\ell}^-)
    \stackrel{(a)}{=}
    \Exp(f_{i\ell}^+(P^0_{\lm1},P_{i,\lm1}^+,Q_{j\ell}^-,W_\ell,\Lambdabar_{i\ell})Q_{j\ell}^-) \nonumber \\
    &\stackrel{(b)}{=} \Exp\left[ \frac{\partial f_{i\ell}^+(P^0_{\lm1},P_{i,\lm1}^+,Q_{i\ell}^-,W_\ell,\Lambdabar^+_{i\ell})}{\partial p_{i,\lm1}^+} \right]
        \Exp(P_{i,\lm1}^+Q_{j\ell}^-)  \nonumber \\
    &+ \Exp\left[ \frac{\partial f_{i\ell}^+(P^0_{\lm1},P_{i,\lm1}^+,Q_{i\ell}^-,W_\ell,\Lambdabar^+_{i\ell})}
        {\partial p_{\lm1}^0} \right]
        \Exp(P_{\lm1}^0Q_{j\ell}^-)  \nonumber \\
    & + \Exp\left[ \frac{\partial f_{i\ell}^+(P^0_{\lm1},P_{i,\lm1}^+,Q_{i\ell}^-,W_\ell,\Lambdabar^+_{i\ell})}
        {\partial q_{i\ell}^-} \right]
        \Exp(Q_{i\ell}^-Q_{j\ell}^-) 
     \stackrel{(c)}{=} 0, \label{eq:Qijstein}
\end{align}
where (a) follows from \eqref{eq:Qpf};
(b) follows from Stein's Lemma; and in (c), we use \eqref{eq:fpdivfree} and \eqref{eq:pqxcorrpf}.
The above calculations show that
\beq \label{eq:BBlim}
    \lim_{N_\ell \arr \infty} \frac{1}{N_\ell} \Bbf\tran_{k\ell}\Bbf_{k\ell} = \left[ \begin{array}{cc}
        \Rbf_{\km1,\ell}^+ & \zero \\
        \zero & \Rbf_{k\ell}^-
        \end{array} \right].
\eeq
A similar calculation shows that
\beq \label{eq:Bqlim}
    \lim_{N_\ell \arr \infty} \frac{1}{N_\ell} \Bbf_{k\ell}\tran\qbf_{k\ell}^+= \left[
    \begin{array}{c} \bbf^+_{k\ell} \\ \zero \end{array} \right],
\eeq
where $\bbf^+_{k\ell}$ is the vector of correlations
\beq
    \bbf^+_{k\ell} = \left[\Exp(Q_{0\ell}^+Q_{k\ell}^+), ~ \Exp(Q_{1\ell}^+Q_{k\ell}^+),
        ~\cdots, \Exp(Q_{\km1,\ell}^+Q_{k\ell}^+) \right]\tran.
\eeq
Combining \eqref{eq:BBlim} and \eqref{eq:Bqlim} shows that
\beq \label{eq:Bqmult}
    \lim_{N_\ell \arr \infty} (\Bbf\tran_{k\ell}\Bbf_{k\ell})^{-1}\Bbf_{k\ell}\tran\qbf_{k\ell}^+ =
    \left[ \begin{array}{c}  \betabf_{k\ell}^+ \\ \mathbf{0} \end{array} \right], \quad \betabf_{k\ell} = \left[ \Rbf^+_{\km1,\ell} \right]^{-1}\bbf^+_{k\ell}.
\eeq
Therefore,
\begin{align}
    \pbf_{k\ell}^{\rm det} &= \Abf_{k\ell}(\Bbf\tran_{k\ell}\Bbf_{k\ell})^{-1}\Bbf_{k\ell}\tran\qbf_{k\ell}^+
    = \left[ \tilde{\Pbf}_{\km1,\ell}^+ ~ \Pbf_{k,\ell}^- \right]
   \left[ \begin{array}{c} \mathbf{\beta}_{k\ell}^+  \\ \zero \end{array} \right]
    + O\left(\frac{1}{N_\ell}\right) \nonumber \\
    &= \beta^0_\ell \pbf^0_\ell +
    \sum_{i=0}^{\km1} \beta_{i\ell}^+\pbf_{i\ell}^+ + O\left(\frac{1}{N_\ell}\right),
\end{align}
where $\beta^0_\ell$ and $\beta_{i\ell}^+$ are the components of $\betabf_{k\ell}^+$ and
the term $O(1/N)$ means a vector sequence, $\xibf(N) \in \R^N$ such that
\[
    \lim_{N \arr\infty} \frac{1}{N} \|\xibf(N)\|^2 = 0.
\]
A continuity argument then shows \eqref{eq:pconvdet}.
\end{proof}

\begin{lemma} \label{lem:pconvran}
Under the induction hypothesis, the components of
the ``random" part $\pbf_{k\ell}^{\rm ran}$ along with the components
of the vectors in $\Gsetbar_{k\ell}^+$ almost surely converge empirically.
The components of $\pbf_{k\ell}^{\rm ran}$ converge as
\beq \label{eq:pconvran}
     \lim_{N \arr \infty} \{ p_{k\ell,n}^{\rm ran} \} \PLeq U_{k\ell},
\eeq
where $U_{k\ell}$ is a zero mean Gaussian random variable
independent of the limiting random variables corresponding to the variables
in $\Gsetbar_{k\ell}^+$.
\end{lemma}
\begin{proof}
This proof is very similar to that of \cite[Lemma 7,8]{rangan2016vamp}.
\end{proof}

We now combine the above lemmas to prove the following, which proves
all the conditions for the hypothesis $\mathcal{H}^+_{k,\lp1}$.  Hence, this will show the
induction implication and completes the proof of Theorem~\ref{thm:genConv}.

\begin{lemma} \label{lem:pqconvinduc}
Under the induction hypothesis, the parameter list $\Lambda_{k,\lp1}^+$ almost surely converges as
\beq \label{eq:Lampliminduc}
    \lim_{N_{\lp1} \arr \infty} \Lambda_{k,\lp1}^+ = \Lambdabar_{k,\lp1}^+,
\eeq
where $\Lambdabar_{k,\lp1}$ is the parameter list generated from the SE recursion, Algorithm~\ref{algo:gen_se}.
Also, the components of
$\wbf_\lp1$, $\pbf^0_{\ell}$, $\qbf^0_{\lp1}$, $\pbf_{0,\ell}^+,\ldots,\pbf_{k,\ell}^+$ and $\qbf_{0,\lp1}^\pm,\ldots,\qbf_{k,\lp1}^\pm$
almost surely empirically converge jointly with limits,
\beq \label{eq:PQpliminduc}
    \lim_{N \arr \infty} \left\{
        (p^0_{\ell,n},p^+_{i\ell,n},q^0_{\lp1,n},q^-_{j,\lp1,n},q^+_{j,\lp1,n}) \right\} =
        (P^0_{\ell},P^+_{i\ell},Q^0_{\lp1},Q^-_{j,\lp1}, Q^+_{j,\lp1}),
\eeq
for all $i,j=0,\ldots,\kp1$, where the variables
\beq \label{eq:pqvecinduc}
    (P^0_\ell,P_{0\ell}^+,\ldots,P_{k,\ell}^+,Q_{0,\lp1}^-,\ldots,Q_{k,\lp1}^-),
\eeq
are zero-mean jointly Gaussian random variables independent of $W_\ell$ with
\beq \label{eq:PQpcorrinduc}
    \Cov(P^0_{\ell},P_{i,\ell}^+) = \Kbf_{i\ell}^+, \quad \Exp(Q_{j,\lp1}^-)^2 = \tau_{j,\lp1}^-, \quad \Exp(P_{i,\ell}^+Q_{j,\lp1}^-)  = 0,
    \quad \Exp(P^0_{\ell}Q_{j,\lp1}^-)  = 0,
\eeq
and $Q^0_{\lp1}$ and $Q^+_{j,\lp1}$ are the random variables in line~\ref{line:qp_se_gen}:
\beq \label{eq:Qpfinduc}
    Q^0_{\lp1} = f^0_{\lp1}(P^0_{\ell},W_{\lp1}), \quad
    Q^+_{j,\lp1} =
    f^+_{j,\lp1}(P^0_{\ell},P^+_{i\ell},Q^-_{j,\lp1},W_{\lp1},\Lambdabar_{k,\lp1}^+).
\eeq
\end{lemma}
\begin{proof}
Using the partition \eqref{eq:ppart} and Lemmas~\ref{lem:pconvdet} and \ref{lem:pconvran},
we see that the components of the
vector sequences in $\Gsetbar_{k\ell}^+$ along with $\pbf^+_{k\ell}$
almost surely converge jointly empirically, where the components of $\pbf^+_{k\ell}$
have the limit
\beq \label{eq:pklim}
    \lim_{N_\ell \arr \infty} \left\{ p^+_{k\ell,n} \right\}
    = \lim_{N_\ell \arr \infty} \left\{ p^{\rm det}_{k\ell,n} + p^{\rm ran}_{k\ell,n} \right\}
    \PLeq  \beta^0_\ell P^0_\ell + \sum_{i=0}^{\km1}  \beta_{i\ell}^+ P_{i\ell}^+ + U_{k\ell} =: P_{k\ell}^+.
\eeq
By the induction hypothesis, we can assume $\mathcal{H}_{\km1,\lp1}^-$ is true since this
hypothesis appears before $\mathcal{H}_{k,\lp1}^+$ in the induction sequence \eqref{eq:induc}.
Therefore, we can assume that
\beq \label{eq:pqvecinduc0}
    (P_{0\ell}^+,\ldots,P_{\km1,\ell}^+,Q_{0,\lp1}^-,\ldots,Q_{k,\lp1}^-),
\eeq
is jointly Gaussian.  If we add the variable $P^+_{k\ell}$ to this set, we obtain the set \eqref{eq:pqvecinduc}.
From \eqref{eq:pklim} and the fact that
$U_k$ is Gaussian independent of the variables in \eqref{eq:pqvecinduc0}, the set of variables
in \eqref{eq:pqvecinduc} must be jointly Gaussian.  We next need to prove the correlations
in \eqref{eq:PQpcorrinduc}.  Since we can assume $\mathcal{H}_{\km1,\lp1}^-$ is true, we know
that \eqref{eq:PQpcorrinduc} is true for
all $i=0,\ldots,\km1$ and $j=0,\ldots,k$.  Hence,
we need only to prove the additional identities in \eqref{eq:PQpcorrinduc} for $i=k$,
namely the equations:
\beq \label{eq:PQcorrpf2}
    \Cov(P^0_\ell,P_{k\ell}^+)^2 = \Kbf_{k\ell}^+
    \qquad
    \mbox{and}
    \qquad
    \Exp(P_{k\ell}^+Q_{j,\lp1}^-) = 0.
\eeq
First observe that
\[
    \Exp(P_{k\ell}^+)^2  \stackrel{(a)}{=} \lim_{N_\ell \arr \infty} \frac{1}{N_\ell}
        \|\pbf_{k\ell}^+\|^2
          \stackrel{(b)}{=} \lim_{N_\ell \arr \infty} \frac{1}{N_\ell}
        \|\qbf_{k\ell}^+\|^2 \stackrel{(c)}{=}  \Exp\left( Q_{k\ell}^+ \right)^2
\]
where (a) follows from the fact that the components of $\pbf^+_{k\ell}$ converge empirically
to $P_{k\ell}^+$;
(b) follows from line \ref{line:pp_gen} in Algorithm~\ref{algo:gen} and the fact that $\Vbf_\ell$ is orthogonal;
and
(c) follows from the fact that the components of $\qbf^+_{k\ell}$ converge empirically
to $Q_{k\ell}^+$.  Since $\pbf^0_\ell = \Vbf_\ell \qbf^0$, we similarly obtain that
\[
    \Exp(P^0_\ell P_{k\ell}^+) = \Exp(Q^0_\ell Q_{k\ell}^+), \quad
    \Exp(P^0_\ell)^2 = \Exp(Q^0_\ell)^2,
\]
from which we conclude
\beq \label{eq:PQcorr3}
    \Cov(P^0_\ell, P_{k\ell}^+) = \Cov(Q^0_\ell, Q_{k\ell}^+) =: \Kbf^+_{k\ell},
\eeq
where the last step follows from the definition of $\Kbf^+_{k\ell}$ in line~\ref{line:pp_se_gen}.
For the second term in \eqref{eq:PQcorrpf2}, we observe that
\beq \label{eq:PQcorr4}
    \Exp(P_{k\ell}^+Q_{j,\lp1}^-) \stackrel{(a)}{=}
     \beta^0_\ell\Exp(P_{\ell}^0Q_{j,\lp1}^-) + \sum_{i=0}^{\km1} \beta_{i\ell}^+ \Exp(P_{i\ell}^+Q_{j,\lp1}^-)
        + \Exp(U_{k\ell}Q_{j,\lp1}^-) \stackrel{(a)}{=} 0,
\eeq
where (a) follows from \eqref{eq:pklim} and, in (b), we used the fact that
$\Exp(P_{\ell}^0Q_{j,\lp1}^-) = 0$ and
$\Exp(P_{i\ell}^+Q_{j,\lp1}^-) = 0$ since \eqref{eq:PQpcorrinduc} is true for $i\leq \km1$ and
$\Exp(U_{k\ell}Q_{j,\lp1}^-) = 0$ since $U_{k\ell}$ is independent of all the variables $Q_{j,\lp1}^-$.
Thus, with \eqref{eq:PQcorr3} and \eqref{eq:PQcorr4}, we have proven all the correlations in
\eqref{eq:PQpcorrinduc}.

Next, we prove \eqref{eq:Lampliminduc}.  Since $\Lambda^+_{k\ell} \arr \Lambdabar_{k\ell}^+$,
and $\varphi_{k,\lp1}^+(\cdot)$ is uniformly Lipschitz continuous,
we have that $\mu^+_{k,\lp1}$ from line~\ref{line:mup_gen} in Algorithm~\ref{algo:gen}
converges almost surely as
\begin{align}
    \MoveEqLeft\lim_{N \arr \infty} \mu^+_{k,\lp1} = \lim_{N \arr \infty}
        \bkt{\varphibf_{k,\lp1}^+(\pbf^0_\ell,\pbf^+_{k\ell},\qbf_{k,\lp1}^-,\wbf_{\lp1},\Lambdabar_{k\ell}^+)}
        \nonumber \\
    &=    \Exp\left[ \varphi_{k,\lp1}^+(P^0_\ell,P^+_{k\ell},Q_{k,\lp1}^-,W_{\lp1},\Lambdabar_{k\ell}^+)
        \right]  = \mubar^+_{k,\lp1},
\end{align}
where $\mubar^+_{k,\lp1}$ is the value in line~\ref{line:mup_se_gen} in Algorithm~\ref{algo:gen_se}.
Since $T^+_{k,\lp1}(\cdot)$ is continuous, we have that $\lambda_{k,\lp1}^+$ in
line~\ref{line:lamp_gen} in Algorithm~\ref{algo:gen} converges as
\beq
    \lim_{N \arr \infty} \lambda_{k,\lp1}^+ =
    \lim_{N \arr \infty} T_{k,\lp1}^+(\mu_{k,\lp1}^+)
    = T_{k,\lp1}^+(\mubar_{k,\lp1}^+) = \lambdabar_{k,\lp1}^+,
\eeq
where $\lambdabar_{k,\lp1}^+$ is the value in
line~\ref{line:lamp_se_gen} in Algorithm~\ref{algo:gen_se}. Therefore, we have the limit
\beq
    \lim_{N \arr \infty} \Lambda_{k,\lp1}^+ =
    \lim_{N \arr \infty} (\Lambda_{k,\ell}^+,\lambda_{k,\lp1}^+)
    = (\Lambdabar_{k,\ell}^+,\lambdabar_{k,\lp1}^+) = \Lambdabar_{k,\lp1}^+,
\eeq
which proves \eqref{eq:Lampliminduc}.
Finally, using \eqref{eq:Lampliminduc}, the convergence of the vector
sequences $\pbf^0_\ell$, $\pbf_{k\ell}^+$ and $\qbf_{k,\lp1}^-$ and the uniform Lipschitz continuity of
the update function $f_{k\lp1}^+(\cdot)$ we obtain that
\begin{align}
    \lim_{N \arr \infty} \left\{ q_{k,\lp1,n}^+ \right\}
    &= \left\{ f_{k,\lp1}^+(p^0_{\ell,n},p_{k\ell,n}^-, q_{k,\lp1,n}^-,w_{\lp1,n},\Lambda_{k,\lp1}^+) \right\}
    \nonumber \\
    &= f_{k,\lp1}^+(P^0_\ell,P_{k\ell}^-, Q_{k,\lp1}^-,W_{\lp1},\Lambdabar_{k,\lp1}^+) =: Q^+_{k,\lp1},
    \nonumber
\end{align}
which proves \eqref{eq:Qpfinduc}.  This completes the proof.
\end{proof}

\section{Proof of Theorem~\ref{thm:semlvamp} }  \label{sec:semlpf}

\subsection{Equivalence of ML-VAMP to Gen-ML} \label{sec:equivpf}

The first step of the proof of Theorem~\ref{thm:semlvamp} is to show that
the ML-VAMP Algorithm (Algorithm \ref{algo:ml-vamp}) is a special case of the the Gen-ML Algorithm
(Algorithm~\ref{algo:gen}).   To this end, we have to identify the various
components of Gen-ML algorithm in terms of the quantities in ML-VAMP.

\paragraph*{Transformed MLP}  We first rewrite the MLP \eqref{eq:nntrue} in
a certain transformed form.  Define the disturbance vectors as:
\begin{subequations} \label{eq:wdef}
\begin{align}
    \wbf_0 &:= \zbf^0_0, \quad
    \wbf_{\ell} := \xibf_\ell, \quad \ell =2,4,\ldots,L \\
    \wbf_{\ell} &= (\bar{\sbf}_{\ell},\bar{\bbf}_\ell,\bar{\xibf}_\ell),  \quad \ell =1,3,\ldots,\Lm1,
\end{align}
\end{subequations}
where  $\bar{\xibf}_\ell=\Vbf_\ell\tran\xibf_\ell$
and $\bar{\bbf}_\ell=\Vbf_\ell\tran\bbf_\ell$ are the transformed bias and noise
and $\bar{\sbf}_\ell$ is the zero-padded singular value vector \eqref{eq:sbar}.
Next, define the scalar-valued functions
\begin{subequations} \label{eq:f0ml}
\begin{align}
    f^0_0(w_0) &:= w_0, \\
    f^0_\ell(p^0_{\lm1},w_\ell) &= f^0_\ell(p^0_{\lm1},\xi_\ell) := \phi_\ell(p^0_{\lm1},\xi_\ell), \quad \ell=2,4,\ldots,L
        \label{eq:f0nonlin} \\
    f^0_\ell(p^0_{\lm1},w_\ell) &= f^0_\ell(p^0_{\lm1},(\bar{s}_\ell,\bar{b}_\ell,\bar{\xi}_\ell))
        = \bar{s}_\ell p^0_\ell + \bar{b}_\ell + \bar{\xi}_\ell, \quad \ell=1,3,\ldots,\Lm1.
        \label{eq:f0lin}
\end{align}
\end{subequations}
For all $\ell$, let $\fbf^0_\ell(\cdot)$ be the componentwise extension of $f^0_\ell(\cdot)$.

\begin{lemma} \label{lem:pq0ml}  With the above definitions,  $\qbf^0_0 = \fbf^0_0(\wbf_0)$ and
\beq \label{eq:pq0ml}
    \qbf^0_\ell = \fbf^0_\ell(\pbf^0_{\lm1},\wbf_\ell), \quad \pbf^0_\ell = \Vbf_\ell \qbf^0_\ell, \quad
    \ell = 1,2,\ldots,\Lm1.
\eeq
\end{lemma}
\begin{proof}
In the initial stage, \eqref{eq:wdef} and \eqref{eq:pq0} show that $\wbf_0 = \qbf^0_0 = \zbf^0_0$.
Hence, \eqref{eq:f0nonlin} shows that $\qbf^0_0 = \fbf^0_0(\wbf_0)$.
For the nonlinear stages $\ell=2,4,\ldots,L$,
\[
    \qbf^0_\ell \stackrel{(a)}{=} \zbf^0_\ell \stackrel{(b)}{=} \phibf(\zbf^0_{\lm1},\xibf_\ell)
    \stackrel{(c)}{=} f^0_\ell(\pbf^0_{\lm1},\wbf_\ell),
\]
where (a) follows from \eqref{eq:pq0}; (b) follows from \eqref{eq:nnnonlintrue}; and (c) follows from
\eqref{eq:f0nonlin}.
For the linear stages $\ell=1,3,\ldots,\Lm1$,
\[
    \qbf^0_\ell \stackrel{(a)}{=} \Vbf_\ell\tran\zbf^0_\ell \stackrel{(b)}{=}
    \sbf_\ell \odot \pbf^0_{\lm1} + \bar{\bbf}_\ell + \bar{\xibf}_\ell
    \stackrel{(c)}{=} f^0_\ell(\pbf^0_{\lm1},\wbf_\ell),
\]
where (a) follows from \eqref{eq:pq0}; (b) follows from \eqref{eq:WSVD} and the
definitions in \eqref{eq:bxibar}; and (c) follows from
\eqref{eq:f0lin}.  Thus, $\qbf^0_\ell = f^0_\ell(\pbf^0_{\lm1},\wbf_\ell)$ for all $\ell=1,2,\ldots,\Lm1$.
The fact that $\pbf^0_\ell = \Vbf_\ell\qbf^0_\ell$ follows from the construction of the
terms in \eqref{eq:pq0}.  Also, by assumption, $\phibf(\cdot)$ acts componentwise and is $PL(2)$.
So $\fbf^0_\ell(\cdot)$ in \eqref{eq:f0nonlin} acts componentwise and is also $PL(2)$.
Since $\sbf_\ell$ has bounded components $\fbf^0_\ell(\cdot)$ in \eqref{eq:f0lin}
acts componentwise and is Lipschitz continuous.
\end{proof}

To understand this lemma, recall that the MLP \eqref{eq:nntrue} generates the vectors $\zbf^0_\ell$
via an alternating sequence of linear operations and nonlinear componentwise activation functions.
In Lemma~\ref{lem:pq0ml}, we have rewritten this recursion as an alternating sequence of
 multiplications by orthogonal matrices $\Vbf_\ell$ and componentwise functions $\fbf^0_\ell(\cdot)$.
We will call the recursions \eqref{eq:pq0ml} the \emph{transformed MLP}.

\noindent
\paragraph*{Error terms}
To analyze the ML-VAMP algorithm, we will look at how well the ML-VAMP algorithms estimates
the  states in the transformed system.  To this end, for $\ell=0,2,\ldots,L-2$, define the vectors:
\begin{subequations} \label{eq:pqdef}
\begin{align}
    &\qbfhat^{\pm}_{k\ell} = \zbfhat^{\pm}_{k\ell}, \quad
    \qbf^{\pm}_{k\ell} = \rbf_{k\ell}^\pm - \zbf^0_\ell, \label{eq:qdefeven} \\
    &\pbfhat^{\pm}_{k,\lp1} = \zbfhat^{\pm}_{k,\lp1}, \quad
    \pbf^{\pm}_{k,\lp1} = \rbf_{k,\lp1}^{\pm} - \zbf^0_{\lp1}, \label{eq:pdefodd}  \\
    &\qbfhat^{\pm}_{k,\lp1} = \Vbf_{\lp1}\tran\pbfhat^{\pm}_{k,\lp1}, \quad
    \qbf^{\pm,\lp1}_{k,\lp1} = \Vbf_{\lp1}\tran\pbf^{\pm}_{k,\lp1} \label{eq:qdefodd} \\
    &\pbfhat^{\pm}_{k\ell} = \Vbf_\ell\qbfhat^{\pm}_{k\ell}, \quad
    \pbf^{\pm}_{k\ell} = \Vbf_\ell\qbf^{\pm}_{k\ell},  \label{eq:pdefeven}
\end{align}
\end{subequations}
The vectors $\qbfhat^{\pm}_{k\ell}$ and $\pbfhat^{\pm}_{k\ell}$ represent the estimates
of $\qbf^0_\ell$ and $\pbf^0_\ell$ in the transformed MLP \eqref{eq:pq0ml}.
Also, the vectors $\qbf^{\pm}_{k\ell}$ and $\pbf^{\pm}_{k\ell}$
are the differences $\rbf_{k\ell}^{\pm}-\zbf^0_\ell$ or their transforms.  These
represent errors on the \emph{inputs} $\rbf_{k\ell}^\pm$ to the estimation functions
$\gbf^{\pm}_\ell(\cdot)$.  The above definitions apply for all $k\geq 0$,
with the exception that, for $k= 0$, we define $\qbf^-_{0\ell}:=\zero$.  This definition
will simplify the proofs below.

\noindent
\paragraph*{Parameter lists}
The parameters in the ML-VAMP algorithm are the terms $\alpha_{k\ell}^\pm$ and $\gamma_{k\ell}^\pm$.
We define the parameter lists, $\Lambda_{k\ell}^\pm$ in Gen-ML as the accumulated sets of these parameters
in the order that they are computed in the ML-VAMP Algorithm~\ref{algo:ml-vamp}:
\begin{subequations} \label{eq:Lamml}
\begin{align}
    \Lambda_{01}^- &:= (\gamma_{00}^-,\ldots,\gamma_{0,\Lm1}^-) \label{eq:Lam01ml} \\
    \Lambda_{k0}^+ &:= (\Lambda_{k0}^-, \alpha_{k\ell}^+, \gamma_{k\ell}^+), \\
    \Lambda_{k\ell}^+ &:= (\Lambda_{k,\lm1}^+, \alpha_{k\ell}^+, \gamma_{k\ell}^+), \quad \ell=1,\ldots,\Lm1, \\
    \Lambda_{\kp1,L}^- &:= (\Lambda_{k,\Lm1}^+,\alpha_{k,\Lm1}^-,\gamma_{\kp1,\Lm1}^-) \\
    \Lambda_{\kp1,\ell}^- &:= (\Lambda_{k,\lp1}^-,\alpha_{k,\lm1}^-,\gamma_{\kp1,\lm1}^-), \quad
        \ell=\Lm1,\ldots,1.
\end{align}
\end{subequations}

\paragraph*{Vector update functions}
The vector update functions, $\fbf^0_\ell(\cdot)$  for the initial pass are defined
as the componentwise extensions of the functions \eqref{eq:f0ml}.
For the forward and backward passes,
let $\ell=0,2,4,\ldots,L$ be the index of  a nonlinear stage.
Define the functions,
\begin{subequations} \label{eq:hnonlin}
\begin{align}
    \hbf^{+}_{k0}(\qbf_{0}^-,\wbf_0,\Lambda_{k1}^-)
        &:= \gbf^{+}_0(\qbf^-_0 + \wbf^0_\ell,\gamma^-_{k\ell})
        - \wbf^0_\ell. \label{eq:hp0nonlin} \\
    \hbf^{+}_{k\ell}(\pbf^0_{\lm1},\pbf_{\lm1}^+,\qbf_{\ell}^-,\wbf_\ell,\Lambda_{k,\lm1}^+)
        &:= \gbf^{+}_\ell(\pbf_{\lm1}^+ +\pbf^0_{\lm1},\qbf^-_\ell + \qbf^0_\ell,\gamma^+_{k,\lm1},\gamma^-_{k\ell})
        - \qbf^0_\ell,  \label{eq:hpnonlin} \\
    \hbf^{-}_{kL}(\pbf^0_{\Lm1},\pbf_{\Lm1}^+,\wbf_L,\Lambda_{k,\Lm1}^+)
        &:= \gbf^{-}_L(\pbf_{\Lm1}^+ +\pbf^0_{\Lm1},\gamma^+_{k,\Lm1}) - \pbf^0_{\lm1}. \label{eq:hnLnonlin} \\
    \hbf^{-}_{k\ell}(\pbf^0_{\lm1},\pbf_{\lm1}^+,\qbf_{\ell}^-,\wbf_\ell,\Lambda_{k,\lp1}^+)
        &:= \gbf^{-}_\ell(\pbf_{\lm1}^+ +\pbf^0_{\lm1},\qbf^-_\ell + \qbf^0_\ell,
        \gamma^+_{k,\lm1},\gamma^-_{\kp1,\ell})
        - \pbf^0_{\lm1}, \label{eq:hnnonlin}
\end{align}
\end{subequations}
In \eqref{eq:hpnonlin} and \eqref{eq:hnnonlin}, the stage index is $\ell=2,4,\ldots,L-2$.
Also, due to Lemma~\ref{lem:pq0ml}, $\qbf_\ell^0 = \fbf_\ell(\pbf_{\lm1}^0,\wbf_\ell)$,
so we can regard $\qbf^0_\ell$ as functions of $\wbf_\ell$.
From \eqref{eq:pqdef} and the definitions in \eqref{eq:wdef} and \eqref{eq:pq0},
we see that the updates for $\zbfhat_{k\ell}^\pm$ in
lines~\ref{line:zp} and \ref{line:zn} in Algorithm~\ref{algo:ml-vamp}
can be rewritten as
\begin{subequations} \label{eq:zhnonlin}
\begin{align}
    \zbfhat^+_{k0} &=  \hbf_{k0}^+(\qbf_{k\ell}^-,\wbf_0,\Lambda_{k1}^-) + \zbf^0_\ell. \\
    \zbfhat^+_{k\ell} &=  \hbf_{k\ell}^+(\pbf^0_{\lm1},\pbf_{k,\lm1}^+,\qbf_{k\ell}^-,
                        \wbf_\ell, \Lambda_{k,\lm1}^+) + \zbf^0_\ell, \quad \ell=2,\ldots,L-2, \\
    \zbfhat^-_{k,\Lm1} &=  \hbf_{kL}^-(\pbf^0_{\Lm1},\pbf_{k,\Lm1}^+,
                        \wbf_L,\Lambda_{kL}^+) + \zbf^0_{\Lm1}.  \\
    \zbfhat^-_{k,\lm1} &=  \hbf_{k\ell}^-(\pbf^0_{\lm1},\pbf_{k,\lm1}^+,\qbf_{\kp1,\ell}^-,
                        \wbf_\ell, \Lambda_{k,\lp1}^-) + \zbf^0_{\lm1}.
\end{align}
\end{subequations}
From this equation, we see that the functions $\hbf_{k\ell}^\pm(\cdot)$ represent the differences between
the estimates $\zbfhat^{\pm}_{k\ell}$ and the true vector $\zbf^0_\ell$.  We will thus call these functions
the nonlinear \emph{error functions}.

Next consider the linear stages, $\ell=1,3,\ldots,\Lm1$.  For these stages, define the functions,
\begin{subequations} \label{eq:hlin}
\begin{align}
    \MoveEqLeft
    \hbf^{+}_{k\ell}(\pbf_{\lm1}^0,\pbf_{\lm1}^+,\qbf_{\ell}^-,\wbf_\ell,\Lambda_{k,\lm1}^+)
        :=  \Gbf^{+}_\ell(\pbf_{\lm1}^+ +\pbf^0_{\lm1},\qbf^-_\ell + \qbf^0_\ell, \sbf_\ell, \bbf_\ell,
        \gamma^+_{k,\lm1},\gamma^-_{k\ell})    - \qbf^0_\ell, \label{eq:hplin} \\
    \MoveEqLeft
    \hbf^{-}_{k\ell}(\pbf_{\lm1}^0,\pbf_{\lm1}^+,\qbf_{\ell}^-,\wbf_\ell,\Lambda_{k,\lp1}^+)
        := \Gbf^{-}_\ell(\pbf_{\lm1}^+ +\pbf^0_{\lm1},\qbf^-_\ell + \qbf^0_\ell, \sbf_\ell, \bbf_\ell,
            \gamma^+_{k,\lm1},\gamma^-_{k\ell})  - \pbf^0_{\lm1}, \label{eq:hnlin}
\end{align}
\end{subequations}
where $\Gbf^{\pm}_\ell(\cdot)$ are the transformed estimation functions for the linear nodes
as described in Appendix~\ref{sec:linestim}.
In the above definition, we have again used Lemma~\ref{lem:pq0ml} to consider
$\qbf_\ell^0$ as a function of $\wbf_\ell$ and $\pbf^0_{\lm1}$.
Combining \eqref{eq:hplin}, \eqref{eq:hnlin}, \eqref{eq:pqdef} and \eqref{eq:glin} with
the updates for $\zbfhat_\ell^+$ and $\zbfhat_{\lm1}^-$ in Algorithm~\ref{algo:ml-vamp}
for the linear stages satisfy, we see that
\begin{subequations} \label{eq:zhlin}
\begin{align}
    \zbfhat^+_{k\ell} &= \Vbf_\ell \left[ h^+_\ell(\pbf^0_{\lm1},\pbf^+_{k,\lm1},
          \qbf_{k\ell}^-,\wbf_\ell,\gamma^+_{k,\lm1},\gamma^-_{k\ell})  + \qbf^0_{\ell} \right] \\
    \zbfhat^-_{k,\lm1} &= \Vbf_{\lm1}\tran \left[ h^-_\ell(\pbf^0_{\lm1},\pbf^+_{k,\lm1},
          \qbf_{\kp1,\ell}^-,\wbf_\ell,\gamma^+_{k,\lm1},\gamma^-_{\kp1,\ell})
          +  \pbf^0_{\lm1} \right].
\end{align}
\end{subequations}
Hence, we can interpret the functions $h^\pm(\cdot)$ are producing transforms of the errors
$\zbfhat^{\pm}_{k\ell} - \zbf^0_\ell$.
For both the linear and nonlinear stages, we then define the vector update functions as
\begin{subequations} \label{eq:fhdef}
\begin{align}
    \fbf^{+}_{k0}(\qbf_{0}^-,\wbf_0,\Lambda_{k0}^+)
    &:= \frac{1}{1-\alpha_{k\ell}^+}\left[
            \hbf^{+}_{k0}(\qbf_{0}^-,\wbf_0,\Lambda_{k1}^-)  - \alpha_{k0}^+ \qbf_0^- \right], \label{eq:fh0p} \\
    \fbf^{+}_{k\ell}(\pbf^0_{\lm1},\pbf_{\lm1}^+,\qbf_{\ell}^-,\wbf_\ell,\Lambda_{k\ell}^+)
    &:= \frac{1}{1-\alpha_{k\ell}^+}\left[
            \hbf^{+}_{k\ell}(\pbf^0_{\lm1},\pbf_{\lm1}^+,\qbf_{\ell}^-,\wbf_\ell,\Lambda_{k,\lm1}^+)
            - \alpha_{k\ell}^+ \qbf_\ell^- \right], \label{eq:fhp} \\
    \fbf^{-}_{kL}(\pbf^0_{\Lm1},\pbf_{\Lm1}^+,\wbf_L,\Lambda_{\kp1,L}^-)
    &:= \frac{1}{1-\alpha_{k\ell}^-}\left[
            \hbf^{-}_{kL}(\pbf^0_{\Lm1},\pbf_{\Lm1}^+,\wbf_L,\lambda_{k,\Lm1}^+)
            - \alpha_{k,\Lm1}^- \pbf_{\Lm1}^+ \right], \label{eq:fhLn}    \\
    \fbf^{-}_{k\ell}(\pbf^0_{\lm1},\pbf_{\lm1}^+,\qbf_{\ell}^-,\wbf_\ell,\Lambda_{\kp1,\ell}^-)
    &:= \frac{1}{1-\alpha_{k\ell}^-}\left[
            \hbf^{-}_{k\ell}(\pbf^0_{\lm1},\pbf_{\lm1}^+,\qbf_{\ell}^-,\wbf_\ell,\Lambda_{\kp1,\lp1}^-)
            - \alpha_{k\ell}^- \pbf_{\lm1}^+ \right], \label{eq:fhn}
\end{align}
\end{subequations}

\noindent
\paragraph*{Parameter updates} Define
\beq \label{eq:lammlv}
    \lambda^+_{k\ell} = (\alpha^+_{k\ell},\gamma^+_{k\ell}), \quad
    \lambda^-_{\kp1,\ell} = (\alpha^-_{k,\lm1},\gamma^-_{\kp1,\lm1}),
\eeq
and parameter statistic functions,
\begin{subequations} \label{eq:vpmlv}
\begin{align}
    \varphibf^{+}_{k0}(\qbf_{0}^-,\wbf_0,\Lambda_{k1}^-) &:=
        \partial \hbf^{+}_{k0}(\qbf_{0}^-,\wbf_0,\Lambda_{k1}^-)/\partial \qbf_0^- \label{eq:vpmlv0p} \\
    \varphibf^{+}_{k\ell}(\pbf_{\lm1}^+,\qbf_{\ell}^-,\zbf^0_{\lm1},\zbf^0_\ell,\Lambda_{k,\lm1}^+)
           &:= \partial \hbf^{\pm}_\ell(\pbf_{\lm1}^+,\qbf_{\ell}^-,\zbf^0_{\lm1},\zbf^0_\ell,\Lambda_{k,\lm1}^+)/
           \partial \qbf_\ell^-, \quad \ell > 0 \label{eq:vpmlvp} \\
   \varphibf^{-}_{kL}(\pbf_{\Lm1}^+,\wbf_L,\Lambda_{k,\Lm1}^+)
          &:= \partial \hbf^{-}_{kL}(\pbf_{\Lm1}^+,\wbf_L,\lambda_{\kp1,\Lm1}^+)/
          \partial \pbf_{\Lm1}^+  \label{eq:vpmlvLn} \\
   \varphibf^{-}_{k\ell}(\pbf_{\lm1}^+,\qbf_{\ell}^-,\wbf_\ell,\Lambda_{\kp1,\lp1}^-) &:=
            \partial \hbf^{-}_{k\ell}(\pbf_{\lm1}^+,\qbf_{\ell}^-,\wbf_\ell,\Lambda_{\kp1,\lp1}^-)/
            \partial \pbf_{\lm1}^-, \quad \ell < L \label{eq:vpmlvn}
\end{align}
\end{subequations}
Define $\mu^{\pm}_{k\ell}$ as
\begin{subequations} \label{eq:mumlv}
\begin{align}
    \mu^+_{k0} &:= \bkt{\varphibf^{+}_{k0}(\qbf_{0}^-,\wbf_0,\Lambda_{k1}^-)}, \quad
    \mu^{+}_{k\ell} :=
        \bkt{\varphibf^{+}_{k\ell}(\pbf_{\lm1}^+,\qbf_{\ell}^-,\zbf^0_{\lm1},\zbf^0_\ell,\Lambda_{k,\lm1}^+)}
    \label{eq:mupmlv} \\
    \mu^{-}_{kL} &:= \bkt{\varphibf^{-}_{kL}(\pbf_{\Lm1}^+,\wbf_L,\Lambda_{k,\Lm1}^+)},\quad
    \mu^{-}_{k\ell} := \bkt{\varphibf^{-}_{k\ell}(\pbf_{\lm1}^+,\qbf_{\ell}^-,\wbf_\ell,\Lambda_{\kp1,\lp1}^-)}
    \label{eq:munmlv}
\end{align}
\end{subequations}
Also, define the parameter update functions as
\begin{subequations} \label{eq:Tmlv}
\begin{align}
    T^{+}_{k0}(\mu_{k0},\Lambda_{k1}^-) &:= \left(\mu_{k0}^+,
        \frac{(1-\mu_{k0}^+)\gamma_{k0}^-}{\mu_{k0}^+} \right), \\
    T^{+}_{k\ell}(\mu_{k\ell}^+,\Lambda_{k,\lm1}^+) &:=
        \left(\mu_{k\ell}^+, \frac{(1-\mu_{k\ell}^+)\gamma_{k\ell}^-}{\mu_{k\ell}^+} \right) \\
    T^{-}_{kL}(\mu_{kL}^-,\Lambda_{k,\Lm1}^+) &:=
         \left(\mu_{kL}^-, \frac{(1-\mu_{kL}^-)\gamma_{k,\Lm1}^+}{\mu_{kL}^-} \right), \\
    T^{-}_{k\ell}(\mu_{k\ell}^-,\Lambda_{\kp1,\lp1}^-) &:=
        \left(\mu_{k\ell}^-,\frac{(1-\mu_{k\ell}^-)\gamma_{k,\lm1}^+}{\mu_{k\ell}^-}\right).
\end{align}
\end{subequations}
With the above definitions, we may state our first key result which establishes
the equivalence of Algorithm~\ref{algo:ml-vamp} to Algorithm~\ref{algo:gen}.

\begin{lemma} \label{lem:genEquiv}
Let $\rbf_{\ell k}^{\pm}$,
$\zbfhat_{\ell k}^{\pm},\ldots$ be the outputs of the ML-VAMP Algorithm (Algorithm~\ref{algo:ml-vamp}).
Define the quantities $\pbf_{\ell k}^{\pm}$, $\qbf_{\ell k}^\pm, \ldots$ as above.
Then the defined quantities satisfy the recursions in the
Gen-ML Algorithm (Algorithm~\ref{algo:gen}).
\end{lemma}
\begin{proof}
We must prove that the quantities satisfy all the updates in Algorithm~\ref{algo:gen}.
Lemma~\ref{lem:pq0ml} shows that the defined quantities satisfies the initial steps,
lines~\ref{line:q00init_gen} to \ref{line:p0init_gen} in Algorithm~\ref{algo:gen}.
We next prove that defined
quantities satisfy line~\ref{line:qp_gen} of Algorithm~\ref{algo:gen}.
First, using lines~\ref{line:rp} and \ref{line:alphap} in Algorithm~\ref{algo:ml-vamp},
we see that, for all $\ell$,
\begin{align}
        \rbf^+_{k\ell} &= \frac{1}{\gamma^+_{k\ell}}\left[
        \eta^+_{k\ell}\zbfhat^+_{k\ell} - \gamma^-_{k\ell}\rbf^-_{k\ell} \right]
        = \frac{1}{1-\alpha^+_{k\ell}}\left[
        \zbfhat^+_{k\ell} - \alpha^+_{k\ell}\rbf^-_{k\ell} \right]. \label{eq:rpalpha}
\end{align}
At this point, we have to consider the case of even $\ell$ (corresponding to the
nonlinear stages) and odd $\ell$ (corresponding to the linear stages) separately.
For the nonlinear stages ($\ell$ even), observe that
\begin{align}
        \qbf^+_{k\ell} &\stackrel{(a)}{=} \rbf^+_{k\ell}-\qbf^0_\ell
        \stackrel{(b)}{=} \frac{1}{1-\alpha^+_{k\ell}}\left[
        \zbfhat^+_{k\ell} - \alpha^+_{k\ell}\rbf^-_{k\ell} \right] - \qbf^0_\ell \nonumber \\
        & \stackrel{(c)}{=}\frac{1}{1-\alpha^+_{k\ell}}\Bigl[
        \hbf^+_\ell(\pbf^0_{\lm1},\pbf_{k,\lm1}^+,\qbf_{k\ell}^-,\wbf_\ell,\Lambda_{k\ell}^+)
             +\qbf^0_{\ell}
        - \alpha^+_{k\ell}\rbf^-_{k\ell} \Bigr] - \qbf^0_\ell \nonumber \\
        & \stackrel{(d)}{=}\frac{1}{1-\alpha^+_{k\ell}}\left[
        \hbf^+_\ell(\pbf^0_{\lm1},\pbf_{k,\lm1}^+,\qbf_{k\ell}^-,\wbf_\ell,\Lambda_{k\ell}^+)
            - \alpha^+_{k\ell}\qbf^-_{k\ell} \right] \nonumber \\
        & \stackrel{(e)}{=}
        \fbf^+_\ell(\pbf^0_{\lm1},\pbf_{k,\lm1}^+,\qbf_{k\ell}^-,\wbf_\ell,\Lambda_{k\ell}^+),
        \label{eq:qfnonlin}
\end{align}
where (a) follows from the definition of $\qbf^+_{k\ell}$ in \eqref{eq:qdefeven} and $\qbf^0_\ell=\zbf^0_\ell$ in
\eqref{eq:pq0};
(b) follows from \eqref{eq:rpalpha};
(c) follows from \eqref{eq:zhnonlin};
(d) follows \eqref{eq:qdefeven} and eliminating the $\qbf^0_\ell$ terms;
and (e) follows from \eqref{eq:fhp}.
Next consider a linear stage $\ell=1,3,\ldots,\Lm1$.    We have that
\begin{align}
     \qbf^+_{k\ell} &\stackrel{(a)}{=} \Vbf_{\ell}\tran(\rbf^+_{k\ell} - \qbf^0_{\ell}) \nonumber \\
        & \stackrel{(b)}{=} \frac{1}{1-\alpha^+_{k\ell}}\Vbf_{\ell}\tran\left[
        \zbfhat^+_{k\ell} - \alpha^+_{k\ell}\rbf^-_{k\ell} \right] - \qbf^0_\ell \nonumber \\
        & \stackrel{(c)}{=}\frac{1}{1-\alpha^+_{k\ell}}\Bigl[
        \hbf^+_\ell(\pbf^0_{\lm1},\pbf_{k,\lm1}^+,\qbf_{k\ell}^-,\wbf_\ell,\Lambda_{k\ell}^+)
        + \Vbf_\ell\tran \qbf_\ell^0
        - \alpha^+_{k\ell}\Vbf_\ell\tran\rbf^-_{k\ell} \Bigr] - \qbf^0_\ell \nonumber \\
        & \stackrel{(d)}{=}\frac{1}{1-\alpha^+_{k\ell}}\left[
        \hbf^+_\ell(\pbf^0_{\lm1},\pbf_{k,\lm1}^+,\qbf_{k\ell}^-,\wbf_\ell,\Lambda_{k\ell}^+)
        - \alpha_{k\ell}^+ \qbf_{k\ell}^- \right] \nonumber \\
        & \stackrel{(e)}{=}
        \fbf^+_\ell(\pbf^0_{\lm1},\pbf_{k,\lm1}^+,\qbf_{k\ell}^-,\wbf_\ell,\Lambda_{k\ell}^+)
        \label{eq:qflin}
\end{align}
where (a) follows from \eqref{eq:qdefodd}, \eqref{eq:pdefodd} and \eqref{eq:pq0};
(b) follows from \eqref{eq:rpalpha};
(c) follows from \eqref{eq:zhlin};
(d) follows from \eqref{eq:qdefodd} and \eqref{eq:pdefodd} and
eliminating the terms with $\qbf_\ell^0$;
and (e) follows from \eqref{eq:fhp}.
Together, \eqref{eq:qfnonlin} and \eqref{eq:qflin} show that line~\ref{line:qp_gen}
of Algorithm~\ref{algo:gen} holds for both even and odd $\ell$.
The proof of lines~\ref{line:q0_gen}, \ref{line:pL_gen} and \ref{line:pn_gen}
are all proved similarly.  One slight issue that we need to consider is the fact that we have defined
$\qbf^-_{0\ell} = \zero$ instead of the definition for $\qbf^-_{k\ell}$ in \eqref{eq:pqdef} used for other $k\geq 0$.
From Section~\ref{sec:mlvamp}, we use the initialization
$\gamma^-_{0\ell} = 0$, and therefore the estimation functions
\eqref{eq:gexp} do not depend on the value of $\rbf^-_{0\ell}$.  Hence, we can substitute any value for
$\qbf^-_{0\ell}$ without changing the other vectors.

We next turn to proving that the quantities satisfy line~\ref{line:pp_gen} in Algorithm~\ref{algo:gen}.
By construction \eqref{eq:pdefeven}, $\pbf^+_{k\ell} = \Vbf_\ell\qbf^+_{k\ell}$ for all even $\ell$.
Also, from \eqref{eq:qdefodd}, for odd $\ell$,
\[
    \qbf^+_{k\ell} = \Vbf_\ell\tran \pbf^+_{k\ell} \Rightarrow
    \pbf^+_{k\ell} = \Vbf_\ell \qbf^+_{k\ell}.
\]
This, line~\ref{line:pp_gen} in Algorithm~\ref{algo:gen} is true for both odd and even $\ell$.
Lines~\ref{line:p0_gen}, \ref{line:qL_gen} and \eqref{line:qn_gen} are proven similarly.

We next show that the defined quantities satisfy the parameter updates in
lines~\ref{line:mup_gen} and \ref{line:lamp_gen} in Algorithm~\ref{algo:gen}.
First consider a nonlinear stage $\ell=2,4,\ldots,L-2$.
Then,
\begin{align}
    \alpha_{k\ell}^+ &\stackrel{(a)}{=} \bkt{\partial
            \gbf_\ell^+(\rbf^+_{k,\lm1},\rbf^-_{k\ell},\gamma^+_{k,\lm1},\gamma^-_{k\ell})/
            \partial \rbf^-_{k\ell}}    \nonumber \\
    &\stackrel{(b)}{=} \bkt{ \partial \hbf^{+}_{k\ell}(\pbf^0_{\lm1},\pbf_{k,\lm1}^+,\qbf_{k\ell}^-,\wbf_\ell,\Lambda_{k,\lm1}^+)/
    \partial \qbf_{k\ell}^- } \nonumber \\
    &\stackrel{(c)}{=} \bkt{ \varphibf^{+}_{k\ell}(\pbf^0_{\lm1},\pbf_{k,\lm1}^+,\qbf_{k\ell}^-,\wbf_\ell,\Lambda_{k,\lm1}^+)
        }
    \stackrel{(d)}{=}  \mu^+_{k\ell},
\end{align}
where (a) follows from line~\ref{line:alphap} in Algorithm~\ref{algo:ml-vamp};
(b) follows from \eqref{eq:hpnonlin}; (c) follows from \eqref{eq:vpmlvp};
and
(d) follows from \eqref{eq:mupmlv}.
Similarly, for a linear stage $\ell=1,3,\ldots,\Lm1$,
\begin{align}
    \alpha_{k\ell}^+ &\stackrel{(a)}{=} \bkt{
    \frac{\partial \Gbf_\ell^+(\bar{\ubf}_{k,\lm1},\bar{\ubf}_{k\ell},
        \sbf_\ell,\bar{\bbf}_\ell,\gamma_{k,\lm1}^+,\gamma_{k\ell}^-)}
        {\partial \bar{\ubf}_{k\ell}}  }    \nonumber \\
    &\stackrel{(b)}{=} \bkt{ \partial \hbf^{+}_{k\ell}(\pbf^0_{\lm1},\pbf_{k,\lm1}^+,\qbf_{k\ell}^-,\wbf_\ell,\Lambda_{k,\lm1}^+)/
    \partial \qbf_{k\ell}^- } \nonumber \\
    &\stackrel{(c)}{=} \bkt{ \varphibf^{+}_{k\ell}(\pbf^0_{\lm1},\pbf_{k,\lm1}^+,\qbf_{k\ell}^-,\wbf_\ell,\Lambda_{k,\lm1}^+)
        }
    \stackrel{(d)}{=}  \mu^+_{k\ell},
\end{align}
where (a) follows from \eqref{eq:alphapG}, where
$\bar{\ubf}_{k,\lm1} = \Vbf_{\lm1}\rbf^+_{k,\lm1}$
and $\bar{\ubf}_{k\ell} = \Vbf_{\ell}\tran\rbf^+_{k\ell}$;
(b) follows from \eqref{eq:hplin};
(c) again follows from \eqref{eq:vpmlvp}; and
(d) follows from \eqref{eq:mupmlv}.
This shows that for all $\ell=1,\ldots,\Lm1$, $\mu_{k\ell}^+ = \alpha_{k\ell}^+$.
Also, from line~\ref{line:gamp} in Algorithm~\ref{algo:ml-vamp}, we have that
$\gamma^+_{k\ell} = (1/\alpha_{k\ell}^+-1)\gamma^-_{k\ell}$.  Therefore,
\beq
    (\alpha^+_{k\ell},\gamma^+_{k\ell}) = \left(\mu_{k\ell}^+,
        \frac{\gamma_{k\ell}^-(1-\mu_{k\ell}^+)}{\mu_{k\ell}^+}\right)
        = T^+_{k\ell}(\mu_{k\ell}^+,\Lambda_{k,\lm1}^+),
\eeq
where we have used the definition of $T^+_{k\ell}(\cdot)$ in \eqref{eq:Tmlv}.
This proves that the defined quantities satisfy line~\ref{line:mup_gen} and \ref{line:lamp_gen}
in Algorithm~\ref{algo:gen}.  The other updates for $\mu_{k\ell}^{\pm}$ and $\Lambda_{k\ell}^{\pm}$
are proven similarly.
Thus, we have shown that the defined quantities satisfy all the recursions in Algorithm~\ref{algo:gen}.
\end{proof}

We next establish that all the assumptions are satisfied.

\begin{lemma} \label{lem:asml}
The defined quantities above satisfy Assumptions~\ref{as:gen} and \ref{as:gen2}.
\end{lemma}
\begin{proof} We begin with Assumption~\ref{as:gen}.  Assumption~\ref{as:gen}(a)
follows from the definitions of $\wbf_\ell$ in \eqref{eq:wdef},
the convergence assumptions in \eqref{eq:varinitnl} and \eqref{eq:varinitnl}, the definition  $\qbf^-_{0\ell} = \zero$
and $\gamma^-_{0\ell}=0$
and the definition of the initial parameter list $\Lambda_{01}^-$ in \eqref{eq:Lam01ml}.
Assumption~\ref{as:gen}(b) follows from the construction of $\Vbf_\ell$ in Section~\ref{sec:seevo}.
Also, for the nonlinear stages, the estimation functions $\gbf^{\pm}_{\ell}(\cdot)$ described in Section~\ref{sec:nnest}
act componentwise as do the functions $\Gbf^{\pm}_\ell(\cdot)$ for the linear stages.
This implies that the functions $\hbf^{\pm}_{k\ell}(\cdot)$ and the $\fbf^{\pm}_{k\ell}(\cdot)$ also act componentwise
which proves Assumption~\ref{as:gen}(c).

For Assumption~\ref{as:gen2}(a), the functions in \eqref{eq:Tmlv} are continuous since we have assumed
that $\alphabar_{k\ell}^{\pm} \in (0,1)$.  For any nonlinear stages $\ell$,
we have assumed the denoiser $g^{\pm}_{k\ell}(\cdot)$ is uniformly Lipschitz continuous.  Also, since the
singular values $s_\ell$ are bounded the estimation functions for the linear stages $g^{\pm}_{k\ell}(\cdot)$
in Appendix~\ref{sec:linestim} are also uniformly Lipschitz continuous.  Thus, the functions $h^{\pm}(\cdot)$
in \eqref{eq:hlin} and \eqref{eq:hnlin} and the functions \eqref{eq:fhdef} are uniformly Lipschitz continuous,
which proves Assumption~\ref{as:gen2}(b).  A similar argument can be used for Assumption~\ref{as:gen2}(d).
The divergence free property in Assumption~\ref{as:gen2}(c) occurs since
\begin{align*}
    \MoveEqLeft \bkt{\frac{\partial \fbf^+_{k\ell}(\pbf^+_{k,\lm1},\qbf_{k\ell}^-,\wbf_\ell,\Lambdabar^+_{k\ell})}{
        \partial \qbf_{k\ell}^-} } \\
    &= \frac{1}{1-\alpha^-_{k\ell}} \left[ \bkt{\frac{\partial \hbf^+_{k\ell}(\pbf^+_{k,\lm1},\qbf_{k\ell}^-,\wbf_\ell,\Lambdabar^+_{k,\lm1})}{
        \partial \qbf_{k\ell}^-} } -\alpha_{k\ell} \right] = \frac{1}{1-\alpha^-_{k\ell}} \left[\alpha_{k\ell}-\alpha_{k\ell}\right] = 0.
\end{align*}
\end{proof}

\subsection{A General Convergence Result}

To describe the state evolution, we next need to introduce a number of random variables that will model
the asymptotic distribution of the components of various vectors in the ML-VAMP algorithm.
First, given the random variables in \eqref{eq:varinitnl} and \eqref{eq:varinitlin}, define $W_\ell$ as,
\begin{subequations} \label{eq:Wdeflim}
\begin{align}
    W_0 &:= Z^0_0, \quad
    W_\ell := \Xi_\ell, \quad \ell =2,4,\ldots,L \\
    W_\ell &= (\bar{S}_\ell,\bar{B}_\ell,\bar{\Xi}_\ell),  \quad \ell =1,3,\ldots,\Lm1,
\end{align}
\end{subequations}
which we can think of as random variables corresponding to the components of disturbance vectors \eqref{eq:wdef}.
Next, suppose we are given parameters $\gamma_{\lm1}^+$, $\gamma_{\ell}^-$ and $\tau^0_{\lm1}$.
Then, define the set of random variables,
\begin{align} \label{eq:PQnoise}
\begin{split}
    R^+_{\lm1} &\sim \Norm\left(0,\tau^0_{\lm1}-\frac{1}{\gamma_{\lm1}^+}\right), \quad
    P^+_{\lm1} \sim \Norm\left(0,\frac{1}{\gamma_{\lm1}^+}\right), \quad P^0_{\lm1} = R^0_{\lm1} - P^+_{\lm1} \\
    Q^0_{\ell} &= f^0_\ell(P^0_{\lm1},W_\ell), \quad
    Q^-_{\ell} \sim \Norm\left(0,\frac{1}{\gamma^-_\ell}\right), \quad
    R^-_\ell = Q^0_\ell - Q^-_{\ell},
\end{split}
\end{align}
where we assume $P^-_{\lm1}$ is independent of $R^+_{\lm1}$ and $Q^-_{\ell}$ is independent of $(R^+_{\lm1},P^-_{\lm1},W_\ell)$.
In the model \eqref{eq:PQnoise},  represent the inputs and outputs of a component function $f^0_\ell(\cdot)$ for the $\ell$-th
stage  in the transformed MLP and $R^+_{\lm1}$ and $R^-_{\ell}$ represent noisy corrupted versions of these vectors.
Also, for $\ell=1,2,\ldots,\Lm1$, define the random variables
\begin{align} \label{eq:pqmean}
\begin{split}
    \hat{Q}^+_\ell &:= h^+_{k\ell}(R^+_{\lm1}, R^-_\ell,W_\ell,\bar{\gamma}^+_{\lm1},\bar{\gamma}^-_\ell) - Q^0_{\ell} \\
    \hat{P}^-_{\lm1} &:= h^-_{k\ell}(R^+_{\lm1}, R^-_\ell,W_\ell,\bar{\gamma}^+_{\lm1},\bar{\gamma}^-_\ell) - P^0_{\lm1},
\end{split}
\end{align}
where $h^{\pm}_{k\ell}(\cdot)$ are the component functions  corresponding to the functions \eqref{eq:hnlin} and \eqref{eq:hlin}.
We have the following:

\begin{lemma} \label{lem:hgexp}
For the $\ell=2,4,\ldots,L-2$, the variables in \eqref{eq:pqmean} satisfy
\begin{align} \label{eq:pqmeannl}
    \hat{Q}^+_\ell = \Exp(Q^0_\ell|R^+_{\lm1}, R^-_\ell), \quad
    \hat{P}^-_{\lm1} = \Exp(P^0_{\lm1}|R^+_{\lm1}, R^-_\ell),
\end{align}
and, for $\ell=1,3,\ldots,\Lm1$,
\beq \label{eq:pqmeanlin}
    \hat{Q}^+_\ell = \Exp(Q^0_\ell|R^+_{\lm1}, R^-_\ell,\bar{S}_\ell,\bar{B}_\ell), \quad
    \hat{P}^-_{\lm1} = \Exp(P^0_{\lm1}|R^+_{\lm1}, R^-_\ell,\bar{S}_\ell,\bar{B}_\ell).
\eeq
\end{lemma}
\begin{proof}  For $\ell=2,4,\ldots,L-2$, \eqref{eq:hnlin} shows that
\[
    \hat{Q}^+_\ell := g^+_{\ell}(R^+_{\lm1}, R^-_\ell,\bar{\gamma}^+_{\lm1},\bar{\gamma}^-_\ell).
\]
But $g^+_\ell(\cdot)$ in \eqref{eq:gnl} is precisely the conditional expectation in \eqref{eq:pqmeannl}.
For $\ell=1,3,\ldots,\Lm1$, \eqref{eq:hlin} shows that
\[
    \hat{Q}^+_\ell := G^+_{\ell}(R^+_{\lm1}, R^-_\ell,\bar{S}_\ell,\bar{B}_\ell,\bar{\gamma}^+_{\lm1},\bar{\gamma}^-_\ell).
\]
The derivation in Appendix~\ref{sec:linestim} shows that $G^+_\ell(\cdot)$ is exactly the conditional expectation in
\eqref{eq:pqmeanlin}.  The arguments for $\hat{P}^-_{\lm1}$ are similar.
\end{proof}

\medskip
The lemma shows that the variables $\hat{Q}^+_\ell$ and $\hat{P}^-_{\lm1}$ in \eqref{eq:pqmean}
represent the estimates of $Q^0_\ell$ and $P^0_{\lm1}$ from the noisy measurements $(R^+_{\lm1}, R^-_\ell)$.
Finally, for all $\ell$, given parameters $\alpha^+_\ell$ and $\alpha^-_{\lm1}$, define
\beq \label{eq:pqpn}
    Q^+_\ell = \frac{1}{1-\alpha^+_\ell} \left[ \hat{Q}^+_\ell - Q^0_\ell - \alpha^+_\ell Q^-_\ell\right], \quad
    P^-_{\lm1} = \frac{1}{1-\alpha^-_{\lm1}} \left[ \hat{P}^-_{\lm1} - P^0_{\lm1} - \alpha^-_{\lm1} P^+_{\lm1}\right].
\eeq
Given variables as distributed in \eqref{eq:PQnoise}, \eqref{eq:pqmean} and \eqref{eq:pqpn}, we will write
\begin{align} \label{eq:gamp}
\begin{split}
    (P^0_{\lm1},Q^0_{\lm1},P^+_{\lm1},Q^-_\ell,\hat{Q}^+_\ell,Q^+_\ell) &\sim
        \Gamma_\ell^+(\gamma^-_{\lm1},\gamma^+_\ell,\tau^0_{\lm1},\alpha^+_\ell), \\
    (P^0_{\lm1},Q^0_{\lm1},P^+_{\lm1},Q^-_\ell,\hat{P}^-_{\lm1},P^-_{\lm1}) &\sim
        \Gamma_\ell^-(\gamma^-_{\lm1},\gamma^+_\ell,\tau^0_{\lm1},\alpha^-_{\lm1}).
\end{split}
\end{align}
For $\ell=0$, we can define $\Gamma^+_0(\gamma_0^-,\alpha^+_0)$
by dropping the term $P^+_{\lm1}$, and, for $\ell=L$, we define
$\Gamma_L^-(\gamma_{\Lm1}^+,\tau^0_{\Lm1},\alpha^-_{\Lm1})$ by dropping the term $Q^-_{\ell}$.
We now state the following.

\medskip
\begin{theorem} \label{thm:semlgen}  Under the assumptions of Theorem~\ref{thm:semlvamp},
for any fixed $k$ and $\ell$:
\begin{enumerate}[(a)]
\item Almost surely,
\begin{align}
    \MoveEqLeft \lim_{N \arr \infty} \left\{ (p^0_{\lm1,n},q^0_{\ell,n},p^+_{k,\lm1,n},q^-_{k\ell,n},\hat{q}^+_{k\ell,n},q^+_{k\ell,n}) \right\} \\
    & \PLeq (P^0_{\lm1},Q^0_\ell,P^+_{k,\lm1},Q^-_{k\ell},\hat{Q}^+_{k\ell},Q^+_{k\ell}), \label{eq:hpveclim}
\end{align}
where the variables on the right hand side are distributed as
\beq \label{eq:hpvar}
    (P^0_{\lm1},Q^0_\ell,P^-_{k,\lm1},Q^-_{k\ell},\hat{Q}^+_{k\ell},Q^+_{k\ell}) \sim
            \Gamma_\ell^+(\gammabar^+_{k,\lm1},\gammabar^-_{k\ell},\tau^0_{\lm1},\alphabar_{k\ell}^+).
\eeq
In addition, almost surely,
\beq \label{eq:hpparamlim}
    \lim_{N \arr \infty} (\eta^+_{k\ell},\alpha^+_{k\ell},\gamma^+_{k\ell})
        =  (\etabar^+_{k\ell},\alphabar^+_{k\ell},\gammabar^+_{k\ell}),
\eeq
and
\beq \label{eq:hpmom}
    \Exp(P^0_{\lm1})^2 = \tau^0_{\lm1}, \quad \Exp(Q^0_{\ell})^2 = \tau^0_\ell, \quad
     \Exp(\hat{Q}^{+}_{k\ell}-Q^0_\ell)^2  =\frac{1}{\etabar^{+}_{k\ell}},
\eeq
and $\etabar^+_{k\ell},\alphabar^+_{k\ell},\gammabar^+_{k\ell}$ satisfy lines~\ref{line:etap_mlse} and \ref{line:gamp_mlse}
of Algorithm~\ref{algo:mlvamp_se}.
The same statement holds for $\ell=0$ where
 we remove the terms associated with $P^0_{\lm1}$ and $P^-_{k,\lm1}$.

\item Almost surely,
\begin{align}
    \MoveEqLeft \lim_{N \arr \infty} \left\{ (p^0_{\lm1,n},q^0_{\ell,n},p^+_{\km1,\lm1,n},q^-_{k,\ell,n},\hat{p}^-_{\km1,\lm1,n},p^-_{k,\lm1,n}) \right\} \\
    & \PLeq (P^0_{\lm1},Q^0_\ell,P^+_{\km1,\lm1},Q^-_{k\ell},\hat{P}^-_{\km1,\lm1},P^-_{k,\lm1}), \label{eq:hnveclim}
\end{align}
where the variables on the right hand side are distributed as
\beq \label{eq:hnvar}
    (P^0_{\lm1},Q^0_\ell,P^+_{\km1,\lm1},Q^-_{k\ell},\hat{P}^-_{\km1,\lm1},P^-_{k,\lm1})
    \sim   \Gamma_{\lm1}^-(\gammabar^+_{\km1,\lm1},\gammabar^-_{k,\lm1},\tau^0_{\lm1},\alphabar_{k,\lm1}^-)
\eeq
In addition, almost surely,
\beq \label{eq:hnlim}
    \lim_{N \arr \infty} (\eta^-_{\km1,\lm1},\alpha^-_{\km1,\lm1},\gamma^-_{k\ell})
        =  (\etabar^-_{\km1,\lm1},\alphabar^-_{\km1,\lm1},\gammabar^-_{k\ell}),
\eeq
and
\beq \label{eq:hnmom}
    \Exp(P^0_{\lm1})^2 = \tau^0_{\lm1}, \quad \Exp(Q^0_{\ell})^2 = \tau^0_\ell, \quad
     \Exp(\hat{P}^{-}_{\km1,\ell}-P^0_{\lm1})^2  = \frac{1}{\etabar^{-}_{k,\lm1}},
\eeq
and $\etabar^-_{k,\lm1},\alphabar^-_{k,\lm1},\gammabar^+_{\kp1,\lm1}$ satisfy lines~\ref{line:etan_mlse} and \ref{line:gamn_mlse}
The same statement holds for $\ell=L$ where
we remove the terms associated with $Q^0_\ell$ and $Q^-_{k\ell}$.
\end{enumerate}
\end{theorem}

\medskip
Theorem~\ref{thm:semlvamp} is a special case of Theorem~\ref{thm:semlgen}.  So, we only need to
prove Theorem~\ref{thm:semlgen}, which we do now.

\subsection{Proof of Theorem~\ref{thm:semlgen}}
Lemmas~\ref{lem:genEquiv} and \ref{lem:asml} show that the defined quantities from the
ML-VAMP algorithm follow the recursions of the Gen-ML algorithm and satisfy all the necessary assumptions
of Theorem~\ref{thm:genConv}.  We can therefore apply Theorem~\ref{thm:genConv} to show that
all the variables converge empirically.  Let $P^0_\ell,P^{\pm}_{k\ell},\hat{P}^\pm_{k\ell}$
be the empirical limits of the components of the vectors $\pbf^0_\ell$, $\pbf^{\pm}_{k\ell}$ and
$\hat{\pbf}^{\pm}_{k\ell}$ as described in the Gen-ML state evolution, Algorithm~\ref{algo:gen_se}.
Similarly, let $Q^0_\ell,Q^{\pm}_{k\ell},\hat{Q}^\pm_{k\ell}$
be the empirical limits of the components of the vectors $\qbf^0_\ell$, $\qbf^{\pm}_{k\ell}$ and
$\hat{\qbf}^{\pm}_{k\ell}$.   Also, since $\alpha^{\pm}_{k\ell}$ and $\gamma^{\pm}_{k\ell}$
are in the parameter lists as defined in \eqref{eq:Lamml}, we know, from Theorem~\ref{thm:genConv},
these converge to limits $\alphabar^{\pm}_{k\ell}$ and $\gammabar^{\pm}_{k\ell}$.
Let $\etabar^{\pm}_{k\ell} = \gammabar^{\pm}_{k\ell}/\alphabar^{\pm}_{k\ell}$.
We need to show that these limiting random variables and constants agree with
the distributions as described in Theorem~\ref{thm:semlgen}.

First observe that lines in the ``Initial pass" section of the Gen-ML SE (Algorithm~\ref{algo:gen_se})
exactly matches the corresponding lines of ML-VAMP SE in Algorithm~\ref{algo:mlvamp_se} when we use the functions
as defined in \eqref{eq:f0lin} and \eqref{eq:f0nonlin}.  So, we have
that for all $\ell$,
\beq \label{eq:q00init_pf}
    Q^0_0 = f^0_0(W_0), \quad P_0 \sim \Norm(0,\tau^0_0),
\eeq
and, for $\ell=1,\ldots,\Lm1$,
\beq \label{eq:q0init_pf}
    Q^0_\ell=f^0_\ell(P^0_{\lm1},W_\ell), \quad
    P^0_\ell = \Norm(0,\tau^0_\ell), \quad
    \tau^0_\ell = \Exp(Q^0_\ell)^2.
\eeq

We next proceed with an induction argument.
Similar to the proof of Theorem~\ref{thm:genConv}, define the following sequence of hypotheses:
\begin{itemize}
\item $\mathcal{H}_{k\ell}^+$:  The hypothesis that Theorem~\ref{thm:semlgen}(a)
is true for some $k$ and $\ell$.
\item $\mathcal{H}_{k\ell}^-$:  The hypothesis that Theorem~\ref{thm:semlgen}(b)
is true for some $k$ and $\ell$.
\end{itemize}
We can prove the hypotheses in the sequence \eqref{eq:induc}. We illustrate how to prove
the implication $\mathcal{H}^+_{k\ell} \Rightarrow \mathcal{H}^+_{k,\lp1}$.  The proof
of the other implications are similar.  Thus, we assume that all the hypotheses prior
to $\mathcal{H}^+_{k,\lp1}$ in the sequence \eqref{eq:induc} are true.  We will show that, under
this assumption $\mathcal{H}^+_{k,\lp1}$ is true.

Now, proving hypothesis $\mathcal{H}^+_{k,\lp1}$ is equivalent to showing that, if we define,
\beq  \label{eq:Rdefpf}
    R^+_{k,\ell} := P^0_\ell + P^+_{k,\ell}, \quad
    R^-_{k,\lp1} := Q^0_{\lp1} + Q^-_{k,\lp1},
\eeq
then
\begin{subequations} \label{eq:gaminduc}
\begin{align}
	&R^+_{k,\ell} \sim \Norm\left(0,\tau^0_{\ell}-\frac{1}{\gamma^+_{k\ell}}\right), \quad
     P^+_{k\ell} \sim \Norm\left(0,\frac{1}{\gamma^+_{k\ell}}\right), \label{eq:gamind_rp} \\
    & Q^0_{\lp1} = f^0_{\lp1}(P^0_{\ell},W_{\lp1})  \label{eq:gamind_q0}  \\
    & Q^-_{\lp1} \sim \Norm\left(0,\frac{1}{\gamma^-_{k,\lp1}}\right)    \label{eq:gamind_qn}    \\
    & \hat{Q}^+_{k,\lp1} = h^+_{k,\lp1}(R^+_{k\ell},R^-_{k,\lp1},W_{\lp1},\gammabar^+_{k\ell},\gammabar^-_{k,\lp1})  \label{eq:gamind_qhat} \\
    & Q^+_{k,\lp1} = \frac{1}{1-\alphabar^+_{k,\lp1}} \left[ \hat{Q}^+_{k,\lp1} - Q^0_{\lp1} - \alphabar^+_{k,\lp1} Q^-_{\lp1} \right],
         \label{eq:gamind_qp}
\end{align}
\end{subequations}
where $R^+_{k,\ell}$ and $P^+_{k\ell}$ are independent of one another and $Q^-_{k,\lp1}$ is independent of
$P^0_\ell, P^+_{k\ell}$ and $W_{\lp1}$.
We also need to show that
\beq \label{eq:etagamind}
    \etabar_{k,\lp1}^+ = \frac{1}{\Ecal_{\lp1}(\gammabar^+_{k\ell},\gammabar^-_{k,\lp1})}, \quad
    \alphabar_{k,\lp1}^+ = \frac{\gammabar^-_{k,\lp1}}{\etabar_{k,\lp1}^+}, \quad
    \gammabar_{k,\lp1}^+ = \etabar_{k,\lp1}^+ - \gammabar^-_{k,\lp1}.
\eeq
We will prove all the items in \eqref{eq:gaminduc} and \eqref{eq:etagamind} under the induction hypothesis.

We begin with proving \eqref{eq:gamind_rp}.
Since $\mathcal{H}^+_{k\ell}$ is prior to $\mathcal{H}^+_{k,\lp1}$ in the sequence \eqref{eq:induc},
we can assume it is true.
Under this assumption, we first evaluate various covariance terms.  If $\ell$ is odd,
\begin{align}
    \MoveEqLeft \Exp\left[ (\hat{Q}^+_{k\ell}-Q^0_{\ell})Q^0_\ell \right]
    \stackrel{(a)}{=} \Exp\left[ (\Exp(Q^0_{\ell} |R^+_{k,\lm1},R^-_{k\ell})-Q^0_\ell) Q^0_{\ell} \right]  \nonumber \\
    &\stackrel{(b)}{=} -\var(Q^0_\ell|R^+_{k,\lm1},R^-_{k\ell}) \stackrel{(c)}{=} -\frac{1}{\etabar^+_{k\ell}}, \label{eq:covqhatq0}
\end{align}
where (a) follows from the model \eqref{eq:pqmeannl};
(b) is the law of iterated expectations; and
(c) follows from the fact that $\etabar^+_{k\ell} = 1/\Ecal_{\ell}^+(\gammabar^+_{k,\lm1},\gammabar^-_{k\ell})$
as defined in \eqref{eq:Ecalnl}.  A similar argument shows that \eqref{eq:covqhatq0} holds for $\ell$ even.
Also, since $Q^-_{k\ell}$ is independent of $Q^0_\ell$, we have that
$\Exp(Q^-_{k\ell} Q^0_\ell) = 0$.  Therefore,
\begin{align}
    \MoveEqLeft \Exp\left[ Q^+_{k\ell}Q^0_\ell \right]
     \stackrel{(a)}{=}
    \frac{1}{1-\alphabar^+_{k\ell}} \left[ \Exp((\hat{Q}^+_{k\ell} - Q^0_\ell)Q^0_\ell) - \alphabar^+_{k\ell} \Exp(Q^-_{k\ell} Q^0_\ell) \right] \nonumber \\
    &\stackrel{(b)}{=}  -\frac{1}{(1-\alphabar^+_{k\ell})\etabar^+_{k\ell}} \stackrel{(c)}{=} -\frac{1}{\gammabar^+_{k\ell}}, \label{eq:covqpq0},
\end{align}
where (a) follows from the definition of $Q^+_{k\ell}$ in \eqref{eq:hpvar};
(b) used \eqref{eq:covqhatq0}; and
(c) used the fact the definition of $\gammabar^+_{k\ell}$ in line~\ref{line:gamp_mlse} in
Algorithm~\ref{algo:mlvamp_se}.  We next consider the correlation $\Exp(Q^+_{k\ell})^2$.  Using \eqref{eq:hpvar},
\beq \label{eq:covqpsq1}
    \Exp(Q^+_{k\ell})^2
    =
    \frac{1}{(1-\alphabar^+_{k\ell})^2} \left[ \Exp(\hat{Q}^+_{k\ell} - Q^0_\ell)^2 - 2\alphabar^+_{k\ell} \Exp((\hat{Q}^+_\ell - Q^0_\ell)Q^-_{k\ell})
        + (\alphabar^+_{k\ell})^2 \Exp(Q^-_{k\ell})^2 \right].
\eeq
Now, using Stein's Lemma similar as in \eqref{eq:Qijstein}, one can show
\beq
    \Exp((\hat{Q}^+_{k\ell} - Q^0_\ell)Q^-_{k\ell}) = \alphabar^+_{k\ell} \Exp(Q^-_{k\ell})^2
\eeq
Substituting this into \eqref{eq:covqpsq1}, we obtain
\begin{align}
    \MoveEqLeft \Exp(Q^+_{k\ell})^2
    =  \frac{1}{(1-\alphabar^+_{k\ell})^2} \left[ \Exp(\hat{Q}^+_{k\ell} - Q^0_\ell)^2 - (\alphabar^+_{k\ell})^2 \Exp(Q^-_{k\ell})^2 \right] \nonumber \\
    &= \frac{1}{(1-\alphabar^+_{k\ell})^2} \left[ \frac{1}{\etabar^+_{k\ell}} - \frac{(\alphabar^+_{k\ell})^2}{\gammabar^-_{k\ell}} \right]
    = \frac{1}{\gammabar^+_{k\ell}}, \label{eq:covqpsq2}
\end{align}
Equations \eqref{eq:covqpq0} and \eqref{eq:covqpsq2} show that
\[
    \Cov(Q^0_\ell,Q^+_{k\ell}) = \left[ \begin{array}{cc}
        \tau^0_\ell & -1/\gammabar^+_{k\ell} \\ -1/\gammabar^+_{k\ell} & 1/\gammabar^+_{k\ell} \end{array} \right].
\]
Now, from the SE equations in Algorithm~\ref{algo:gen_se}, we know $\Cov(P^0_\ell,P^+_{k\ell}) = \Cov(Q^0_\ell,Q^+_{k\ell})$ and hence
\[
    \Cov(P^0_\ell,P^+_{k\ell}) = \left[ \begin{array}{cc}
        \tau^0_\ell & -1/\gammabar^+_{k\ell} \\ -1/\gammabar^+_{k\ell} & 1/\gammabar^+_{k\ell} \end{array} \right].
\]
Since $R^+_{k\ell} = P^0_\ell + P^+_{k\ell}$, we have
\[
    \Cov(R^+_{k\ell},P^+_{k\ell}) = \left[ \begin{array}{cc}
        \tau^0_\ell-1/\gammabar^+_{k\ell}  &  0 \\ 0 & 1/\gammabar^+_{k\ell} \end{array} \right].
\]
This proves \eqref{eq:gamind_rp}.  Equation \eqref{eq:gamind_q0} holds since we have already proven this in \eqref{eq:q0init_pf}.
Equation \eqref{eq:gamind_qn} and the independence of $Q^-_{k\lp1}$ with the other variables follows from
Theorem~\ref{thm:genConv}.  Equation \eqref{eq:gamind_qhat} follows from the definition of $\qbfhat^+_{k\ell}$ in \eqref{eq:pqdef}
and the fact that $\zbfhat^+_{k\ell}$ and $\Vbf_\ell\tran\zbfhat^+_{k\ell}$ are given by \eqref{eq:zhnonlin} and \eqref{eq:zhlin}.
Finally, \eqref{eq:gamind_qp} follows from the \eqref{eq:fhdef}.  Thus, we have established all the necessary relations in
\eqref{eq:gaminduc}.

For \eqref{eq:etagamind}, first observe that since $\eta^+_{k,\lp1}$ is given by \eqref{eq:gderiv},
it can be shown that the limit $\etabar^+_{k,\lp1} = 1/\Ecal_{\lp1}(\gammabar^+_{k\ell},\gammabar^-_{k,\lp1})$.
Also, taking the limits of line~\ref{line:gamp} in Algorithm~\ref{algo:ml-vamp}, we obtain that
$\alphabar_{k,\lp1}^+ = \gammabar^-_{k,\lp1}/\etabar_{k,\lp1}^+$ and $\gammabar_{k,\lp1}^+ = \etabar_{k,\lp1}^+ - \gammabar^-_{k,\lp1}$.
Thus, we have proven \eqref{eq:etagamind} and with it we have proven the induction step.  This proves Theorem~\ref{thm:semlgen}.

\section{Numerical Experiments Details} \label{sec:simdetails}

\paragraph*{Synthetic random network}  As described in Section~\ref{sec:sim},
the network input is a $N_0=20$ dimensional Gaussian unit noise vector $\zbf_0$.
and has three hidden layers with 100, 500 and 784 units.  For the weight matrices and bias vectors
in all but the final layer,
we took $\Wbf_\ell$ and $\bbf_\ell$ to be random i.i.d.\ Gaussians.  The mean of the bias vector was selected
so that only a fixed fraction, $\rho=0.4$, of the linear outputs would be positive.  The activation
functions were rectified linear units (ReLUs), $\phi_\ell(z)=\max\{0,x\}$.  Hence, after activation,
there would be only a fraction $\rho=0.4$ of the units would be non-zero.  In the final layer,
we constructed the matrix similar to \citep{RanSchFle:14-ISIT} where $\Abf=\Ubf\diag(\sbf)\Vbf\tran$,
with $\Ubf$ and $\Vbf$ be random orthogonal matrices and $\sbf$ be logarithmically spaced valued
to obtain a desired condition number of $\kappa=10$.  It is known from \citep{RanSchFle:14-ISIT}
that matrices with high condition numbers are precisely the matrices in which AMP algorithms fail.
For the linear measurements,
$\ybf = \Abf\zbf_5+\wbf$, the noise level $10\log_{10}( \Exp\|\wbf\|^2/\|\Abf\zbf_5\|^2)$ is set at 30 dB.
In Fig.~\ref{fig:randmlp_sim}, we have plotted the normalized MSE (in dB) which we define as
\[
    \mathrm{NMSE} := 10\log_{10}\left[ \frac{\|\zbf_0-\hat{\zbf}_0\|^2}{\|\zbf_0\|^2} \right].
\]
Since each iteration of ML-VAMP involves a forward and reverse pass, we say that each iteration
consists of two ``half-iterations", using the same terminology as turbo codes.  The left panel
of Fig.~\ref{fig:randmlp_sim} plots the NMSE vs. half iterations.

\paragraph*{MNIST inpainting}  The well-known MNIST dataset consists of handwritten images
of size $28 \x 28 = 784$ pixels. We followed the procedure in \citep{kingma2013auto} for training
a generative model from 50,000 digits.
Each image $\xbf$ is modeled as the output of a neural network
input dimension of 20 variables followed by
a single hidden layer with 400 units and an output layer of 784 units,
corresponding to the dimension of the digits.
ReLUs were used for activation functions and a sigmoid was placed at the output
to bound the final pixel values between 0 and 1.
The inputs $\zbf_0$ were the modeled as zero mean Gaussians with unit variance.
The data was trained using the Adam optimizer
with the default parameters in TensorFlow
\footnote{Code for the training was based on
\url{https://github.com/y0ast/VAE-TensorFlow} by Joost van Amersfoort.}
The training optimization was run
with 20,000 steps with a batch size of 100 corresponding to 40 epochs.

The ML-VAMP algorithm was compared against two other standard inference methods:
MAP and stochastic gradient Langevin dyanmics (SGLD).
To describe both methods,
let $\zbf_0$ be the 20-dimensional unknown input to the neural network
that generates the image $\xbf$ and let $\ybf$ be the occluded image.
Define the Hamiltonian
\beq \label{eq:Hmnist}
    H(\zbf_0) := -\ln p(\ybf|\zbf_0) - \ln p(\zbf_0),
\eeq
so that the posterior density of $\zbf_0$ given $\ybf$ is given by
$p(\zbf_0|\ybf) \propto \exp(-H(\zbf_0))$.  MAP estimation,
as studied in \citep{yeh2016semantic,bora2017compressed},
can be performed via numerical minimization of the Hamiltonian $H(\zbf_0)$.
Given the estimated input $\zbf_0$, one can then estimate the
image $\xbf$ by running $\zbf_0$ through the neural network.
We used Tensorflow for the minimization.  We found the fastest
convergence with the Adam optimizer at a step-size of 0.01.  This required
only 500 iterations to be within 1\% of the final loss function.

Reconstruction was also be performed via SGLD
as described in detail in \citep{welling2011bayesian}.
SGLD is a classic technique to generate samples $\zbf_0^k$ from the
$p(\zbf_0|\ybf)$.  For each sample, one can compute a predicted image, $\xbf^k$, and
then average these predicted images to estimate the posterior mean image $\Exp(\xbf|\ybf)$.
SGLD generates the samples $\zbf_0^k$ via a sequence of updates,
\beq \label{eq:sgld}
    \zbf_0^{\kp1} = \zbf^k_0 - \lambda_k \nabla H(\zbf_0^k) + \sqrt{2\lambda_k}\wbf^k,
    \quad \wbf^k \sim \Norm(0,\Ibf),
\eeq
where the noise sequence $\wbf^k$ is i.i.d.\ and $\lambda_k$ is a step-size.
If the step size $\lambda_k$ is sufficiently small, then it can be shown
\citep{welling2011bayesian} that,
under suitable conditions, the steady-state density of the
samples $\zbf^k_0$ approximates the posterior density $p(\zbf_0|\ybf)$.

In this simulation, we performed the update \eqref{eq:sgld} by
using Tensorflow's Stochastic Gradient Descent optimizer and adding random Gaussian noise.
We took the step-size of $\lambda_k = 0.002$ and generated 10000 samples.
The first 5000 samples were used for burn-in
so that posterior mean was then estimated from the final 5000 samples.

Finally, for ML-VAMP, the sigmoid
function does not have an analytic denoiser, so it was approximated with a probit output.
The ML-VAMP algorithm was run for 20 iterations, although good convergence was observed after
approximately 10 iterations.  

\bibliographystyle{IEEEtran}
\bibliography{../bibl}

\end{document}